\documentclass[opre,nonblindrev]{informs3} 

\DoubleSpacedXI 

\DeclareFontFamily{U}{calligra}{}
\DeclareFontShape{U}{calligra}{m}{n}{<->callig15}{}

\def\cF{{\cal F}}

\def\cG{{\cal G}}

\def\cD{{\cal D}}
\def\cO{{\cal O}}
\def\cS{{\cal S}}

\def\event{\mathcal{E}}

\def\cG{\mathcal{G}}

\def\dis{{\sf d}}

\def\reals{{\mathbb R}}

\def\up{{\sf p}}

\def\eps{{\varepsilon}}
\def\prob{{\mathbb P}}
\def\E{{\mathbb E}}

\def\scale{{\sf d}}

\def\L0{{L_i}}

\def\de{{\rm d}}
\def\<{\langle}
\def\>{\rangle}

\def\hth{\widehat{\theta}}

\def\hSigma{\widehat{\Sigma}}

\def\supp{{\rm supp}}

\def\F{{\sf F}}
\def\ind{{\mathbb I}}

\def\F{{\sf F}}

\def\P{{\mathbb{P}}}

\def\sT{{\sf T}}

\def\event{\mathcal{E}}

\def\v*{v_i}
\def\T*{T_i}

\def\u*{u_i}
\def\F*{F_i}

\def\reg{\mathcal{R}}

\def\cL{\mathcal{L}}

\def\hth{{\widehat{\theta}}}

\def\cX{\mathcal{X}}

\def\event{\mathcal{E}}

\def\l1u{W}

\def\bv{\mathbf{v}}

\newcommand{\ajcomment}[1]{}

\makeatletter
\newcommand{\labitem}[2]{%
\def\@itemlabel{\text{#1}}
\item
\def\@currentlabel{#1}\label{#2}}
\makeatother

\DeclareMathAlphabet{\mathpzc}{OT1}{pzc}{m}{it}

\usepackage{algorithm2e}
\usepackage[noend]{algorithmic}
\usepackage{caption}
\usepackage{nameref}
\usepackage{makecell}
\usepackage{hyperref}
\usepackage[font={small}]{caption} 
\usepackage{float}
\usepackage{enumitem}
\usepackage{amssymb,amsfonts,amsmath,amscd,dsfont,mathrsfs,bbm}
\usepackage{amsmath}
\usepackage{amssymb}
\usepackage{graphicx,float,psfrag,epsfig,amssymb}
\usepackage[usenames,dvipsnames,svgnames,table]{xcolor}
 \def\bibfont{\small}%

 \def\vm{v_{t}^-}
  \def\bm{b_{t}^-}

  \def\cO{{\cal O}}
 
 \def\ind{{\mathbb I}}
 
  \def\reg{{\sf Reg}}

  \def\tr{\tilde{r}}
  \def\epi{E}
  \def\initial{I}
  
  \def\lF{{l_F}}
  \def\uF{{u_F}}
  \def\M{M}
  \def\b{b_{-it}^+}
  \def\v{v_{-it}^+}

 \renewcommand{\qed}{\hfill \mbox{\raggedright $\square$}}

\renewcommand{\P}[1]{\mathbb{P}\left(#1\right)}
\newcommand{\px}{{\cal D}}
\def\<{\langle}
\def\>{\rangle}
\def\pv{\beta}
\def\cX{{\cal{X}}}
\def\nbuyer{N}
\def\dnoise{F}
\def\fnoise{f}
\def\maxpv{B_p}
\def\maxn{B_n}
\def\covx{\Sigma_x}

\def\rev{{\sf{rev}}}
\def\mbf{\mathbf}

\def\dis{\gamma}

\def\bb{\mbf{b}}
\def\bv{\mbf{v}}
\def\br{\mbf{r}}
\def\is{{i^\star}}

\def\hbeta{\widehat{\beta}}
\def\Lie{{\sf L}}
\def\shade{s}
\def\length{{\lceil\ell_k^{2/3}\rceil}}
\def\oscale{\alpha_0}
\def\scale{\alpha}

\newcommand{\new}[1]{{{\color{black}{#1}}}}

\newcommand{\ajb}[1]{{{\color{black}{#1}}}}
\newcommand{\nc}[1]{{{\color{black}{#1}}}}
\newcommand{\ngg}[1]{{{\color{black}{#1}}}}
\newcommand{\CRR}[1]{{{\color{black}{#1}}}}

  \newtheorem{propo}{Proposition}[section]
\newtheorem{coro}[propo]{Corollary}
\newtheorem{thm}[propo]{Theorem}

\newenvironment{policy}[1][htb]
  {
   \begin{algorithm}[#1]%
  }{\end{algorithm}}

\usepackage{natbib}
 \bibpunct[, ]{(}{)}{,}{a}{}{,}%
 \def\bibfont{\small}%
\TheoremsNumberedThrough     
\ECRepeatTheorems

\EquationsNumberedThrough    

\begin{document}

\RUNAUTHOR{Golrezaei et al.}

\RUNTITLE{Dynamic Incentive-Aware Learning: Robust Pricing in Contextual Auctions}

\TITLE{Dynamic Incentive-aware Learning: Robust Pricing in Contextual Auctions}


\ARTICLEAUTHORS{%
\AUTHOR{Negin~Golrezaei}
\AFF{Sloan School of Management, Massachusetts Institute of Technology, Cambridge, MA, \EMAIL{golrezae@mit.edu}} 
\AUTHOR{Adel~Javanmard}
\AFF{Data Sciences and Operations Department, University
of Southern California, Los Angeles, CA, \EMAIL{ajavanma@usc.edu}}
\AUTHOR{Vahab~Mirrokni\footnote{The names of the authors are in alphabetical order. Moreover, part of this work is done when Negin Golrezaei was a postdoctoral researcher  at  Google Research, New York. }
}
\AFF{Google Research, New York, NY, \EMAIL{mirrokni@google.com}}
} 

\ABSTRACT{%
  Motivated by pricing in ad exchange markets, we consider the problem  of  robust learning of reserve  prices against strategic  buyers in repeated contextual second-price auctions. Buyers' valuations \new{for} an item depend on the context  that describes the item.  However, the seller is not aware of  the relationship between the context  and buyers' valuations, i.e., buyers' preferences. The seller's goal is to design a learning policy to set reserve prices via observing the past sales data, and her objective is  to minimize her regret for revenue, where the regret  is computed against   a clairvoyant policy that knows buyers' heterogeneous  preferences. Given the seller's goal,  utility-maximizing buyers  have the incentive to bid untruthfully in order to manipulate the seller's learning policy.  {We propose  learning policies that are robust to such strategic behavior. These policies use  the outcomes of the auctions, rather than the submitted bids, to estimate the preferences  while controlling the long-term effect of the outcome of each auction on the future reserve prices. 
   When  the market noise distribution is known to the seller, we propose a policy called  Contextual Robust Pricing (CORP) that achieves a T-period regret  of  $O(d\log(Td) \log (T))$, where $d$ is the dimension of {the} contextual information. When  the market noise distribution is unknown to the seller, we propose two policies whose regrets are sublinear in $T$.}}


\maketitle

%


\section{Introduction}  
{In many online marketplaces, both sides of the market have access to  rich dynamic contextual information about the products that are being sold over time. On the buying side, such information can influence the willingness-to-pay of the buyers for the products, potentially in a heterogeneous way. On the selling side, the information can help the seller differentiate the products and set contextual and possibly personalized prices. To do so,  the seller needs to learn the impact of this information on buyers' willingness-to-pay. Such contextual learning can be challenging for the seller when there are  repeated interactions between the buying and the selling sides. With repeated interactions, the utility-maximizing  buyers may have the incentive to act strategically and trick the learning policy of the seller into lowering their prices.  Motivated by this, our key research question is as follows:
How can the seller dynamically optimize  (personalized) prices in a \emph{robust} manner, taking into account 
the strategic behavior of the buyers?}

{One of the online marketplaces that faces this problem is the online advertising market. In this market, a prevalent approach to sell  ads is via running real-time second-price auctions in which  advertisers can use an abundance of detailed contextual
information before deciding what to bid. In this practice, advertisers can target Internet users based on their (heterogeneous) preferences and targeting criteria. Targeting can create a thin and uncompetitive  market in which few advertisers show  interest in each auction. In such a thin market, it is crucial for the ad exchanges to effectively optimize the reserve prices in order to boost their revenue. However, learning the optimal reserve prices is rather difficult due to frequent interactions between  advertisers and  ad exchanges.   
 
{Inspired by this environment, we study a model in which a seller runs repeated {(lazy)} second-price auctions with reserve over time. In the lazy auction, an item is being sold to the buyer with the highest submitted bid as long his bid exceeds his reserve.\footnote{\ngg{Another version of the second price auction is called ``eager". In this version, the buyers whose bids are less than their reserve price are first removed from the auction. Then, the item is allocated to one of the remaining buyers who has the highest submitted bid.}}} 
 The valuation (willingness-to-pay) of each buyer \new{for the item} in period $t$, which is his private information, depends on an observable $d$-dimensional  contextual  information in that period and his preference vector. 
  We focus on an important special case of this contextual-based  valuation model in which the buyer's value is a linear function of his preference vector and contextual  information, plus some random noise term, where the noise models the impact of contexts that are not measured/observed by the seller.\footnote{\ngg{Appendix \ref{sec:discussion} discusses how  our results can be  extended to some of the nonlinear valuation models.}} The preference vector, which is unknown to the seller and fixed throughout the time horizon,  varies across  buyers. Thus, the preference vectors  capture heterogeneity in buyers' valuation.}\footnote{\ngg{In Appendix \ref{sec:discussion}, we discuss pricing under the settings where  the preference vectors change  over time and as a result, the obtained data is perishable. }}
  
    {The seller's goal is to design a policy that dynamically  learns/optimizes  personalized reserve prices. 
  The buyers are fully aware of the learning policy used by the seller and  act strategically in order to maximize their (time-discounted) cumulative utility. Dealing with such a strategic  population of buyers, the seller aims at 
    extracting as much revenue  as  the clairvoyant policy {that is cognizant of the preference vectors a priori}. These vectors determine the relationship between the 
 valuation of the buyers and  contextual information. 
        \new{Put} differently, the seller would like to minimize her regret where the regret is defined as the difference between the seller's  revenue  and that under the  {clairvoyant policy}.  Note that the clairvoyant policy provides a strong benchmark because the policy
     posts the optimal personalized  reserve prices based on the observed contexts. 
  
  As stated earlier,  the main hurdles in designing a low-regret learning policy in this setting are the frequent interactions between the seller and the buyers. 
  Due to such interactions, the strategic buyers might have the incentive to {bid untruthfully. This way,}  they may sacrifice their short-term utility in order to deceive the  seller, to post them lower future reserve prices. 
Thus, while a single shot second-price auction is a truthful mechanism, repeated second-price auctions in which the seller aims at  dynamically  learning optimal reserve prices of {strategic and  utility-maximizing} buyers may not be truthful.  The untruthful bidding behavior of the  buyers makes it hard for the seller to learn the optimal reserve prices, and this, in turn, can lead to her revenue loss. This highlights the necessity to 
design a robust learning policy that reduces  buyers' incentive to follow untruthful strategy.  
Beside this hurdle, the availability  of the dynamic  contextual information requires the seller to change the reserve prices dynamically over time, based on the contextual information. To do so, the seller needs to learn how  buyers react to such information and based on the reactions, posts  (dynamic) personalized reserve prices. The need to have a personalized reserve price is caused by  heterogeneity in  buyers' preferences.}
  
{We consider setting where the seller (firm) is more patient than the buyer. We formalize it by considering time-discounted utility for the buyers. This is motivated by various {applications}. For example, in online advertisement markets, the advertisers (buyers) who {retarget Internet users}  prefer showing their ads to the users \new{who visited their website}  sooner {rather} than later.} 
  
{\ngg{In this paper, we propose three learning policies.} The first policy, which we call Contextual Robust Pricing (CORP), is tailored to a setting where the distribution of the noise term in buyers' valuation is known to the seller.  We will refer to this noise as  market or valuation noise.  
  By studying this setting, we can characterize the seller's revenue loss due to 
 her  lack of knowledge about the buyers' (heterogeneous) response to contextual information. \ngg{The second policy, that we call  CORP-II, is a variant of the first policy. This policy lends itself to a setting where the unknown  market noise distribution is fixed throughout the time horizon and belongs to a location–scale family.\footnote{\CRR{A location–scale family is a family of probability distributions parametrized by a location parameter and a non-negative scale parameter. Then, if a probability distribution function of a random variable $Y$ belongs to this family, the probability distribution function of random variable $aY+b$ also belongs to this family.  The location-scale families are quite broad and contain Normal, Elliptical, Cauchy, Uniform, Logistic, Laplace, and Extreme value distributions, as examples.}} The third policy, which is called Stable CORP (SCORP), is designed to the setting where the market noise distribution varies over time and as a result, the seller does not have the intention of learning the market noise distribution. She instead would like to design a learning policy that is robust to the uncertainty in the noise distribution.}

{In the remaining part of the introduction, we briefly discuss the salient characteristics of each policy separately and defer the formal description to Sections~\ref{sec:corp} and~\ref{sec:unknownF}.}

\begin{itemize}
\item \textbf{{CORP Policy:}} {When the market noise distribution is known to the seller, under a log-concavity assumption on the noise distribution,  our CORP policy gets  {the cumulative} T-period regret of order    
$O\left(Nd \left(\log(Td)\log(T) +\frac{\log^2(T)}{\log^2(1/\gamma)}\right)\right)$, where the regret is computed against the clairvoyant policy that knows the preference vectors as well as the market noise distribution. Here, $N$ is the number of buyers, $\gamma$ is the buyers' discount factor, and $O\left(Nd \frac{\log^2(T)}{\log^2(1/\gamma)}\right)$ is the extra regret due to the strategic behavior of the buyers.} The policy works in an episodic manner {where the  length of episodes doubles each time.  Some} of the periods in each episode are randomly assigned to exploration, and the rest of the periods are dedicated to exploitation. 
At the beginning of each episode, CORP updates its estimates of the preference vectors by running a maximum Likelihood estimator using only the auction outcomes from the previous episode and then adheres to those estimates throughout the episode. During the exploitation periods, CORP sets the reserves based on its estimates of the preference vectors and {its}  knowledge of the noise distribution.
 As time progresses, the policy becomes more confident about its estimates and consequently uses those estimate over a longer episode.

We now highlight two important aspects of CORP. 
As explained earlier, the CORP policy has an episodic structure and updates its estimate of preference vectors only at the beginning of each episode. 
Such design makes the policy robust by restricting the future effect of the submitted bids. Specifically, bids in an episode are not used in choosing the reserve prices until the beginning of the next episode. Therefore, there is always a delay until a buyer observes the effect of a bid on reserves. Then, considering the fact that buyers are impatient and discount the future, they are more incentivized to bid truthfully.

{There is another important aspect of the policy that ensures its robustness: its estimation method. Rather than using the submitted bids to estimate the preference vectors, the policy  simply uses the \emph{outcome of the auctions}. 
 Because of this feature of the policy, bidding untruthfully does not always result in lower reserve prices; {instead,   
 it} can {impact} the future reserve prices of a buyer when it  leads to {changing the outcome of an auction, i.e., \new{when} a buyer loses an auction due to underbidding or
 a buyer wins an auction due to overbidding.} 
 {As it becomes more clear later, the CORP-II and SCORP policies are also designed in a way to enjoy the aforementioned  robustness properties.}}
\smallskip

\item \ngg{\textbf{{CORP-II Policy:}} We design this policy for the setting where market noise distribution, which is fixed throughout time horizon, is unknown and belongs to a location–scale family.  This policy obtains a regret in the order of $O\left(Nd \log(Td) \sqrt{T} + N^2 d \Big(1+\frac{1}{\log^2(1/\gamma)}\Big) \log^3(T) \right)$ against a clairvoyant policy that knows the preference vectors and  market noise distribution. Similar to the CORP policy, CORP-II estimates the preference vectors and parameters of the market noise distribution using a maximum Likelihood estimator.  It also enjoys an episodic structure. However,   due to uncertainty in the market noise distribution, the length of the episodes grows at a slower rate, compared with that in CORP.  } 

\item \textbf{{SCORP Policy:}} 
{\ngg{Our SCORP policy is designed for a setting where the time-varying market noise distribution is unknown to the seller and belongs to an ambiguity set. Then, under the  log-concavity assumption on the noise distribution, our  policy that knows the ambiguity set,  obtains  {the} T-period  regret of order   
  $O\left( N\sqrt{d  \log(Td)}\; T^{2/3} +\ngg{ \frac{N}{\log(1/\gamma)} \log(T)\; T^{1/3}} \right)$.} Here, the regret is computed against a benchmark called \new{\emph{robust}}; see Definition \ref{propo:benchmark-WC}.  The robust benchmark  bears some resemblance to the benchmark used in the regret analysis of CORP; it knows the true preference vectors and the ambiguity set and based on this knowledge chooses  the reserve prices that work well against the worst distribution in the ambiguity set. \ngg{{In contrast to the benchmark used in CORP-II, the benchmark here does not intend to learn the noise distribution, as the  distribution of the noise can be time-varying. It instead posts ``robust" reserve prices.}} } 
         
        {{The increase in the regret, compared to CORP-II, is due to the fact that the noise distribution is time-varying and as a result, the seller cannot hope to learn it. {Because of this,}  the policy  spends more time on exploration compared to CORP-II, which leads to its higher regret.} \ngg{SCORP uses the same episodic structure as CORP-II but dedicates the beginning {portion} of each episode to pure exploration. Concretely, in episode $k$, with length $\ell_k$, pure exploration phase consists of $\lceil\ell_k^{2/3}\rceil$ periods.       
     Spending more time on exploration is not the sole difference between CORP-II and SCORP.  Given that the noise distribution is  time-varying,  SCORP uses a least-square estimator to update the estimates of preference vectors, while 
        CORP-II  employs  the maximum Likelihood estimator, taking advantage of the fact the market noise distribution is fixed and belongs to a location–scale family.} }

\end{itemize}
 
  \ngg{The rest of the paper is organized in the following manner.  In Section \ref{sec:related}, we review the  literature related to our work. Section \ref{sec:model} formally defines our model.  We present the CORP in Section \ref{sec:corp} and  present CORP-II and SCORP policies in  
   Section \ref{sec:unknownF}.  
Finally, we conclude in Section \ref{sec:conclude}. 

This paper has an electronic companion. Appendix \ref{sec:lowerB} reviews  lower bounds on regret in different pricing settings. In Appendix \ref{sec:discussion}, we provide a discussion on (i) extending our policies to a setting with some  non-linear valuation models, and (ii) learning how to price when data is perishable.   {Appendices} \ref{proof:thm-main}, \ref{proof:thm-F}, and \ref{proof:thm-main2} provide the proof of the regret bound of  CORP, CORP-II, and SCORP,  respectively.}

 \section{Related Work}\label{sec:related}
 In this section, we briefly discuss the literature related to our work. 
 
 \textit{Dynamic Pricing with Learning:} 
 Our work is related to the growing body of research on dynamic pricing with learning; see \citep{den2015dynamic} for a survey.  
 \citep{rothschild1974two, araman2009dynamic, farias2010dynamic, harrison2012bayesian, cesa2015regret, ferreira2016online, cheung2017dynamic} studied dynamic pricing with demand uncertainty in the non-contextual and Bayesian settings. In contrast, \citep{broder2012dynamic, den2013simultaneously,besbes2009dynamic} studied this type of problems in the frequentist settings. In these settings, the parameters of the model, which are unknown (but fixed), are estimated using the maximum Likelihood  (ML) method or other estimation techniques. We note that there are two important aspects that distinguish our work from this line of literature: the presence of the contextual information and strategic behavior of the buyers. In the following, we elaborate on these aspects by reviewing the related work.

 \medskip
  \textit{Contextual Dynamic Pricing with Learning in {Non-strategic Environment:}} Recently, several works considered the problem of dynamic pricing in a {non-strategic} setting when the unknown demand function depends on the customers' characteristics (aka contextual information). {In such settings, customers are not strategic in a sense that they do not consider the impact of their current actions on the future prices they will see.}
  {\cite{chen2015statistical}} studied this problem when the demand function follows the logit model and proposed an ML-based learning algorithm. \ngg{{\cite{leme2018contextual, cohen2016feature}, and \cite{lobel2016multidimensional}}   proposed a learning algorithm based on the binary search method when the  demand function is linear and deterministic.} In their models, buyers have homogenous preference vectors and are non-strategic. Hence, the problem reduces to a single buyer setting, where the buyer acts myopically, {i.e.,  the buyer} does not consider the impact of the current actions on the future prices. In our setting, however, the seller interacts with a heterogeneous set of buyers in repeated second-price auctions, rather than the posted-price mechanism. Thus, the seller should estimate a preference vector per buyer and use these estimates to set personalized contextual-based reserve prices.  There is also a new line of literature  that studied dynamic pricing with demand learning when the contextual information is high dimensional (but sparse){; see} \cite{javanmard2016dynamic, ban2017personalized}.  Similar problems have been investigated in \cite{bastani2015online,javanmard2017perishability}.

   {As mentioned earlier,  in our setting, the seller repeatedly  interacts with a small number of strategic and heterogeneous buyers.} We note that  \cite{edelman2007strategic} presented empirical evidence that showed buyers in online advertising markets act strategically. There is  also a large body of literature that studied dynamic pricing in a setting where buyers are strategic and the demand function is known \new{a priori}; see, for example, \cite{borgs2014optimal, besbes2015intertemporal,golrezaei2017dynamic}.\footnote{{Very recently, \cite{chen2018markdown} study dynamic pricing with unknown demand. Here, customers  have unit-demand, arrive over time, and time their purchase strategically. }} This literature highlights the importance of considering the strategic behavior of buyers in updating prices over time.

    \medskip
   \textit{Pricing with Strategic Buyers and Demand Learning:}  \cite{amin2013learning, medina2014learning}, and \cite{kanoria2017dynamic} examined the problem of dynamic pricing with strategic buyers in a non-contextual environment.\footnote{{Learning with strategic players has been studied very recently  in different settings including spread betting markets \citep{birge2018dynamic} and multi-armed bandit settings \citep{braverman2017multi}.}
    }
    In \cite{amin2013learning} and \cite{medina2014learning}, the seller repeatedly interacts with a single strategic buyer via a posted-price mechanism. Similar to our setting, the seller is more patient than the buyer in a sense that the buyer discounts his future utility.  {\cite{amin2013learning}} showed that no learning  algorithm can obtain a sub-linear regret when the buyer  is as patient as the seller. In addition, via designing learning {policies}, they demonstrated that the seller can get a sub-linear regret bound when the buyer is less patient. 
   
 \cite{kanoria2017dynamic} studied dynamic pricing when a group of strategic buyers competes with each other in  repeated non-contextual second-price auctions. A main difference with our setting is that in their model, buyers are as patient as the firm and hence there is no time-discount factor for buyers' utilities.  They designed a near-optimal {elegant} pricing policy in which the reserve price of each buyer is computed using the submitted bids of other buyers. {Specifically, for any $\epsilon>0$, their policy can be designed to achieve $(1-\epsilon)$ of the expected revenue obtained under the static Myerson optimal auction {for the valuation distribution}. Note that this corresponds to \new{a linear regret bound} in our terms. Indeed in the setting that buyers do not discount their future utilities and buyers are utility-maximizer, it is impossible to get a sub-linear regret~\citep{amin2013learning}. Further, in \citep{kanoria2017dynamic} it is 
 assumed that {products to be sold are ex-ante identical, and that} buyers are homogenous and  their valuations are all drawn from a single distribution, which is unknown to the seller.

With respect to the homogeneity assumption, we point out that there  exists empirical evidence that buyers are indeed heterogeneous \citep{guimaraes2011sales, johnson2003multiproduct, golrezaei2017boosted}. It is not surprising that the heterogeneity  in the markets makes the design of  selling mechanisms more difficult. In addition, such difficulties get more severe when the seller needs to design dynamic selling mechanisms  for a group of strategic buyers that  compete with each other repeatedly. 
 
Recently, \cite{incentive2017} studied a similar problem in a static non-contextual setting where the seller has access to $m$ (strategic) data points and using these data points, she would like to design a mechanism that can incentivize the buyers {to be truthful in the first place.} They show that when the market power of each buyer is negligible, designing such a mechanism is feasible. To achieve this result, they apply the technique of differential privacy~\citep{MT07}.

 {Closer to the spirit of this paper,~\cite{amin2014repeated} studies the problem of pricing inventory in a repeated posted-price auction. The authors propose a pricing algorithm whose regret is in the order of $O(\sqrt{\log T}\, T^{2/3})$ in a contextual setting, against a strategic buyer. \footnote{\ajb{Dependency on $d$ is hidden in the big-$O$ notation.}}}
     We point out that our regret result  improves upon~\cite{amin2014repeated} in the following directions: 
\begin{itemize}[leftmargin=*]
\item [-] We allow for market noise in our model, whereas~\cite{amin2014repeated} considers noiseless setting which posits that buyer's valuation is given as a linear function of features. \ngg{Due to this difference, their algorithm and regret bound obtained for noiseless setting in \cite{amin2014repeated} cannot be  applied to our noisy setting and vice versa.} Nonetheless, 
by adding the noise component, we make the model {richer}. {When the noise distribution is known, our CORP policy obtains  a T-period regret  of  $O(d\log(Td) \log (T))$. In addition, when the noise distribution is unknown, our SCORP policy, which is \emph{doubly} robust against strategic buyers and the uncertainty in the noise distribution, obtains  a T-period regret  of $O(d\sqrt{\log(Td)}\;T^{2/3})$.}

\item [-] We consider a market of strategic buyers who participate in a second-price auction at each round, while~\cite{amin2014repeated}, motivated by targeting in online advertising, considers a single buyer \nc{case}. Note that in case of a single buyer, there is no notion of bid, as the buyer only needs to decide if he is willing to purchase the item at the posted price. By contrast, in a market of buyers, each submitted bid of a buyer can potentially affect the utility of that buyer (instant and long-term utility), other buyers' utilities and the seller's revenue. \ngg{We note that Section 5 in~\cite{amin2014repeated} considers an extension to {the multiple buyers case} but assumes that the highest valuation {in} each period $t$ can be written as $\<x_t,\beta\>$ for a fixed parameter vector $\beta$, and product feature (context) $x_t$, which we find to be {a strong} assumption. }  

\end{itemize}

\textit{Behavior-based Pricing:}  Our work is also related to the literature on behavior-based pricing where the seller \new{uses} the past behavior of the buyers to update the prices \citep{hart1988contract, schmidt1993commitment, fudenberg2006behavior, acquisti2005conditioning, esteves2009survey, bikhchandani2012behavior}. 
In this literature, it is mostly assumed that 
the buyer's valuation, which is drawn from a publicly known distribution, is fixed throughout the time horizon. Thus, the seller does not need to learn the valuation distribution; instead, the seller aims at learning the realized valuation of the buyer. Note that considering a static  valuation for a buyer 
in the online advertising market is not reasonable, as in this market, buyer's valuation can depend on the  rich contextual information, which varies over time. 
 We also note that in the  behavior-based pricing literature, the seller and  buyer usually get engaged in repeated games, where each of the player responds to other player's strategy to form  a perfect Bayesian Nash equilibrium.

\section{Model}
\label{sec:model}
Before we describe the model, we adopt {some notation} that will be used throughout the paper. For an integer $a$, we write $[a] = \{1,2,\dotsc, a\}$. In addition, for a vector $v\in \reals^d$, we denote its {$j^{\text{th}}$} coordinates by $v_j$, for $j\in [d]$, and indicate its $\ell_2$ norm by $\|v\|$. For two vectors {$v, u\in \reals^d$, $\<u,v\> = \sum_{j=1}^d u_j v_j$} represents their inner product. {Finally, $\ind(\cdot)$ denotes the indicator function: $\ind(A) =1$ when event $A$ happens, and is zero otherwise.   }

We consider a firm who runs  repeated second-price auctions with \emph{personalized} reserve over a finite time horizon with length $T$. In each period $t\ge 1$, the firm would like to sell an item to one of $\nbuyer$ buyers. The item in period $t$ is represented by an observable feature (context) vector denoted by $x_t \in \reals^d$. We assume that the features are drawn independently from a fixed distribution $\px$, with a bounded support $\cX\subseteq \reals^d$.  
Note that the length of the time horizon $T$ and {distribution $\px$} are unknown to the firm. For the sake of normalization and without loss of generality, we assume that $\|x_t\|\le 1$, and {hence take $\cX = \{x\in \reals^d:\, \|x\|\le 1\}$.} 
We let $\covx =\E[x_tx_t^\sT]$ be {the second moment matrix} of distribution $\px$, and assume that $\covx$ is a positive definite matrix, where $\covx$ is unknown to the firm.

For buyers' valuations, we consider a feature-based model that captures heterogeneity among the buyers. In the following, we discuss the specifics of the valuation model.
Valuation of buyer $i \in [\nbuyer]$ for an item in period $t\ge 1$ depends on the feature vector $x_t$ and period $t$ {and is denoted by $v_{it}(x_t)$.} 
  We assume that $v_{it}(x_t)$ is a linear function of a preference vector $\pv_i$ and the feature vector $x_t$.  \ngg{(We relax this assumption in Appendix \ref{sec:discussion}.)} That is, 
\begin{align}v_{it} (x_t)= \<x_t,\pv_i\> +z_{it} ~~~~~ i\in [\nbuyer],~ t\ge 1\,. \label{eq:val}\end{align}
Whenever it is clear from the context, 
we may remove the dependency of valuation  $v_{it} (x_t)$ on the feature vector $x_t$ and denote it by $v_{it}$. 
Here, $\pv_i \in \mathbb{R}^d$ represents the buyer $i$'s preference vector, and for the sake of normalization, we assume  that  
$\|\pv_i\| \le \maxpv$, $i\in [N]$, {where $\maxpv$ is a constant.} 
The terms $z_{it}$'s, $i\in [N]$, $t\ge 1$, {which are independent of the feature vector $x_t$,}  are idiosyncratic shocks {and are} referred to as noise. \ngg{The noise terms are drawn independently and identically
from a mean zero distribution $\dnoise: [-\maxn,\maxn]\rightarrow [0,1]$ with continuous density $\fnoise : [-\maxn,\maxn] \rightarrow \mathbb{R}^+$}, where $\maxn$ is a constant.\footnote{{The noise aims at capturing features that are not observed/measured by the firm.}} We assume that the firm knows  the distribution of the noise $\dnoise$. We relax this assumption later in Section \ref{sec:unknownF}. Note that the valuation of buyer $i$, {$v_{i t}$}, is not known to the firm, as the preference vector $\pv_i$ and realization of the noise $z_{it}$ are not observable to her. \new{In addition}, by our normalization, $v_{it}(x_t)\le B$, with $B = \maxpv+\maxn$.

We make the following assumption on distribution of the noise $\dnoise$.

\begin{assumption}[Log-concavity] \label{assump:logcancavity} $F(z)$ and $1-F(z)$ are log-concave \new{in $z\in[-\maxn,\maxn]$.}
\end{assumption} 

Assumption \ref{assump:logcancavity}, which is prevalent in the economics literature \citep{bagnoli2005log}, holds by several common probability distributions including uniform, and {(truncated)}
Laplace, exponential, and logistic {distributions}.  
A few remarks are in order regarding Assumption~1. If distribution $F$ is log-concave and its density $f$ is symmetric, i.e., $f(z) = f(-z)$, then $1-F(z) = F(-z)$ is also log-concave. \CRR{Moreover, if  density $f$  is log-concave, the cumulative distribution function $F$ and the reliability function  $1-F$ are also log-concave~\citep{bagnoli2005log}. This implies that Assumption \ref{assump:logcancavity} is satisfied when the density $f$ is log-concave.} We also point out that if a density has a monotone hazard rate (MHR), i.e., $\dfrac{f(z)}{1-F(z)}$ is increasing in $z$, then $1-F(z)$ is log-concave. This point, in turn, shows that all MHR and symmetric densities satisfy Assumption \ref{assump:logcancavity}.
\medskip

We next {describe the repeated second-price auctions and} 
discuss the firm's problem. The goal of the firm is to maximize the cumulative expected revenue {in} {repeated second-price auctions.} The firm tries to achieve this by choosing reserves in a \emph{dynamic} and \emph{personalized} manner.  
\subsection{{Second-price Auctions} with Dynamic Personalized Reserves}
Before defining {a second-price} auction, we need to establish some {notation}. For buyer $i\in [N]$ and period $t\ge 1$, we let $p_{it}$ be the payment from buyer $i$ {in} period $t$. Further, let 
$q_{it}$ be the allocation variable: $q_{it} = 1$  if the item in period $t$ is allocated to buyer $i$ {and is zero otherwise.} We also let $b_{it}$ be the bid submitted by buyer $i$ and $r_{it}$ be the reserve price posted by the firm 
for buyer $i$ in period $t$.  We define $\bb_t = (b_{1t}, \dotsc, b_{Nt})$ {and} $\br = (r_{1t}, \dotsc, r_{Nt})$ as the vectors of bids and reserves in period $t${, respectively}. Moreover, we denote by $H_{\tau}$ the history set observed by the firm up to period $\tau$. This set includes buyers' bids and reserve prices for all $t<\tau$:
\begin{align}\label{def:history}
H_\tau = \{(x_1, \br_1,\bb_1), \dotsc, (x_{\tau-1}, \br_{\tau-1}, \bb_{\tau-1})\}\,.
\end{align} 

Below, we explain the details of the second-price auction with reserve. In period {$t\ge 1$,}
\medskip
\begin{itemize}
\item   The firm  observes the feature vector $x_t\sim \px$. In addition, each buyer $i\in[N]$ learns his valuation $v_{it}$, defined in Equation (\ref{eq:val}).
\item   \ngg{For each buyer, the firm computes reserve price $r_{it}$, as a function of history set $H_{t}$ and the  feature vector $x_t$.}
\item  Each buyer $i \in [\nbuyer]$ submits a bid of $b_{it}$. 
\item  Let $i^{\star} = \arg\max_{i \in [N]} \{b_{it}\}$. If $b_{\is t} \ge r_{\is t}$, then the item is allocated to buyer $\is$, and we have $q_{\is t} = 1$. In case of a tie, the item is allocated  uniformly at random to one of the buyers among those with the highest bid. For all buyers who do not get the item, we have $q_{it} = 0$.
\item  For each buyer $i$, if he gets the item $(q_{it} = 1)$, then he pays $p_{it} =  \max\left\{r_{it}, \max_{j\ne i} \{b_{jt}\}\right\}$. Otherwise, $p_{it} =0$. 
\end{itemize}
To lighten the notation, we henceforth use the following shorthands. For each period $t$, we let $b^+_t$ and $b^-_t$ respectively denote the highest and  second highest bids. Likewise, we define $v^+_t$ and $v^-_t$ as the highest and  second highest valuations in period $t$. We also let $r^+_t$ be the reserve price of {the} buyer with the highest bid. Therefore, $b_{\is t} = b^+_t$, $r_{\is t} = r^+_t$, and the firm receives {a} payment of $\max\{ r^+_t,b^-_t\}$ if the item gets allocated and zero otherwise.{ We assume that for all periods $t$, $b_t^+\le \M$ for some constant $\M>0$. In words, buyers submit bounded bids.}
\medskip

This version of the second-price auctions is called \emph{Lazy} auctions \citep{paes2016field}. Here, the item is allocated to a buyer with the highest submitted bid, as long as the buyer clears his reserve. In other words, the item will not be allocated to any buyer when the buyer with the highest submitted bid does not clear his reserve price. \ngg{As stated in the introduction, there is another version of second-price auctions called \emph{Eager} auctions. In Eager auctions, all the buyers that do not clear their reserve prices are eliminated first, and then the item is allocated to one of the remaining buyers that has the highest submitted bid.}
\cite{paes2016field} {showed} that these two versions do not dominate each other in terms of their yield revenue. Thus, here we focus on the Lazy auctions, as  reserve prices in these auctions can be optimized effectively; see Proposition \ref{prop:opt_reserve}. {We further note  that \cite{kanoria2017dynamic} argued that even when buyers are homogeneous, designing a learning algorithm for eager second-price auctions that can incentivize the buyers to bid truthfully is very challenging.}

The firm's decision in any period $t\ge 1$ is to find optimal reserve price $r_{it}$, $i\in [\nbuyer]$, and her objective is to maximize her \new{(cumulative)} expected revenue. Note that revenue of the firm is the total payment she collects  from the buyers over the length of the time horizon. Let  
\begin{align}
\rev_t  ~=~ \E\Big[{\sum_{i\in [N]}}
 p_{it}q_{it} \Big]~=~
\E \left[\max\{b^-_t, r^+_t\} \ind(b^+_t\ge r^+_t)\right]\, 
\end{align} 
be the expected revenue of the firm in period $t\ge 1$, where the 
  expectation is w.r.t. to the noise distribution $F$,   feature distribution $\px$, and any randomness in the bidding strategy of buyers and  learning policy used by the firm. Then, the total revenue of the firm is given by  $\sum_{t=1}^T \rev_t$.

Maximizing the firm's revenue is equivalent to minimizing her regret where the regret is defined as the difference between the firms' revenue and the maximum expected revenue that the firm could earn if she knew the preference vectors $\{\pv_i\}_{i\in [\nbuyer]}$. In the next section, we will formally define the firm's regret.

\subsection{Benchmark and Firm's  Regret}
As stated earlier, the firm's objective {is} to minimize her regret\nc{,} which is the maximum expected revenue loss 
 relative to a benchmark policy that knows the preference vectors $\{\pv_i\}_{i\in [\nbuyer]}$ in hindsight. 
 {When the preference vectors and noise distribution $F$ {are} known, to set the optimal reserves $r_{it}$, 
the benchmark policy does not need any knowledge from the history set $H_t$.}  
 {Thus, with the knowledge of  the preference vectors, all buyers are incentivized to bid truthfully against the benchmark policy. This is the case because single-shot second-price auctions are strategy proof~\citep{myerson1981optimal}. }
 
 We next characterize the benchmark policy. Let $r^{\star}_{it}$ be the reserve of buyer $i$ in period $t$ posted by the benchmark policy and following our convention, we denote by $r^{\star+}_t$ the reserve price of the buyer with the highest bid.

\begin{propo}[Benchmark] \label{prop:opt_reserve}
{If the firm knows the preference vectors $\{\pv_i\}_{i\in[\nbuyer]}$ and (fixed) noise distribution $F$,} then the optimal reserve price of buyer $i\in [\nbuyer]$ for a feature vector $x\in \cX$ is given by
 \begin{align}
 r_{i}^{\star}(x) ~=~ \arg\max_{y} \big\{y\big(1-F(y-\new{\<x, \pv_i\>})\big)\big\} ~~ i\in [\nbuyer],~ x\in \cX\,,\label{def:rstar}
 \end{align}
  and hence $r^{\star}_{it} = r^{\star}_i(x_t)$. In addition, in any period $t\ge 1$, the benchmark expected revenue 
  is given by  
\begin{align}\label{eq:BenchmarkRev}
\new{\rev^{\star}_{t}} ~= ~\E \big[\max\{v^-_t, r^{\star+}_t\} \ind(v^+_t\ge r^{\star+}_t)\big]\,,
\end{align}
 where expectation is w.r.t. to the noise distribution $F$ and the feature distribution $\px$. \end{propo}

We refer to Appendix~\ref{proof:prop-opt_reserve} for the proof of Proposition~\ref{prop:opt_reserve}. 
We {remark} that the benchmark revenue $\rev^\star_t$ is measured against truthful buyers, while the firm's revenue under our policy is measured against strategic buyers {who may not necessarily follow the truthful strategy.}

{Observe that the optimal reserve price of buyer $i$ in period $t$, denoted by  $r^{\star}_{it}$, solves the following optimization problem
 \[r^{\star}_{it}~=~ \arg\max_{y}~\{y\cdot \P{v_{it}(x_t) \ge y}\} ~=~ \arg\max_{y}\big\{y\cdot \P{\new{\< x_t,\beta_i\>} +z_{it} \ge y}\big\}\,.  \]
 This shows that the optimal reserve price of buyer $i$ does not depend on the number of buyers participating in the auction or their preference vectors.  \ngg{In other words, in (lazy) second-price auctions, when the preference vectors are known to the firm and the noise distribution is log-concave, the problem of optimizing  reserve prices can be decoupled.}\footnote{{This is not the case for the eager second-price auctions.}}
 } 
{Because of this, the benchmark, defined in Proposition \ref{prop:opt_reserve},} has a simple structure: For any feature vector $x\in \cX$, 
the optimal reserve price of buyer $i$, $r_{i}^{\star}(x)$, only depends on $\pv_i$ and feature $x$, and is independent of $\pv_j$, $j\ne i$. 

{In fact,} the benchmark policy offers the best mapping from the features to reserve prices, {where this mapping does not change with time.} 
{This is due to the fact the noise distribution $F$ and the feature distribution $\cD$ remain unaltered across time and  $\rev^{\star}_t$ is the \emph{expected revenue} of the firm in period $t$, with the expectation taken w.r.t. the context vector $x_t$ and valuation noise.}

Note that in non-contextual \new{settings}, \new{the} regret is measured against a policy that posts a single fixed optimal vector of reserve prices. {In \new{our} {contextual} setting,} by contrast, we would like to compare our learning policy with the best mapping from \new{the} feature (context) vectors to \new{the} vector of reserve prices. In addition, the benchmark's optimal mapping depends on the buyer's preference vector $\{\pv_i\}_{i\in [\nbuyer]}$; that is, the benchmark offers personalized  reserve prices. 
   Competing with such a  strong benchmark that takes into account the impact of contexts as well as the heterogeneity among buyers is one of the challenges faced by the firm.

Having defined the benchmark, we are now ready to formally define the regret of a firm's policy $\pi$. Recall that the firm's decision is to optimize {reserve prices}. To set the reserve prices optimally, the firm needs to learn the preference vectors $\{\pv_i\}_{i\in [\nbuyer]}$. To do so, the firm faces the trade-off between exploration and exploitation. {Such a trade-off}  is not the only hurdle that the firm is facing: the buyers can act strategically and interfere with the learning process of the firm by bidding untruthfully. Let us stress that the buyers' behavior not only {affects} the outcome of the current auction {but} also can impact the future outcomes. The reason is that the firm can use the history set in posting reserves. Therefore, {in general,} each buyer's bid may have a \emph{perpetual} {effect} on the firm's revenue. In Section \ref{sc:strategic}, we further elaborate on the buyers' bidding behavior. Thus, the firm's goal is to deploy a robust learning policy that limits the long-run effect of each bid and tries to {incentivize} the buyers to be truthful. 

Consider a policy $\pi$ that posts a vector of reserve prices $\br^\pi_t = (r^\pi_{1t}, \ldots, r^\pi_{Nt})$, as {a function} of history set $H_t$ observed by the firm. 
Suppose that the buyers submit bids of $\bb_t = (b_{1t}, \ldots, b_{Nt})$, $t\ge 1$, where $\bb_t $ may not be equal to {the} vector of valuations $\bv_t = (v_{1t}, \dotsc, v_{Nt})$. {The submitted bid of buyer $i$, $b_{it}$, can depend on  the learning policy used by the firm, context $x_t$, his valuation $v_{it}$, and history $H_{it}$, where
 \[H_{it} =\{(v_{i1}, b_{i1}, q_{i1}, p_{i1}),\ldots, (v_{i(t-1)}, b_{i(t-1)}, q_{i(t-1)}, p_{i(t-1)}) \}.\]}  
  Recalling our notation, we write {$r^{\pi+}_t$} to denote the reserve price, set by policy $\pi$, of the buyer with the highest bid {in period $t$.} Then, the expected revenue of the firm under policy $\pi$ in period $t$ reads as 
\begin{align}
\rev^{\pi}_t ~=~ \E \big[\max\{b^-_t, r^{\pi+}_t\} \ind(b^+_t\ge r^{\pi+}_t)\big]\,,
\end{align}
 where expectation is w.r.t. to the noise distribution $F$, {feature} distribution $\px${, and} any randomness in bidding strategy of the buyers.  
 
Then, the worst-case cumulative regret of policy $\pi$ is defined by
\begin{align}\label{regret}
\reg^\pi(T) = \max\Big\{\sum_{t=1}^T (\rev^\star_t-\rev^\pi_t ):\, \|\beta_i\|\le \maxpv, \text{ for }i\in[N],\, \supp(\px)\subseteq \cX \Big\}\,.
\end{align}
{Note that the regret of the policy $\pi$ is not a function of the feature distribution $\px$ and the feature vectors $\{\pv_i\}_{i\in [N]}$. That is, we compute the regret of the  policy $\pi$ against the worst feature distribution $\px$ and  preference  vectors $\{\pv_i\}_{i\in [N]}$.}

In the next section, we {discuss} buyers' bidding behavior. 

\subsection{{Utility-maximizing} Buyers}\label{sc:strategic}
We assume that each buyer $i\in [\nbuyer]$ is risk neutral and aims at  maximizing his (time-discounted) cumulative {expected} utility. \ngg{The utility of buyer $i$ in period $t\ge 1$ with valuation $v_{i t}$ is given by
\[u_{it} ~=~ \gamma^t(v_{it}q_{it} - p_{it})\,,\] 
where $\dis \in (0, 1)$ is a discount factor.  The discount factor highlights the fact that the firm is more patient than the buyers. For instance, in online advertising markets, advertisers are willing to show their ads {to the users} who just visited their websites.\footnote{Such a practice is known as retargeting \citep{amin2014repeated, golrezaei2017boosted}.} As another example, in cloud computing markets, the consumers would like to access enough capacity whenever they need it \citep{borgs2014optimal}.}
Note that through the allocation variables $q_{it}$, utility $u_{it}$, depends on the submitted bids of all the buyers{, $\bb_t$, and their reserve price $\br_t$} used by the firm.

\ngg{In any period $t$, each buyer $i$ would like to maximize his  time-discounted cumulative utility that he will earn in any period $\tau\ge t$,  which is defined as
\[U_{it}~=~ \sum_{\tau=t}^\infty  \E[u_{i\tau}]\,.   \] }
   We note that \cite{amin2013learning} showed that it is impossible to get a sub-linear regret when buyers are utility-maximizer and do not discount {their future utilities.}   \ngg{We further remark  that the firm does not need to know $\gamma$ as our policies are oblivious to $\gamma$. However, as we show later, our  regret bounds depend on $\gamma$.}

\emph{All buyers fully know the learning policy that the firm is using to set the reserves.}\footnote{This assumption is inspired by the literature on the behavior-based pricing where it is shown  that the {firm} can earn more revenue by committing to a pricing strategy \citep{hart1988contract, salant1989inducing}. See also \cite{aviv2008optimal, aviv2015responsive} for a similar insight.} 
\CRR{More precisely, if the policy involves randomization,  the buyers know in advance the policy and not the realization of the policy.}
Armed with this knowledge, buyers can potentially increase their future utility they earn via {bidding untruthfully.}
{Particularly, a buyer can underbid (shade) his bid by submitting  bid $b_{it}< v_{it}$, or he can  overbid  by submitting  bid $b_{it}> v_{it}$.}  
{Both shading and overbidding can potentially impact the firms' estimate of preference vectors of the buyers and this, in turn, can hurt the firms' revenue.}  
However, shading can lead to a utility loss in the current period, as by shading,  the buyer may lose an auction that he would have won by bidding truthfully. {Similarly, overbidding can result in a utility loss  in the current period, as by overbidding  the buyer might end up paying more than his valuation.}

We next {present our robust policy,} named CORP, for learning  preference vectors $\{\pv_i\}_{i\in [N]}$ through interaction with utility-maximizing buyers in {repeated}  second-price {auctions} with reserve.

\section{CORP:  A Contextual Robust Pricing Policy}\label{sec:corp}
In this section, we present our learning policy. {The description of the policy is provided in Table \ref{alg:corp}.} For reader's convenience,  we also provide a schematic representation of CORP in Figure~\ref{fig:corp}. The policy works in an episodic manner. It tries to learn the preference vectors by using Maximum Likelihood Estimation (MLE) and meanwhile sets the reserve prices based on its current estimates of the preference vectors. Episodes are indexed by $k= 1, 2, \ldots$, where the length of each episode, denoted by $\ell_k$, is given by $2^{k-1}$. Thus, episode $k$ starts in period $\ell_{k} = 2^{k-1}$ and ends in period $\ell_{k+1}-1 = 2^k -1$. Note that the length of episodes {increases} exponentially with $k$. {Throughout, we use notation $\epi_k$ to refer to periods in episode $k$, i.e., $\epi_k \equiv\{\ell_k, \dotsc, \ell_{k+1} - 1\}$.}

\begin{figure}[]
{{\centering
\includegraphics[scale = 0.6]{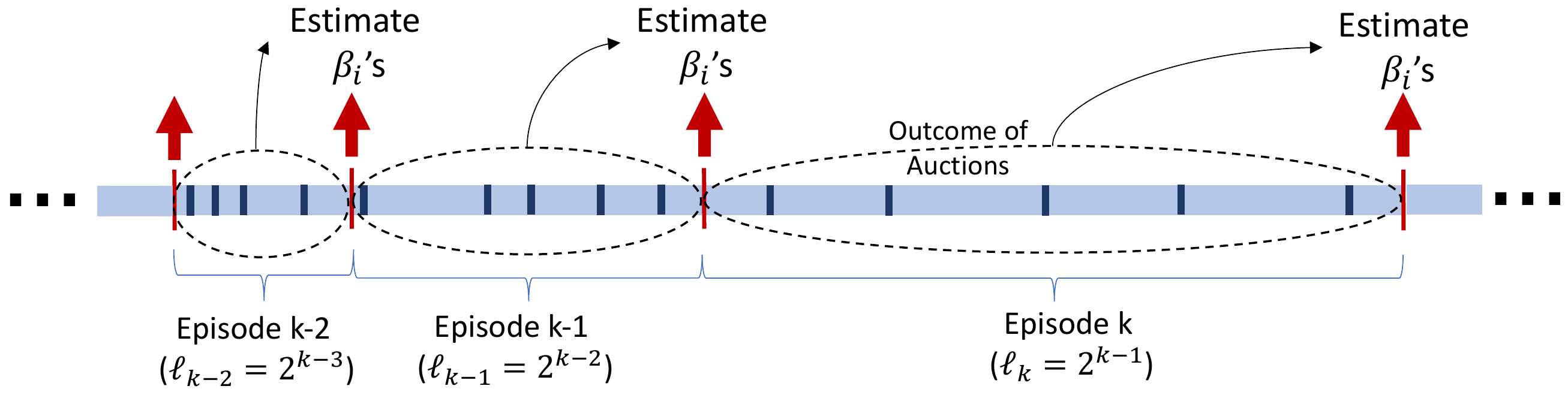}
 \caption{Schematic representation of the CORP policy. CORP has an episodic structure and updates its estimates of buyers' preference vectors at the beginning of each episode. The estimates are computed via the maximum log-likelihood method using outcomes of auctions. The dark blue rectangles show the random exploration periods. }\label{fig:corp}
}}
\end{figure}

At the beginning of each episode $k$, we  estimate the preference vectors of the buyers using the outcome of the auctions ($q_{it}$'s) in the previous episode, i.e., episode $k-1$, and we do not change our estimates during episode $k$. 
Let $\widehat \pv_{ik}$ be the estimated preference vector of buyer $i$ at the beginning of episode $k$. Then, $\widehat \pv_{ik}$ solves the following optimization problem:
\begin{align}{\hbeta_{ik}} ~=~ \underset{\|\pv\|\le \maxpv}{\arg\min\;} {{\cal L}_{ik}(\pv)}, ~~ i\in[N]\,, \label{eq:beta_estimate}\end{align}
where 
\begin{align}\nonumber
{\cal L}_{ik}(\pv) ~=~ -\frac{1}{\ell_{k-1}}\sum_{t \in \epi_{k-1}}\Big\{&q_{it} \log\big((1-F(\max\{{\b}, r_{it}\}- \<x_t,\pv\>))\big)\\&+(1-q_{it}) \log\big(F(\max\{{\b}, r_{it}\}- \<x_t,\pv\>)\big)\Big\}  \label{eq:L}
\end{align}
is the negative of  the log-likelihood function. {Here,
$\b$  refers to the maximum bids of buyers other than buyer $i$, in period $t${; that is, $\b = \max_{j\ne i} b_{jt}$.}  Then, buyer $i$ {wins} the auction {in period $t$} if and only if $b_{it} > \max\{\b,r_{it}\}$. {Throughout the manuscript, to avoid ties, we  make a simplifying assumption that the submitted bids are distinct.}  {Similarly, we define $\v = \max_{j\ne i} v_{jt}$. }
 } {Note that $F(\max\{{\b}, r_{it}\}- \<x_t,\pv\>)$ is the probability of event $\<x_t,\pv\> + z_{it} \le \max\{{\b}, r_{it}\}$, which is the probability that buyer $i$ does not win the item at time $t$, {upon bidding truthfully}.}\footnote{{Since the noise density $f$ is zero outside the interval $[-\maxn,\maxn]$, we have $F(z) = 0$, for $z\le -\maxn$ and $F(z) = 1$ for $z\ge \maxn$. \new{In addition,} $q_{it}\in\{0,1\}$. In computing the negative log-likelihood function, we adopt the convention {of} $0\times (-\infty) = 0$. }} 
 {The log-likelihood function ${\cal L}_{ik}(\pv)$ is computed after running the auctions in all the periods of episode $\epi_{k-1}$. Therefore,  the firm has access to the required knowledge to compute the log-likelihood function ${\cal L}_{ik}(\pv)$. Specifically, by the time the firm computes   ${\cal L}_{ik}(\pv)$, she has access to the submitted bids of the buyers in periods $t\in \epi_{k-1}$ as well as the reserve prices used in these periods.}

\ngg{Now, one may wonder why CORP does not use a simple mean square error estimator for which the estimate of  preference vector $\beta_i$ at the beginning of episode $k$ is given by   $\underset{\|\pv\|\le \maxpv}{\arg\min\;} \sum_{t \in \epi_{k-1}} (b_{it}- \<x_t,\pv\> )^2$.  This quadratic estimator, unlike the maximum Likelihood estimator used in CORP, 
uses the submitted bids  of buyers directly and as a result, it is rather  vulnerable to the strategic behavior of the buyers. This is so because the mean square error estimators are  sensitive to outliers and  this undesirable property would be an advantage to strategic buyers to manipulate the seller's pricing policy. In light of this observation,  to estimate the  preference vector of any buyer $i$, CORP applies a maximum Likelihood estimator that only uses the outcome of the auctions $q_{it}$, that we refer to as ``censored bids", and submitted  bids of other buyers expect buyer $i$; see the definition of the log-likelihood function in {Equation} (\ref{eq:L}). This makes the estimation procedure of the policy robust to untruthful bidding behavior of the buyers, as  untruthful bidding may not necessarily lead to a different outcome. In addition, due to this feature of the learning policy, the buyers are incentivized to bid truthfully unless they are interested in changing the outcome of the auction at the expense of losing their current utility. Later when we outline the proof of the regret bound of CORP, we further elaborate on this.  
}

After estimating the preference vectors at the beginning of {each} episode $k$, the policy proceeds to {use} its estimation to set reserve prices. In particular, the reserve price in period $t \in \epi_k$, $r_{it}$, solves \begin{align}\label{def:r}
r_{it} ~=~ \underset{y}{\arg\max\;} \big\{y\big(1-F(y-\new{\< x_t,\widehat \pv_{ik}\>})\big)\big\}\,.
\end{align}
Note that by Proposition \ref{prop:opt_reserve}, if $\widehat \pv_{ik} =  \pv_{i}$, then $r_{it} = r_{it}^{\star}$ where $r_{it}^{\star}$ is the optimal reserve price of buyer $i$ in period $t$.

\begin{policy}[t]
\begin{center}
\fbox{ 
\begin{minipage}{0.9\textwidth}
\footnotesize{
\parbox{0.95\columnwidth}{ \vspace*{3mm} \label{alg:determin_boosts}
\textbf{CORP:  A Contextual Robust Pricing}
\begin{itemize}[leftmargin = -0.1em]
\item [] {\underline{Initialization}}:  For any $k \in \mathbbm{Z}^+$, let $ \ell_{k}= 2^{k-1}$ and {$\epi_k =\{\ell_k, \dotsc, \ell_{k+1}-1\}$}. Moreover, we let $r_{i1} = 0$ 
and {$\widehat \pv_{i1} = 0$} for any $i\in[N]$.
\item[] \underline{Updating Preference Vectors}: At the start of each episode $k = 1, 2, \ldots$, i.e, at the beginning of period $t= \ell_{k}$, estimate the preference vectors, denoted by $\{\widehat \pv_{ik}\}_{i\in[N]}$,  as follows 
\begin{align}\label{optimization}
\widehat \pv_{ik} = \underset{\|\pv\|\le \maxpv}{\arg\min\;} {\cal L}_{ik}(\pv), ~~ i\in[N]\,,
\end{align}
where $ {\cal L}_{ik}(\pv)$ is defined in Equation (\ref{eq:L}).
\item[] \underline{Setting Reserves:} In  each  episode  $k = 1, 2, \ldots$, and for any 
period $t$ in this episode, i.e.,  $t\in\epi_k$,  
\begin{itemize}
\item[-] \textit{Exploration Phase:} With probability $\frac{1}{\ell_k}$, choose one of the $N$ buyers uniformly at random and offer him the item at price of $r \sim \text{\sf uniform}(0, B)$, where $\text{\sf uniform}(0, B)$ is the uniform distribution in the range $[0, B]$. For other buyers, set their reserve prices to $\infty$.
\item [-]  \textit{Exploitation Phase:} With probability $1-\frac{1}{\ell_k}$, observe the feature vector $x_t$ and  set {the} reserve of each buyer $i\in [N]$ to 
\[r_{it} = \arg\max_{y} \big\{y\big(1-F(y-\new{\< x_t,\widehat \pv_{ik}\>})\big)\big\}\,.  \]
\end{itemize}
\end{itemize}}
}
\end{minipage}
}
\end{center}  
\vspace{0.1cm}
\caption{CORP Policy}\label{alg:corp}
\end{policy}

We now discuss some of the important features of our policy.

 \begin{itemize}
  \item In each episode $k$, every period $t$ is assigned to exploitation with probability $1-{1}/{\ell_k}$, and is assigned to exploration with probability $1/\ell_k$. In the exploration periods, the firm chooses one of the buyers at random and allocates the item to him if his submitted bid is above a reserve price $r\sim \text{\sf uniform}(0, B)$ where $\text{\sf uniform}(0, B)$ is the uniform distribution in the range $[0, B]$.
  In exploitation periods, the firm exploits her current estimate of the preference vectors to set the reserve prices where  the estimates  are obtained by applying the MLE  method to the outcomes of auctions in episode $k-1$. 

{As stated earlier, in the exploration periods, we randomly choose prices. However, the firm  does not crucially use this randomness to identify the best reserve prices. Recall that to set the reserve prices, the firm uses all the data points in the previous episode, not only the data points in the exploration periods.  The main purpose of setting reserve prices randomly in the exploration periods  is to motivate the buyers to be truthful. Note that the buyer  does not know if in a given period $t$, the prices are set randomly. Thus, if he underbids in such a period, with a positive probability, he loses the opportunity to obtain a positive utility. Considering this and the fact that the buyers discount the future utilities, the randomized reserve prices  incentivize the buyers to bid truthfully.   }

\item{Another important factors that makes the CORP policy robust is its episodic structure  and impatience of buyers.
In the CORP policy, submitted bids in episode $k$ are not used in setting reserve prices until the beginning of the episode $(k+1)$. Therefore, there is always a delay until  buyers observe the effect of a bid on their reserves. Then, since buyers are impatient and maximize their discounted cumulative utility, they have less incentive to bid untruthfully. 
This is a salient property of the CORP policy that bounds the perpetual effect of each bid and, as we will see in the analysis, leads to robustness of the learning policy to \new{the strategic} behavior of buyers. 
}
\end{itemize}

\subsection{{Regret Bound of the CORP Policy}}

We now state our main result on the regret of the CORP policy.
\ngg{\begin{thm}[Regret Bound: Known  Noise Distribution]\label{thm:main}
{Suppose that Assumption \ref{assump:logcancavity} holds} and the firm knows the market noise distribution {$F$}. Then,   
the  T-period worst-case regret of the CORP policy  is \ngg{at most $O\left(Nd \left(\log(Td)\log(T) +\frac{\log^2(T)}{\log^2(1/\gamma)}\right)\right)$,} 
 where the regret is computed against the benchmark, defined in Proposition \ref{prop:opt_reserve}.
\end{thm}}

{The regret bound of CORP presented in Theorem \ref{thm:main} consists of two terms $O(Nd \log(Td)\log(T)$ and $O\left(Nd \frac{\log^2(T)}{\log^2(1/\gamma)}\right)$. The first term of the regret bound  is due to the estimation error in preference vectors; that is, this term exists  even if buyers were not strategic. The second term is due to the strategic behavior of the buyers. Observe that the second term decreases as buyers get less patient; i.e., $\gamma$ gets smaller.}

The proof of Theorem~\ref{thm:main} is proved in Appendix~\ref{proof:thm-main}. To bound the regret of our policy, 
we note that untruthful bidding  has two  undesirable effects for the firm.  First of all, {both overbidding and shading}  increase the estimation errors of the preference vectors and thus, consequently, introduce errors in the reserve prices set by the firm. 
{Second of all, bidding untruthfully can lower the second highest submitted bid as well as the reserve price of the winner, and  this reduces the firm's revenue.}   
By this observation, we divide the policy's regret into two parts, where each part captures 
 the negative consequences of  one of the aforementioned  effects.
 
  \ngg{To bound the regret associated with the first effect, in Proposition~\ref{propo:learning}, we  determine to which  extent  buyers' ``lies" impact the estimation errors of the preference vectors and the reserve prices.  We say a buyer lies when his untruthful bid changes the outcome of the auction{, $q_{it}$, for this buyer} relative to the truthful bidding. That is, we say buyer $i$  lies in period $t$ if
 $\ind(v_{it} > \max\{{\b}, r_{it}\})\neq\ind(b_{it} > \max\{{\b}, r_{it}\})$ holds.
Our notion of lies is inspired by our maximum Likelihood estimator  
  used in CORP. To estimate the  preference vector of any buyer $i$, the  maximum Likelihood estimator, defined in (\ref{eq:L}), only uses the outcome of the auctions and submitted  bids of other buyers expect buyer $i$. Thus, untruthful bidding can change our estimate of preference vector of buyer $i$ and consequently his reserve price  only  when he ``lies". 
  }

  { Having established the impact of lies, we then show that 
 the number of times that a   buyer lies in each episode is logarithmic in the length of the episode; see Proposition~\ref{propo:lies}.
This bound is derived using the property of our estimation method. As stated earlier, to influence the estimation of  the preference vectors and reserve prices,  buyers need to change the outcome of the auctions where such a change is costly for them as it can lead to utility loss. We then make use of the fact that 
buyers are utility-maximizer and discount the future to bound the number of lies.} 
{Precisely, to get this bound, we compare the long-term excess utility obtained from a lie with the instant utility loss that it causes.}  {In particular, we derive a lower bound on the utility loss  of the buyers in episode $k$  by focusing on \new{the} random exploration periods.  Note that buyers are not aware  whether a period is an exploration period. Thus, with a positive probability, any untruthful bid leads to a utility loss. We further derive an upper bound on the (future) utility gain of the untruthful bidding in episode $k$. Our upper bound is the total discounted utility that any buyer can hope to achieve in the next episodes. Thus, the bound includes potential future utility gains that a buyer can enjoy by manipulating other buyers' strategy and their reserve prices. Then, by arguing that for any \new{utility-maximizer} buyer,  the upper bound  on the utility gain should be greater than or equal to the lower bound on the utility loss, we bound the number of lies of the buyer.  By characterizing  the impact of lies in Proposition \ref{propo:learning} and bounding the number of lies in Proposition \ref{propo:lies}, we are able to bound the regret associated with the first effect, {namely the gap between the posted reserves and the optimal ones.}     }

  To bound the regret associated with the second effect, we  quantify the impact of {bidding untruthfully on the second highest bids and reserve of the winner}; see Lemma \ref{techlem3}. {When buyers bid untruthfully, the second highest bid may decrease. Further, the winner of the auction can change and this, in turn, can lower the reserve price of the winner.} Any of these events will hurt the firm's revenue. {To quantify this impact, we
  upper bound the amount of underbidding and overbidding} from each buyer using the fact that the buyer is utility-maximizer; see Proposition \ref{propo:lies}. {To do so,} we employ a similar argument that we used to bound the number of lies.

{After characterizing the regret due to both effects, we bound the total regret during each episode, and show that} the total regret  is logarithmic in the length of the  episode. The proof is completed by noting that there are $O(\log T)$ {episodes} up to time $T$, as the length of episodes doubles each time.

{
\section{Knowledge of Market Noise Distribution}\label{sec:unknownF}

In the CORP policy, we assumed that the market noise distribution is known to the seller or can be estimated sufficiently well from side data.  However, we do not always have this commodity in practice. Ideally, we prefer a pricing policy that uses such knowledge minimally, if at all. To relax this assumption, we consider setting where the market noise distribution is not fully known, but is believed to belong to a known ensemble of distributions:  
\begin{itemize}
\item \emph{\CRR{Unknown (Fixed) Distribution from a Known Location-scale Ensemble:}} 
Suppose that the unknown market noise distribution belongs to a known location-scale family of log-concave distributions. To recall, a log-scale family of distribution is a one that for a variable $Y$ belonging to this family, the distribution function of random variable $aY+b$ also belongs to this family.
Some examples include Normal, Uniform, Exponential, Logistic, and Extreme value distribution, to name a few.
We propose a policy that achieves a regret of $O\left(Nd \log(Td) \sqrt{T} + N^2 d \Big(1+\frac{1}{\log^2(1/\gamma)}\Big) \log^3(T) \right)$. (In Appendix \ref{sec:lowerB}, we further show that even when buyers are not strategic,  no other policy can have a regret better than $\Omega(\sqrt{T})$.) To design CORP-II,  we adopt the MLE estimator in CORP to also estimate the parameters of the noise distribution as well as the preference vectors, and in this sense, we follow the path pursued in~\citep{javanmard2016dynamic} for the case of the single, non-strategic buyer. We refer to Section~\ref{sec:noise-parametric} for the details.   

\item \emph{Unknown (Time-varying) Distribution from a Given Ambiguity Set:} This is a more general setting where the noise distribution is unknown and \emph{time-varying}, but belongs to a given ambiguity set of the log-concave distribution. 
 In Section~\ref{sec:rcorp}, we propose the so-called SCORP policy whose regret is of order $O\Big( N\sqrt{d  \log(Td)}\; T^{2/3} +\ngg{ \frac{N}{\log(1/\gamma)} \log(T)\; T^{1/3}} \Big) $.   

\end{itemize}

\subsection{\CRR{Unknown Distribution from a Known Location-Scale Ensemble}}\label{sec:noise-parametric}
 \CRR{Consider a location-scale class of distributions $\cF$:
\begin{align}\label{eq:Fset}
 \cF\equiv \Big\{F_{m,\sigma}:\, m\in \reals,\, \sigma \in [\underline{\sigma}, \bar{\sigma}],\, \underline{\sigma}>0, \, F_{m,\sigma}(x) = F((x-m)/\sigma)\Big \}\,,   
\end{align}
where $F$ is a known log-concave distribution with mean zero and variance one.} 
We assume that the noise  distribution is $F_{m,\sigma}\in \cF$, with \emph{unknown} parameters $m$, $\sigma$, where $m$ and $\sigma^2$ respectively corresponds to the mean and the variance of the noise distribution. Without loss of generality, we can assume that $m =0$; otherwise, the mean of the noise term can be absorbed in the features as an intercept term.   \ngg{As an example, the set $\cF$ can be class of uniform distribution with the support of $[-a,a]$ where $a\in [\underline a, \bar a]$ is unknown and $\underline a, \bar a> 0$. As another example, the set $\cF$ can be a class of (truncated) normal distributions with unknown standard deviation from an interval.}

We let $\oscale \equiv 1/\sigma$ and in the model~\eqref{eq:val}, we multiply both sides with $\oscale$. This leads to 
\begin{align}\label{eq:val2}
\tilde{v}_{it} = \<x_t, \theta_i\> + \tilde{z}_{it}, \quad i\in [N], \;\; t\ge 1\,,
\end{align}
where $\tilde{v}_{it} = \oscale v_{it}$, $\theta_{i} = \oscale \beta_{i}$ and $\tilde{z}_{it} = \oscale z_{it}$. Notably, distribution of $\tilde{z}_{it}$ is  $F$.  The valuation model~\eqref{eq:val2} is similar to our previous model where the market noises were drawn from a  known distribution $F$. 
For this setting, we propose a pricing policy, named CORP-II which is very similar to the CORP policy: It has an episodic structure, where the length of episodes grows {exponentially}, namely episode $k$ is of length $\ell_k = 2^{k-1}$.  The first $\lceil \sqrt{\ell_k}\rceil$ periods of episode $k$ are the pure exploration periods. In each of these period, a buyer is chosen in a round robin fashion (call him $i^\circ_t$) and offer him the item at price of $r_t\sim {{\sf uniform}}(0,B)$. For other buyers, we set their reserves to $\infty$. In the remaining periods of the episode, we set the reserve prices based on the current estimates of the preference vectors and the scale parameter $\alpha_0$. 
Specifically, we let $I_k$ be the set of pure exploration periods in episode $k$ (so $|I_k| = \lceil \sqrt{\ell_k}\rceil$) and form the negative log-likelihood function to estimate both $\beta_i$ and $\alpha_0$ using the outcomes of auctions in $I_k$: 
 \begin{align}
\tilde{{\cal L}}_{ik}(\theta,\alpha) ~=~ -\Big\lceil\frac{N}{|I_k|}\Big\rceil \sum_{{\{t \in I_k, i^\circ_t = i\}}}\Big\{&q_{it} \log\left(1- F(\alpha r_t- \<x_t,\theta\>)\right) \nonumber\\
&+(1-q_{it}) \log\left( F(\alpha r_t- \<x_t,\theta\>)\right)\Big\}\,.  \label{eq:L-F}
\end{align} 
The function $\tilde{\cal L}_{ik}$ is indeed the negative log-likelihood for allocation variables $q_{it}$ for $t\in I_k$, conditional on the feature vectors $x_t$, and the events $i^\circ_t = i$, $r_{it} = r_t$, assuming that buyer $i$ is truthful. The term $\lceil N/|I_k| \rceil$ comes from the fact that the buyer $i^\circ_t$ is chosen in a round robin fashion and so for a fixed $i$, there are $1/N$ fraction of the periods in $I_k$ contributing to the log-likelihood function $\tilde{\cal L}_{ik}$.

Let $(\widehat \theta_{ik}, \widehat \alpha_{0k})$ solves the following optimization problem: \begin{align}(\widehat \theta_{ik},\widehat \alpha_{0k}) ~=~ \underset{\|(\theta,\alpha)\|\le \tilde \maxpv}{\arg\min\;} {\tilde{{\cal L}}_{ik}(\theta, \scale)}, ~~ i\in[N]\,, \label{eq:beta_estimateF}\end{align}
with $\tilde\maxpv = \maxpv + 1/\underline{\sigma}$. 

\begin{figure}[]
{{\centering
\includegraphics[scale = 0.6]{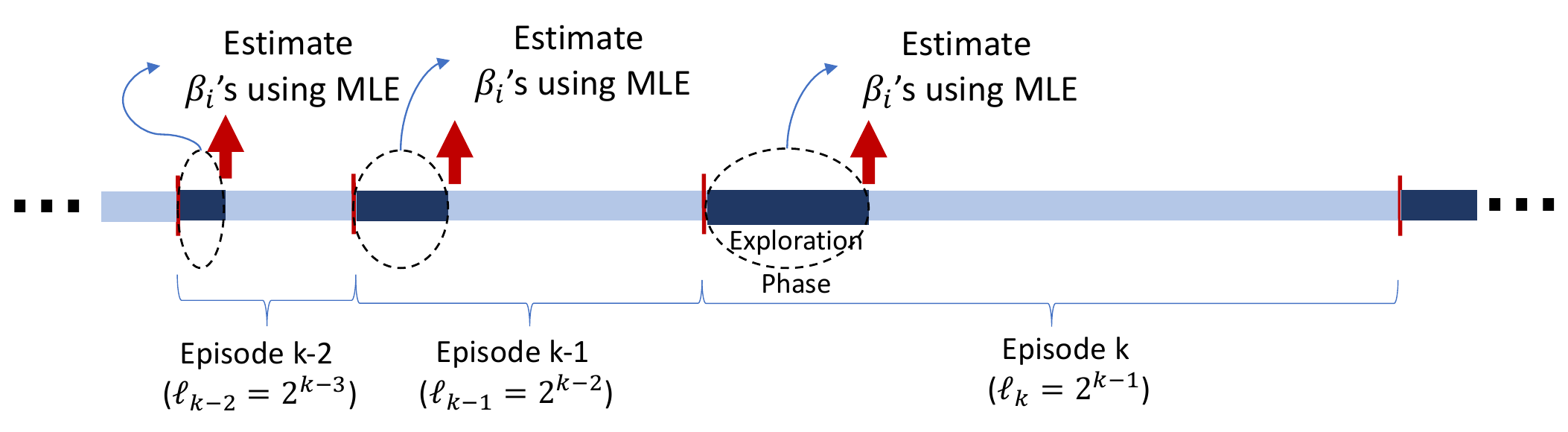}
 \caption{Schematic representation of the CORP-II policy. CORP-II has an episodic structure and updates its estimates of buyers' preference vectors at the beginning of each episode. Each episode $k$ starts with a pure exploration phase $I_k$ of length $\lceil\ell_k^{1/2} \rceil$. CORP-II  estimates preference vectors $\beta_i$ using MLE applied to the outcomes of auctions run in the previous pure exploration phase.  }\label{fig:corp-II}
} }
\end{figure}

In the exploitation phase, 
the reserve of each buyer $i\in[N]$ is set to
\begin{align}\label{eq:reserve}
r_{it} = \arg\max_y \,\{y(1-F(\widehat\alpha_{0k} y - \<x_t, \widehat\theta_{ik}\> )\}\,.
\end{align}
The formal description of CORP-II policy is given in Table~\ref{alg:corp-II}. We also provide a schematic representation of CORP-II in Figure~\ref{fig:corp-II}.
CORP-II has a very similar structure to CORP. The main difference is that unlike CORP, we have some forced (pure) exploration  periods  at the beginning of each episode at which we experiment with random prices to learn both the preference vectors and the parameters of the noise distribution. The force exploration periods  are required because  in the current setting the market noise distribution is not fully known; see lower bound  on regret in Appendix \ref{sec:lowerB}.
By comparison, in CORP, the market noise distribution $F$ is fully known to the seller and  we have much fewer of exploration periods (recall that in CORP, in each period of episode $k$, we do exploration with probability $1/\ell_k$ and so  in expectation, we have only one exploration period in each episode). 

\begin{policy}[t]
\begin{center}
\fbox{ 
\begin{minipage}{0.9\textwidth}
\footnotesize{
\parbox{0.95\columnwidth}{ \vspace*{3mm} \label{alg:determin_boosts}
\ngg{
\textbf{CORP-II:  A Contextual Robust Pricing }
\begin{itemize}[leftmargin = -0.1em]
\item [] {\underline{Initialization}}:  For any $k \in \mathbbm{Z}^+$, let $ \ell_{k}= 2^{k-1}$,  {$\epi_k =\{\ell_k, \dotsc, \ell_{k+1}-1\}$, and $\initial_k =\{\ell_k, \dotsc, \ell_k+ \lceil \sqrt{\ell_k}\rceil \}$.}  Moreover, we let $r_{i1} = 0$  
and {$\widehat \pv_{i1} = 0$} for any $i\in[N]$.
\item[] For $k=1, 2, \dotsc$,  do the following steps:
\item[] \underline{Pure Exploration Phase}: For $t\in I_k$, 
a buyer is chosen in a round robin fashion (call him $i^\circ_t$) and offer him the item at price of $r_t\sim {{\sf uniform}}(0,B)$. For other buyers, we set their reserves to $\infty$.
\item[] \underline{Updating Estimates}: At the end of the exploration phase, update the estimate of the preference vectors by applying {MLE} to the previous pure exploration phase:
\begin{align}\label{optimization2-F}
(\widehat \theta_{ik}, \widehat \alpha_{0k}) = \underset{\|(\theta,\alpha)\|\le \tilde\maxpv}{\arg\min\;} {\tilde{\cal L}}_{ik}(\theta,\alpha), ~~ i\in[N]\,,
\end{align}
where $ \tilde {\cal L}_{ik}(\theta,\alpha)$ is given by Equation (\ref{eq:L-F}).
\item [] \underline{Exploitation Phase:} For other $t\in E_k$,  
 observe the feature vector $x_t$ and  set {the} reserve of each buyer $i\in [N]$ to 
 \begin{align}\label{def:r-WC-F}
r_{it} = \underset{y}{\arg\max}\, \left\{y\big(1-F(\widehat\alpha_{0k} y- \<x_t, \widehat\theta_{ik}\>)\big)\right\}\,.
\end{align}
\end{itemize}
}}
}
\end{minipage}
}
\end{center}  
\vspace{0.1cm}
\caption{CORP-II Policy}\label{alg:corp-II}
\end{policy}

Note that by the decoupling property (as stated in Proposition~\ref{prop:opt_reserve}), the benchmark policy can focus on each buyer separately. This leads to the following optimal reserve, which optimizes the revenue obtained from buyer $i$:
\[
r^\star_i(x) = \arg\max_y \,\{y(1-F(\alpha_0 y - \<x,\theta_0\>))\}\,.
\]
Our next theorem characterizes the regret bound of CORP-II policy against a benchmark that knows the buyers' preference vectors $\beta_i$ and the true market noise distribution. 
\begin{thm}[\CRR{Regret Bound: Unknown  Noise Distribution from a Location-Scale Ensemble}] \label{thm-F}
Consider the valuation model~\eqref{eq:val} where the market noise $z_{it}$'s are generated from a distribution belonging to $\cF$, defined in~\eqref{eq:Fset}. Suppose that $F$ in the
definition of $\cF$ satisfies  Assumption~\ref{assump:logcancavity}. Then, the T-regret worst-case regret of CORP-II policy is at most 
$$
O\left(Nd \log(Td) \sqrt{T} + N^2 d \Big(1+\frac{1}{\log^2(1/\gamma)}\Big) \log^3(T) \right)\,,
$$ 
where the regret is computed against the clairvoyant that knows the buyers' preference vectors $\beta_i$ and the noise distribution. 
\end{thm}

We refer to Appendix~\ref{proof:thm-F} for the proof of Theorem~\ref{thm-F}. Note that the extra regret due of the strategic behavior of the buyer  is in the order of $O\left(\frac{d N^2  \log^3(T)}{ \log^{2}(1/\gamma)} \right)$. Interestingly, the extra regret  scales only poly-logarithmically in $T$. }
\subsection{Unknown Distribution from a Given Ambiguity Set}\label{sec:rcorp} The CORP policy presented in this paper is assumed to know the {market} noise distribution $F$. This knowledge is used in forming the log-likelihood estimator to {learn} preference vectors $\beta_i$'s and also in setting the reserves for buyers as in~\eqref{def:r}. \ngg{Furthermore, we relaxed this assumption in the previous section by assuming that $F$ is unknown and fixed and belongs to a location–scale family, and for this setting, we presented CORP-II policy.  Nevertheless, in practice, it may very well be the case that distribution $F$ is unknown and time-varying and  as  a result, it cannot be well approximated.} To address this problem, we propose a variant of the CORP policy, called Stable Contextual Robust Pricing (SCORP), which is robust against the lack of knowledge of $F$. Specifically, we consider an \emph{ambiguity} set $\cF$ of possible probability distributions for the {market} noise and propose a policy that works well for every probability distribution in the ambiguity set.

 We make the following assumption on the ambiguity set $\cF$. This assumption is analogous to Assumption~\ref{assump:logcancavity}.

\begin{assumption}[Log-concavity of the Ambiguity Set $\cF$]\label{assump1}
All functions $F\in \cF$ are log-concave. 
\end{assumption}

To be fair, in this case, we compare the regret of our policy against a benchmark policy, called  robust,  that knows the true preference vectors $\beta_i$ and the ambiguity set $\cF$ that includes $F$, but is oblivious {to} distribution $F$ itself. The robust benchmark is defined as follows:

\begin{definition} [Robust Benchmark]\label{propo:benchmark-WC}
In the robust benchmark, the reserve price of buyer {$i\in [N]$} for a feature vector $x\in \cX$ is given by
\begin{align}\label{def:rstar-WC}
r^\star_i(x) = \underset{y}{\arg\max}\,~\underset{F\in \cF}{\min}\, \left\{y\left(1-F(y-\new{\<x, \beta_i\>})\right)\right\}\,, \quad i\in {[N]}, \quad x\in \cX\,,
\end{align}
and thus $r^\star_{it} = r^\star_{i}(x_t)$ in this case.
\end{definition}

The robust benchmark is motivated by our previous benchmark presented in Proposition \ref{prop:opt_reserve}. In our previous benchmark, we show that given a distribution $F$ and context $x$, the revenue-maximizing reserve price for buyer $i$ solves  $r^\star_i(x) = {\arg\max}_y\, \left\{y\left(1-F(y-\new{\<x,\beta_i\>})\right)\right\}$. In the robust benchmark, the posted reserve prices have a similar form. However, the reserve prices which solves the optimization problem (\ref{def:rstar-WC}), are chosen  in a robust way so that the benchmark performs well  despite the uncertainty in the market noise distribution. \ngg{We will discuss the complexity of optimization problem (\ref{def:rstar-WC}) in Section \ref{sec:complexity}.}

We note from the robust benchmark, as well as our learning policy, that we will present shortly  do not aim at learning the distribution of market noise, as this distribution can vary across periods. For instance, in online advertising, the distribution of the noise can depend on many different factors including the time of the day and demographic information of the Internet users. Thus, instead of trying to learn the market noise distribution, we would like to use  reserve prices that are robust to the uncertainty in the noise distribution.

\ngg{We are now ready to present {our SCORP policy.}  This policy is a modified version of the CORP-II policy. For reader's convenience, we also provide a schematic representation of SCORP in Figure~\ref{fig:scorp}.
{Similar to the CORP-II policy,} SCORP has an episodic theme, with the length of episodes growing exponentially. As before, we denote the set of periods in episode $k$ by $\epi_k$, i.e., $\epi_k =\{\ell_k, \dotsc, \ell_{k+1}-1\}$, with $\ell_k = 2^{k-1}$. Each episode $k$ starts with a pure exploration phase of length $\length$. 
{As before, we use notation $\initial_k$ to refer to periods in the pure exploration phase of episode $k$, i.e., $\initial_k =\{\ell_k, \dotsc, \ell_k+ \length \}$.} 
During each period in {$\initial_k$}, we choose one of the $N$ buyers uniformly at random and 
offer him the item at price of $r\sim{\sf uniform}(0,B)$. For other buyers, we set their reserve prices to $\infty$. Observe that in SCORP, because distribution $F$ is time-varying,   we dedicate more periods to pure exploration than CORP-II.    In the remaining periods of the episode (i.e., $\epi_k\backslash\initial_k$), we offer the reserve prices based on the current estimates of the preference vectors which are obtained by applying the least-square estimator to the outcomes of auctions in the pure exploration phase, $\initial_k$; see Equations~\eqref{optimization2} and \eqref{def:r-WC}. {This is the exploitation phase as we set reserves based on our best guess of the preference vectors.}}
 
 \CRR{We note that the SCORP policy follows a similar structure to the LEAP policy proposed in Amin et al. (2014). Assuming that the time horizon $T$ is known, the LEAP policy designates the first $O(T^{2/3})$ periods to pure exploration and the remaining periods to exploitation. Using the doubling trick, the LEAP policy is extended to the setting with an unknown time horizon, which then admits a similar structure to SCORP. In comparison between SCORP and LEAP, it is worth highlighting a few points: 1) SCORP generalizes LEAP to the case of multi-buyers. 2) LEAP considers a setting with noiseless valuations while SCORP allows for noise component in the valuation function. As a consequence, the pricing functions differ in the two policies. SCORP uses \eqref{def:r-WC}, which involves a min-max optimization over the ambiguity set $\cF$ while in LEAP, the prices in the exploitation periods are set as $p_t = \<x,\hth\>-\eps_t$, for an appropriate choice of $\eps_t$.
3) More on the technical part, SCORP updates its estimate on the preference vectors only at the end of each exploration phase while LEAP update its estimate at each period by taking a gradient step on the prediction loss.}

\begin{figure}[]
{{\centering
\includegraphics[scale = 0.6]{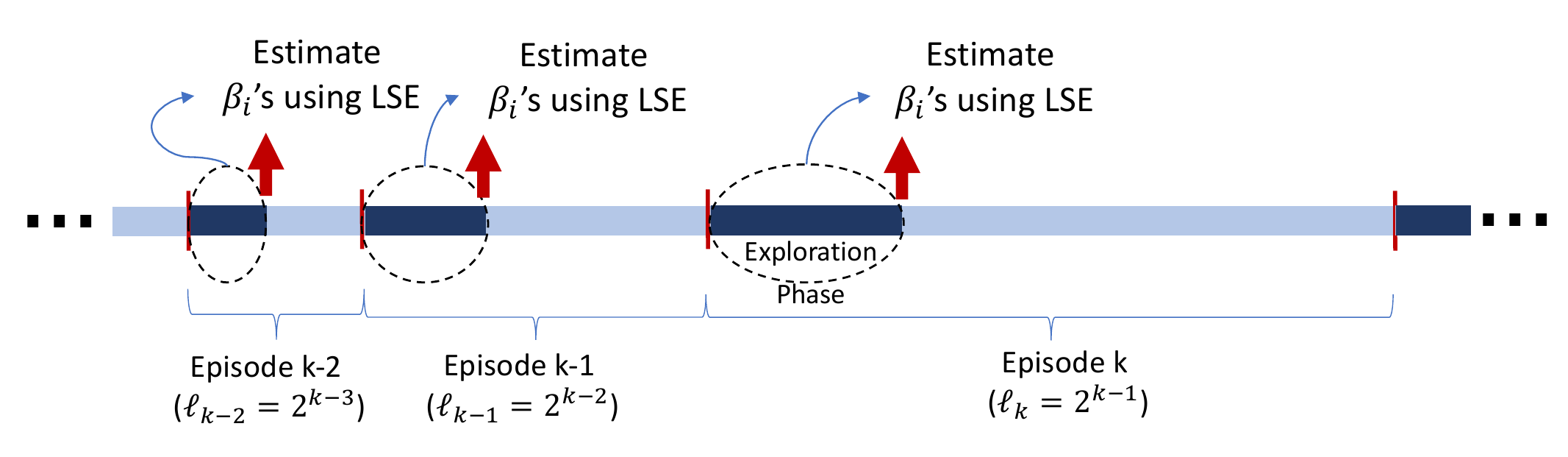}
 \caption{Schematic representation of the SCORP policy. SCORP has an episodic structure and updates its estimates of buyers' preference vectors at the beginning of each episode. Each episode $k$ starts with a pure exploration phase $I_k$ of length $\lceil\ell_k^{2/3} \rceil$. SCORP  estimates preference vectors $\beta_i$ using LSE applied to the outcomes of auctions run in the previous pure exploration phase.  }\label{fig:scorp}
} }
\end{figure}

The formal description of SCORP is given in {Table \ref{alg:scorp}}.

\begin{policy}[t]
\footnotesize{
\begin{center}
\fbox{
\begin{minipage}[htb]{0.9\textwidth}
\parbox{0.97\columnwidth}{ \vspace*{3mm}
\textbf{SCORP: Stable Contextual Robust Pricing Policy }
\begin{itemize}[leftmargin=-0.1em]
\item [] {\underline{Initialization}}:  For any $k \in \mathbbm{Z}^+$, let $ \ell_{k}= 2^{k-1}$,  {$\epi_k =\{\ell_k, \dotsc, \ell_{k+1}-1\}$, and $\initial_k =\{\ell_k, \dotsc, \ell_k+ \length \}$.} Moreover, we let {$r_{i1} = 0$} and $\new{\widehat \pv_{i1}} = 0$ for any $i\in[N]$.
\item[] For $k=1, 2, \dotsc$,  do the following steps:
\item[] \underline{Pure Exploration Phase}: For $t\in \initial_k$, choose one of the $N$ buyers uniformly at random and offer him the item at price of $r \sim \text{\sf uniform}(0, B)$. For other buyers, set their reserve prices to $\infty$.
\item[] \underline{Updating Estimates}: At the end of the exploration phase, update the estimate of the preference vectors by applying {the least-square estimator} to the previous pure exploration phase:
\begin{align}\label{optimization2}
\widehat \pv_{ik} ~=~ \arg\min_{\|\pv\|\le \maxpv} \tilde{\cal L}_{ik}(\pv), ~~ i\in[N]\,,
\end{align}
where $ \tilde {\cal L}_{ik}(\pv)$ is given by 
\begin{align}\label{eq:L_2}
\tilde \cL_{ik}(\beta) ~=~ \frac{1}{|\initial_k|} \sum_{t \in \initial_k}  (BNq_{it} - \<x_t,\beta\>)^2\,.
\end{align}
\item [] \underline{Exploitation Phase:} For $t\in \epi_k\backslash\initial_k$,  
 observe the feature vector $x_t$ and  set {the} reserve of each buyer $i\in [N]$ to 
\begin{align}\label{def:r-WC}
r_{it} = \underset{y}{\arg\max}\,~\underset{F\in \cF}{\min}\, \left\{y\big(1-F(y-\new{\<x_t, \hbeta_{ik}\>})\big)\right\}\,.
\end{align}
\end{itemize}
}
\end{minipage} 
}
\end{center}}
\vspace{0.1cm}
\caption{SCORP Policy}  \label{alg:scorp}
\end{policy}

Having presented our policy, we now highlight few important remarks about the estimation process of the policy. \ngg{(i) Since the noise distribution is unknown and time-varying, SCORP\; employs the least-square estimator (LSE) rather than the maximum Likelihood method, used in CORP and CORP-II; compare  Equations (\ref{eq:L}) and (\ref{eq:L-F}) with (\ref{eq:L_2}). To apply the  least-square estimator, similar to the CORP and CORP-II policies, SCORP\;uses the outcome of the auctions, \emph{not} the submitted bids. In particular, SCORP minimizes the loss function $\tilde {\cal L}_{ik}(\pv)$ where the loss function  is designed in a way to provide an unbiased estimator of the preference vectors under the truthful bidding strategy. To see why note that to estimate  the preference vectors, we only use the outcome of the auctions, $q_{it}$, in the exploration periods $I_k$ where in exploration periods, the prices are chosen uniformly at random in $[0, B]$. Then, 
when buyers are truthful, the expectation of $BNq_{it}$ w.r.t. the randomness in prices and the noise in the valuations is $\E[BN \frac{1}{N} \Pr[v_t\ge r_{it}] | v_t]= \E[v_t]  = \<x_t,\beta\>$. This implies that under the truthful bidding strategy, the expectation of $ BNq_{it}- \<x_t,\beta\>$ for any exploration period  $t\in I_k$ is zero. Put differently,  even under truthful bidding,  the expectation of $ BNq_{it}- \<x_t,\beta\>$ for exploitation periods $t\notin I_k$ is not zero and as a result, we only use the data in the exploration periods to estimate the preference vectors, and this, in turn, enforces the SCORP policy to dedicated $O(T^{2/3})$ periods to exploration.  
 } 
This makes SCORP robust to the strategic behavior of the buyers.  (ii) Due to uncertainty in the noise distribution, for estimation, SCORP only utilizes the auction outcomes in the exploration phase where it does price experimentation. 
This is in contrast to CORP policy where all the auction outcomes in the previous episode are used to estimate the preference vectors. {It is worth noting that in our analysis of \new{the} regret, we give up on the revenue collected during pure exploration phases and only use the outcomes of auctions in these phases to bound the estimation error of the preference vectors. }   
 (iii) So far, we argued SCORP\;is designed in a way to ensure robustness against strategic buyers. Importantly, we also note that the choice of reserve prices in the exploitation phase of SCORP makes this policy robust against the uncertainty in the noise distribution. Thus, SCORP\; is indeed {doubly robust.}

Our next result upper bounds the regret of the SCORP\, policy.

\ngg{\begin{thm}[Regret Bound: Unknown Noise Distribution from an Ambiguity Set ]\label{thm:main2}
{Suppose that Assumption \ref{assump1} holds}, and that  the market noise distribution is unknown and belongs to uncertainty set $\cF$. Then,   
the  T-period worst-case regret of the SCORP\, policy  
 is at most \[O\left( N\sqrt{d  \log(Td)}\; T^{2/3} +\ngg{ \frac{N}{\log(1/\gamma)} \log(T)\; T^{1/3}} \right) = O(N \sqrt{d  \log(Td)}\; T^{2/3})\,, \]
  where the regret is computed against the robust benchmark. 
\end{thm}} 
\ngg{Observe that due to the strategic behavior of the buyers, the firm suffers from an extra regret of  $O\left( \ngg{ \frac{N}{\log(1/\gamma)} \log(T)\; T^{1/3}} \right)$, where this regret shrinks as $\gamma$ decreases.} \ngg{We note that while the regret of the CORP-II policy is in the order of $O(d \log(Td)\sqrt{T})$, the regret of SCORP is  $O(\sqrt{d\log(Td)}\,T^{2/3})$. 
The higher regret of SCORP is mostly due to the fact that distribution $F$ is time-varying, and because of this, SCORP cannot make use of the exploratory effect of the noise. Considering this,  the SCORP policy dedicates more periods to exploration  This implies that  the SCORP policy learns preference vectors at a slower rate than  CORP-II policy. The slower learning rate is the main {reason} behind the higher regret of SCORP.}

The proof of Theorem \ref{thm:main2} is provided in Appendix~\ref{proof:thm-main2}. 

\begin{remark} SCORP\,policy provides a very general machinery to design low-regret doubly robust  learning policies, against different benchmarks.\footnote{The firm may care about other objectives apart from her revenue.  For instance, the firm might be interested in maximizing a convex combination of the  welfare and revenue, or due to contracts and deals, she might be willing to prioritize some of the buyers by offering them lower reserve prices.} To make it clear, assume that firm uses a benchmark that posts reserve price of  $r_{i}^*(x)=
\rho(\<x, \beta_i\>, \cF) $  for buyer $i$, under {context vector} $x\in \cX$. Here,  $\rho(\<x, \beta_i\>, \cF) = \arg\max_y \min_{F\in \cF} G(y, \<x, \beta_i\>, F)$, where $G: \reals\times \reals\times \cF\rightarrow \mathbb{R}$. 
 Then, as long as $\rho(\<x, \beta_i\>, \cF)$ is Lipschitz in its first argument, we can design a  low-regret doubly robust  learning policies against this benchmark by only changing the exploitation phase of the SCORP policy. Particularly, in a period $t$  in the exploitation phase of episode $k$,  we set  $r_{it} =\rho(\<x, {\widehat\beta_{ik}}\>, \cF)$; see Equation~\eqref{def:r-WC} for comparison.  \end{remark}
 
 \ngg{\subsubsection{Complexity of SCORP}\label{sec:complexity}
 In SCORP, in each exploitation period, the optimization problem (\ref{def:r-WC}) needs to be solved to set reserve prices. Here, we discuss several cases where this optimization problem is rather easy to solve. We start by presenting an example in which the ambiguity set  consists of uniform distributions. For this example, we provide a simple closed form solution for problem (\ref{def:r-WC}).  We then discuss the uniform distributions are not exceptions in the sense that for many classes of distributions, the optimization problem  of SCORP is indeed easy to solve.} 

\ngg{
 \begin{example}[Uniform Distributions] \label{example:uniform}
Assume that the ambiguity set $\mathcal F$ includes all the uniform distributions with support of form $[-a, a]$ where $a\in [\underline a, \bar a]$. The following theorem presents the optimal solution of problem (\ref{def:r-WC})  for the described $\mathcal F$. 
Before stating the theorem, let us stress again that in the setting of SCORP, the distribution $F\in \cF$ can change over time. So, for this example it means that at step $t$, the market noises are drawn from a uniform distribution $F$, with support $[-a_t,a_t]$, where $a_t\in  [\underline a, \bar a]$, is unknown and time varying. 

\begin{thm}\label{thm:uniform}
Suppose that $\mathcal F$ includes all the uniform distributions with the support of $[-a, a]$ where $a\in [\underline a, \bar a]$. Then, for any $w$, we have 
 \[\underset{y}{\arg\max}\,~\underset{F\in \cF}{\min}\, \left\{y\big(1-F(y-w)\big)\right\} =\left\{ \begin{array}{ll}
 \frac{w+\underline a}{2}&~~ \mbox{if $w\le  \underline a$};\\
        w & ~~\mbox{if $w \in (\underline a, \bar a)$};\\
         \frac{w+\bar  a}{2} & \mbox{if $w\ge  \bar a$}.\end{array} \right. \]
\end{thm}

The proof of the theorem is deferred to the appendix.  
\Halmos \end{example}
 
 In Example \ref{example:uniform}, we observe that under uniform distributions,  the robust optimization problem in (\ref{def:r-WC}) has a simple and easy-to-compute solution.  This stems from the fact that the uniform distributions enjoy a single-crossing property; see Figure \ref{fig:uniform_cdf}. To make it clear, let $F_a(\cdot)$ be the distribution of the uniform distribution in the range of $[-a, a]$ where $a>0$. Then, (i) for any $a$, $F_a(0) = \frac{1}{2}$, (ii) for any $y>0$ and $a> a'$, $F_{a}(y)< F_{a'}(y)$, and (iii) for any $y<0$ and $a> a'$, $F_{a}(y)> F_{a'}(y)$. That is, for any $a, a'$, $F_a(\cdot)$ and $F_{a'}(\cdot)$ cross each other once at $(0, 1/2)$. Having this single-crossing property, the inner optimization problem in  (\ref{def:r-WC}) is easy to solve and as a result, problem  (\ref{def:r-WC}) has a closed form solution. We highlight that uniform distributions are not the only distributions that enjoy such a property. Consider normal distributions with mean zero and variance $\sigma^2$, denoted by $F_{\sigma}(\cdot)$. Then, for any $\sigma, \sigma'$, $F_{\sigma}(\cdot)$ and $F_{\sigma'}(\cdot)$ cross each other once at $(0, 1/2)$. \CRR{In general, the location-scale families (with fixed location parameter) considered in Section \ref{sec:noise-parametric}, namely $\cF\equiv \{F_\sigma:\; \sigma\in[\underline{\sigma}, \bar{\sigma}], \;\; \underline{\sigma} > 0,\; F_\sigma = F(x/\sigma)\}$, satisfy the single-crossing property if $F$ is strictly increasing.}
  
 \begin{figure}[]
{{\centering
\includegraphics[scale = 0.1]{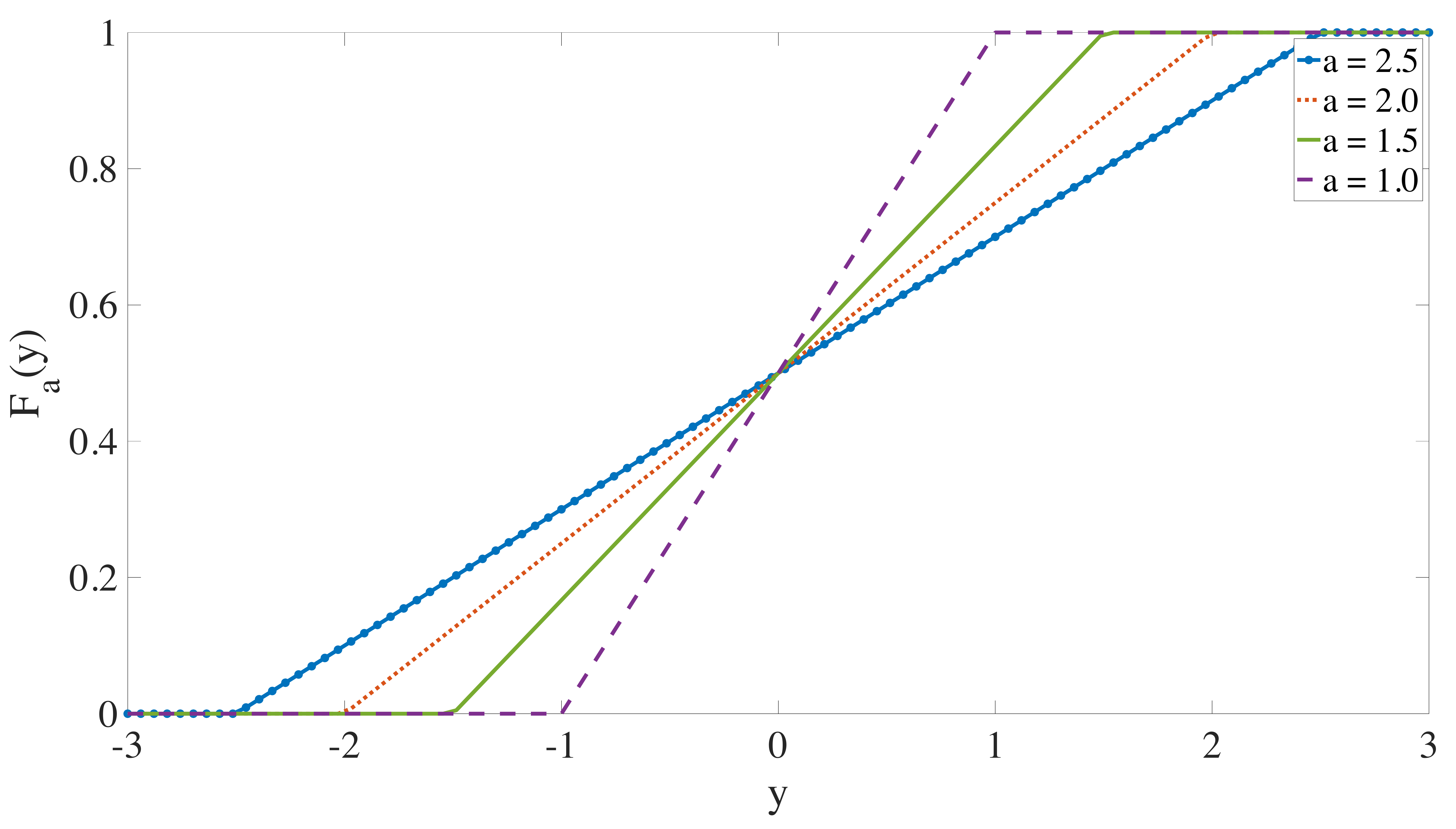}
 \caption{\ngg{Single-crossing property for the uniform distributions.} 
 }\label{fig:uniform_cdf}
} }
\end{figure}}
 \section{Conclusion}\label{sec:conclude} 
{Motivated by online marketplaces with highly differentiated products, we formulated a dynamic pricing problem in the contextual setting. In this problem, a {firm} runs repeated second-price auctions with reserve and the item to be sold {in each period} is described by a context (feature) vector. 
In our model, contextual information of an item influences buyers' valuations of that item in a heterogeneous way, via \new{buyers'} preference vectors.
Due to the repeated interaction of buyers with the {firm}, buyers have {the} incentive to game the {firm}'s policy by bidding untruthfully.  {We proposed three pricing policies to set the reserve prices of buyers.} These policies aim at learning the preference vectors of buyers in a robust way against the strategic behavior of buyers and meanwhile maximize the {firm}'s collected revenue.

The main insight behind the robustness property of our approach is that by an episodic design, we limit the long-term effect of each bid on the {firm}'s estimates of the preference vectors. Further, instead of using the bids (data) we use only the outcomes of auctions (censored data) in estimating preference vectors. Interestingly, we show that using this censored data does not hamper the learning rate while bringing in robustness property. As the granularity of real-time data increases at an unprecedented rate, we believe the ideas of this work can serve as a starting point for other complex dynamic contextual learning and decision making problems.
}

\section*{Acknowledgement}
N.G. was supported in part by the Junior Faculty Research Assistance Program at MIT and a Google Faculty Research Award. 
A.J. was supported in part by an Outlier Research in Business (iORB) grant from the USC Marshall School of Business, a Google Faculty Research Award and the NSF CAREER Award DMS-1844481. A.J. would also like to acknowledge the financial support
of the Office of the Provost at the University of Southern California through the Zumberge Fund
Individual Grant Program.  

\section*{Authors'  Biographies}

\textbf{Negin Golrezaei.} Negin Golrezaei is an Assistant Professor of Operations Management at the MIT Sloan School of Management. Her current research interests are in the area of machine learning, statistical learning theory, mechanism design, and optimization algorithms with applications to revenue management, pricing, and online markets. Before joining MIT, Negin spent a year as a postdoctoral fellow at Google Research in New York where she worked with the Market Algorithm team to develop, design, and test new mechanisms and algorithms for online marketplaces. She is the recipient of several awards including the 2018 Google Faculty Research Award; 2017 George B. Dantzig Dissertation Award; the INFORMS Revenue Management and Pricing Section Dissertation Prize; University of Southern California (USC) Ph.D. Achievement Award (2017), and USC Provost's Ph.D. Fellowship (2011). Negin received her BSc (2007) and MSc (2009) degrees in electrical engineering from the Sharif University of Technology, Iran, and a Ph.D. (2017) in operations research from USC.

\textbf{Adel Javanmard.}  Adel Javanmard is an Assistant Professor in the department of Data Sciences and Operations, Marshall School of Business at the University of Southern California. He also holds a courtesy appointment in the Computer Science Department within the USC Viterbi School of Engineering. Prior to joining USC in 2015, he obtained his Ph.D. in Electrical Engineering from Stanford University, followed by a  postdoc at UC Berkeley and Stanford, supported by a fellowship from the NSF Center for Science of Information.  Before that, he received BSc degrees in Electrical Engineering and Pure Math from Sharif University of Technology in 2009. His research interests are broadly in the area of high-dimensional statistical inference, machine learning and optimization. Adel is the recipient of several awards and fellowships, including the NSF CAREER award (2019), the Outlier Research in Business Award (2018), Dr. Douglas Basil Award for Junior Business Faculty (2018), the Zumberge Faculty Research and Innovation Fund (2017), Google Faculty Research Award (2016), the Thomas Cover dissertation award from the IEEE Society (2015), the CSoI Postdoctoral Fellowship (2015), the Caroline and Fabian Pease Stanford Graduate Fellowship (2010-2012).

\textbf{Vahab Mirrokni.}  Vahab Mirrokni is a distinguished scientist, heading the New York and Zurich algorithms research groups at Google Research. The group consists of three main sub-teams: market algorithms, large-scale graph mining, and large-scale optimization. He received his Ph.D. from MIT in 2005 and his B.Sc. from the Sharif University of Technology in 2001. He joined Google Research in 2008, after spending a couple of years at Microsoft Research, MIT and Amazon.com. He is the co-winner of paper awards at KDD'15, ACM EC'08, and SODA'05. His research areas include algorithms, distributed and stochastic optimization, and computational economics. Recently he has been working on various algorithmic problems in machine learning, online optimization and dynamic mechanism design, and distributed algorithms for large-scale graph mining.

\bibliographystyle{ACM-Reference-Format}
\bibliography{robust_learning_bib}

\newpage

\ECHead{Regret Lower Bounds, Discussion, and Proof of Statements}

 \ngg{\section{Regret Lower Bounds}\label{sec:lowerB}}\ngg{ So far, we proposed three  pricing policies (CORP, CORP-II, SCORP) for different settings of the problem (depending on how much information about the market noise distribution $F$ is available a priori to the seller). Here, we discuss various lower bounds for the regret of any pricing policy under different settings. This will shed light on the optimality gap of the proposed policies. In doing that, we focus on the simpler case of non-strategic (truthful) buyers. Nevertheless, we expect the same lower bounds carry over to the setting with strategic (utility maximizing) buyers because in such a setting, it would be even more challenging to the seller to  design a low regret pricing policy. }
 \medskip
 
 \subsection{\ngg{Dependency on Time Horizon $T$}}\label{sec:dependency-T}

\quad $\bullet$ \emph{\ngg{Unknown $F$:}}  \cite{kleinberg2003value} studied a single buyer non-contextual setting where the buyer's valuation is drawn from an unknown random distribution $F$, and showed $\Omega(T^{2/3})$ lower bound. 
To show the lower bound, they focus on  a family  of  distributions $\mathcal {F} =\{F_j^T, j\in [K]\}$, where   $K+1$ is the number of potential valuations and distribution  $F_j^T$, $j \in [K]$,  which depends on $T$, is obtained from perturbing  the base distribution. Under the base distribution, posting any price leads to the same expected revenue while under distribution $F_j^T$, posting  $(j+1)^{th}$ highest valuation as a price is optimal.  In any period,  buyer's valuation is drawn from distribution $F_j^T$, $j\in [K]$, with probability $1/K$; that is, $F= F_j^T$ with probability $1/K$. For this setting, $\Omega({T^{2/3}})$ lower bound for the worst-case regret of any pricing policy is established, where the regret is computed against a clairvoyant policy  that knows the realized valuation  distribution in any period. \CRR{We note that the lower bound obtained by \cite{kleinberg2003value} does not exactly fit into our framework. The reason is that Kleinberg and Leighton considered a benchmark that knows the valuation distributions while we consider a robust benchmark that hopes to do well for any distribution in the ambiguity set $\mathcal F$.  That being said, we believe that this lower bound highlights the challenges in learning how to set optimal prices where market noise distribution is time varying, which is the setting we considered for the SCORP policy.}
\medskip

$\bullet$ \emph{\ngg{Unknown $F$ from a Known Parametric Family:}} \ngg{\cite{broder2012dynamic} consider a single buyer setting, with a general parametric choice model and established $\Omega(\sqrt{T})$ lower bound for the worst-case regret of any pricing policy.  The main idea is to construct specific problem class with a so-called ``uninformative price".  Concretely, an uninformative price is a price such that all the demand curves (across the model parameters) intersect at that price. Such price impedes the demand learning because it does not reveal any information about the model parameter.  Now, if the optimal price corresponding to a specific choice of  model parameter is uninformative, then balancing the trade-off between  exploration and exploitation   results in the $\Omega(\sqrt{T})$ lower bound. \ngg{When the optimal price can be uninformative,  the policy can try to learn the model parameters fast by  choosing prices that are sufficiently far from the uninformative price, but in doing that a large regret incurs because the policy is posting prices far from the optimal (an uninformative) price. }}

\ngg{A similar trade-off  is also used  by \cite{kleinberg2003value} to establish lower bound for  a single buyer non-contextual setting where the buyer's valuation is drawn from an unknown  distribution $F$. When  the revenue function, i.e., $\text{Rev}(y) =y(1-F(y))$,  has a unique global maximum, they established $\Omega(\sqrt{T})$ lower bound for the worst-case regret of any pricing policy. We note that when $F$ and $1-F$ are  log-concave, the revenue function has a unique global maximum. The same property holds when $F$ is regular. To establish the lower bound, they consider a family of distributions parametrized by a single parameter such that no single price obtains a low regret with respect to all the distributions within the considered family. Similar to the setting when uninformative price exists, they show that  to obtain a lower bound of   $\Omega(\sqrt{T})$, the policy should  post a price that is far from the optimal price, as by doing so, the learning rate increases.}

\ngg{In the following, we show that when the market noise distribution is known only up to a location–scale family (similar to the setting in Section~\ref{sec:noise-parametric}), the uninformative prices do exist, even for settings with a single buyer, and hence following the same argument of~\cite{broder2012dynamic}, we have $\Omega(\sqrt{T})$ lower bound for the regret of any pricing policy.
To see this, consider $N=1$  and suppose that the feature vectors $x_t$ are drawn i.i.d. from a distribution that always takes the value of one on the first entry ($x_{t,1} = 1$ for all $t$, denoting the intercept of the model).  \CRR{Let $\cF$ be the known location-scale class of distributions given by~\eqref{eq:Fset}, with $m = 0$, namely
\begin{align*}
\cF\equiv \Big\{F_\sigma:\; \sigma\in[\underline{\sigma}, \bar{\sigma}], \;\; \underline{\sigma} > 0,\; F_\sigma = F(x/\sigma)\Big\}\,.
\end{align*}}
Fix an arbitrary $\eta$ such that $f(\eta)\neq 0$ and define $b_0$ and $b_1$ as
\[
b_0 = \frac{1-F(\eta)}{f(\eta)}\,, \quad  \quad b_1 =  \frac{1-F(\eta)}{f(\eta)} - \eta\,.
\]
Note that $b_0$ and $b_1$ are fixed. We consider preference vector $\beta = (b, 0)\in \reals^d$, with varying coefficient $b\in \reals$. Further,  consider the scaling parameter $\sigma= (b_0- b_1 +b)^{-1}$ (see definition of family of distributions $\cF$).  Looking at the probability of purchase $ 1- F(\tfrac{r}{\sigma} - \<x_t, \beta\>) = 1- F(\tfrac{r}{\sigma} - b)$, we see that $r = 1$ is an uninformative price because the demand curves ${{\rm d}}(r, b) := 1- F( \tfrac{r}{\sigma} - b)$ intersect at $1- F(b_0 - b_1)$ at price $r = 1$ and hence this price does not give any information about the value of model coefficient $b$. On the other side, it is straightforward to see that $r = 1$ is the optimal price when $b = b_1$ (i.e, the derivative of the revenue $r\times  {{\rm d}}(r,b_1)$ vanishes at $r =1$). This shows that under this setting, we have uninformative prices that are indeed optimal for a specific choice of model parameters.}

\ngg{Comparing the regret of CORP-II with the $\Omega(\sqrt{T})$ lower bound  indicates optimality of its regret (in terms of $T$), up to a logarithmic factor.}

\medskip
  
$\bullet$ \emph{\ngg{Known Distribution $F$:}} \ngg{In constructing the uninformative prices in the previous case, we used the fact that the scaling parameter $a$ of the distribution was unknown. (Recall that $a$ depended on the unknown coefficient $b$.) When the market noise distribution is fully known to the firm, then as we saw in the design of CORP, the firm can harness the randomness of the noise and use it to get free exploration of the demand parameters (by forming the log likelihood function), without having to actively randomize the prices. In such a setting, one can prove a lower bound of $\Omega(d\log T)$ by using the Van Trees inequality. (This approach is classic and we spare the details. We refer the interested reader to ~\cite{broder2012dynamic} or \cite{goldenshluger2013linear} for further details on this approach.) At a high level, the idea is to first lower bound the regret in period  $t$ in terms of the mismatch between the posted price $r_{it} $ by the policy, and the optimal price $r^\star_{it}$, namely $(r_{it} - r^\star_{it})^2$. Then, treating $r^\star_{it}$ as a function of $\<x_t, \beta\>$, i.e., $ r^\star_{it} = g(\<x_t,\beta_i\>)$ with $g(y) \equiv \arg\max_y y(1- F(y - \<x_t, \beta\>))$, one can apply the Van Trees inequality to establish a lower bound of $d/t$ for $(r_{it} - r^\star_{it})^2$, which results in a lower bound of $\Omega(d\log T)$ for the total regret incurred on horizon $T$. This approach has been followed in~\cite{broder2012dynamic} for a low-dimensional model (with single scalar parameter) to prove $\Omega(\log T)$ lower bound. However,  this approach can be extended to feature-based models using a multi-dimensional version of Van Trees inequality~\citep{gill1995applications} to get $\Omega(d\log T)$ lower bound.}

\ngg{Another approach to achieve the same lower bound of $\Omega(d \log T)$  is to relate  the regret in each period  to the minimax $\ell_2$ risk of estimating preference vectors, and then by following the Le Cam's method~\citep{Tsybakov:2008:INE:1522486}, relate it to the error in a multi-way hypothesis problem defined over a packing set of the parameter space. This error, in turn, can be lower bounded using the Fano's inequality from information theory. This type of argument to derive lower bound for minimax risks is quite standard in statistics; see, for example, \cite{raskutti2011minimax,zhang2010nearly}, and \cite{loh2017lower}. In \cite[Theorem 5]{javanmard2016dynamic}, the authors pursued this path for the setting of a single truthful buyer with sparse preference vector and established a lower bound of $\Omega(s_0 \log(T/s_0))$, where $s_0$ denotes the number of nonzero entries for the preference vector. Setting $s_0 = d$, this yields the $\Omega(d\log(T/d))$ lower bound for a general (non-sparse) preference vector. }

\ngg{Recall that our proposed CORP policy achieves a regret of $O(d\log(Td) \log (T))$. Comparing this with the above lower bound implies that CORP is optimal  up to a logarithmic factor. }

 \subsection{\ngg{Dependency on Feature Dimension  $d$}} \ngg{When talking about the dependency of the regret on the feature dimension $d$, we should also consider its dependency on $T$. For example, a naive policy that always posts random reserves, will get a regret of $O(T)$, without any dependency on the dimension $d$. The regret bounds we proved for CORP and CORP-II scale as $O(d\log d)$ in terms of the feature dimension $d$ and the regret of SCORP scales as $O(\sqrt{d\log d})$. 
 Note that a linear dependency on $d$ is inevitable when the regret scales logarithmically in $T$. Indeed, as discussed in Section~\ref{sec:dependency-T}, we have a lower bound $\Omega(d\log T)$ for the case of single truthful buyer and assuming that the market noise distribution $F$ is fully known to the firm.   
 Somewhat related, \ngg{\cite{lobel2016multidimensional} studied a pricing problem in a single buyer setting with a contextual and noise-less valuation model, where contexts are chosen by an adversary.  They show a lower bound  of 
$\Omega(d \log(1/\epsilon \sqrt{d}))$ to obtain prices that are $\epsilon$ away from the optimal prices. }}
 
\section{Discussion}\label{sec:discussion}
\ngg{In Section \ref{sec:nonlinear}, we discuss how our polices can be extended to a setting with some of the nonlinear valuation models.  We then discuss pricing with perishable data in Section \ref{sec:per}. 
 \subsection{Beyond the Linear Valuation Model} \label{sec:nonlinear}
 In this paper, we focused on linear valuation model~\eqref{eq:val}, where a buyer's valuation of a product is a linear function of the product feature with an additive noise.
 The linear feature-based model is already rich enough to capture interesting interplay between the buyers and the firm and the main components of the dynamic pricing problem. That said,
 it is quite straightforward to generalize our analysis to some of the nonlinear valuation models. Concretely, consider the following model
 \begin{align}v_{it} (x_t)= \psi(\<\phi(x_t),\pv_i\> +z_{it}) ~~~~~ i\in [\nbuyer],~ t\ge 1\,. \label{eq:val-NL}\end{align}
 where  $\phi:\reals^d \mapsto \reals^d$ is a feature mapping and $\psi:\reals\mapsto \reals$ is a general strictly increasing function. Some examples of such models are: (i) Log-log model ($\psi(x) = e^x, \phi(x) = \ln(x)$); (ii) Semi-log model $(\psi(x) = e^x, \phi(x) = x)$; (iii) Logistic model $(\psi(x) = e^x/(1+e^x)$, $\phi(x) = x)$. 
 
 Note that we can treat $\tilde{x}_t \equiv \phi(x_t)$ as new features. Moreover, by the change of variable $\tilde{v}_{it} = \psi^{-1}(v_{it})$, we arrive at the valuation model $\tilde{v}_{it} = \<\tilde{x}_t,\beta_i\> + z_{it}$, which is the same as the linear model~\eqref{eq:val} studied in this work. Letting $\tilde{b}_{it} = \psi^{-1}(b_{it})$ and $\tilde{r}_{it} = \psi^{-1}(r_{it})$,  the negative log-likelihood function for the preference vector $\beta$ given the outcome of the auctions $q_{it}$ reads as
\begin{align}\nonumber
{\cal L}_{ik}(\pv) ~=~ -\frac{1}{\ell_{k-1}}\sum_{t \in \epi_{k-1}}\Big\{&q_{it} \log\big((1-F(\max\{\tilde{b}^+_{-it}, \tilde{r}_{it}\}- \<\tilde{x}_t,\pv\>))\big)\\&+(1-q_{it}) \log\big(F(\max\{\tilde{b}^+_{-it}, \tilde{r}_{it}\}- \<\tilde{x}_t,\pv\>)\big)\Big\}\,,  \label{eq:L-NL}
\end{align}
 with $\tilde{b}^+_{-it} = \max_{j\neq i} \tilde{b}_{jt}$.
 
 Hence, in adopting the CORP policy to the nonlinear setting, we estimate preference vectors $\beta_i$ as 
 \begin{align}
 {\hbeta_{ik}} ~=~ \underset{\|\pv\|\le \maxpv}{\arg\min\;} {{\cal L}_{ik}(\pv)}, ~~ i\in[N]\,, \label{eq:beta_estimate-NL}
 \end{align}
 with ${{\cal L}_{ik}(\pv)}$ given by \eqref{eq:L-NL}.
 
 To understand the benchmark reserves,  note that similar to Proposition~\ref{prop:opt_reserve}, the decoupling principle applies and $r^\star_{it}$ solves the following optimization problem
 \begin{align*}
 r^\star_{it} &= \arg\max_y \{y\cdot \prob(v_{it}(x_t)\ge y)\} = \arg\max_y \{y\cdot \prob(\psi(\<\tilde{x}_t,\beta_i\> + z_{it})\ge y)\}\\
  &= \arg\max_y \{y\cdot \prob(\<\tilde{x}_t,\beta_i\> + z_{it}\ge \psi^{-1}(y))\} = \arg\max_y  y(1-F(\psi^{-1}(y) - \<\tilde{x}_t,\beta_i\> ))\,.
 \end{align*}
 Therefore, the exploitation phase of the CORP policy should be modified as follows. Given the estimated preference vector $\hbeta_{ik}$, the reserve price of each buyer $i\in [N]$ is set as
 \[
 r_{it} = \arg\max_y\; \; y(1-F(\psi^{-1}(y) - \<\tilde{x}_t,\beta_i\> ))\,.
 \]
 With these modifications, our analysis of CORP policy for linear valuation model carries over to the nonlinear model~\eqref{eq:val-NL}, establishing the $T$-period worst-case regret of 
 $O\left(Nd\left(\log(Td)\log(T) + \frac{\log^2(T)}{\log^2(1/\gamma)}\right)\right)$.  
 
 \subsection{Perishability of Data and Varying Coefficient Valuation Model}\label{sec:per}
 As discussed in  CORP, CORP-II, and  SCORP policies, the estimated preference vectors are updated in an episodic manner with the lengths of episodes growing over time. 
 The rationale is that as the policy proceeds, the estimates get more accurate and hence they will be used over a longer episode. 
 
 In practice, however, due to the temporal behavior of buyers, their preference vectors may vary over time and hence the historical sale data is perished after a while to be used for estimating buyer's valuation model. In such applications, it is not wise to stay with an estimated preference vector for a long time because the true preference vector may change significantly over this time frame. 
 Time-varying demand environments have also been studied recently in the literature; see e.g.~\cite{keskin2016chasing} and \cite{javanmard2017perishability}, where they are modeled via varying-coefficient demand models. For example,~\cite{keskin2016chasing} considers the setting where a firm is selling one type of product to customers arriving over time. Following a price $p_t$, the firm observes demand $D_t = \alpha_t + \beta_t p_t + z_t$, where $\alpha_t$, $\beta_t\in \reals$ are unknown model parameters  and $z_t$ is the unobserved demand noise.  The authors consider both smooth and bursty changes in a demand environment. For the case of
smooth changes, they propose a weighted least squares estimation procedure that discounts older
observations at an (asymptotically) optimal rate, whereas for the case of bursty changes, they develop a joint pricing and detection policy that continuously test if there has been a significant change in the environment. 

Closer to the spirit of our work,~\cite{javanmard2017perishability} considers a feature-based valuation model with varying coefficient model for the case of a single non-strategic buyer. This work proposes a pricing policy based on projected gradient descent that update its estimate of the model parameters at every step as the information accrues to keep up with the possible volatility in the model parameters. The regret of the policy is characterized against a clairvoyant policy that knows the sequence of the model parameters in advance, and in terms of the time, feature dimension, as well as the temporal variability of the model parameters.  

Extending our setting to the case of time-varying valuation model requires substantially different algorithms that update the estimates of the model parameters frequently. This is beyond the scope of the current work and is indeed the subject of a future work.  
}

\section{Proof of Theorem~\ref{thm:main}}
\label{proof:thm-main}
The regret of the CORP policy is the sum of its regret across all episodes. Thus, in the following, we  compute the regret incurred during an episode {$k  > 1$}. {(The regret of episode $1$ that has a length of $1$ is a constant.)}

We start with a definition.  Let 
\begin{align}\label{eq:lB}
{\lF =\inf_{|x|\le \maxn} \left\{\min \left\{ -\log''F(x), -\log''(1-F(x)) \right\}\right\}\,,}
\end{align} 
{where $\log''F(x) = \frac{d^2}{dx^2} (\log (F(x)))$ and $\log''(1-F(x)) = \frac{d^2}{dx^2} (\log (1-F(x)))$.} 
Note that $\lF$ is a measure of ``flatness" of function $\log F$.
Because of log-concavity of $F$ and $1-F$ (cf. Assumption~\ref{assump:logcancavity}), we have $\lF > 0$. %

Recall that in the CORP policy, at the beginning of each episode $k>1$,  
the preference vectors $\beta_i$ are estimated via optimizing the log-likelihood function corresponding to {the outcomes of auctions} in the previous episode; see  Equation (\ref{eq:beta_estimate}). Now, consider buyer $i$ that bids untruthfully in period $t \in \epi_{k-1}$.  \ngg{Assume {for the moment} that bids of other buyers in this period are fixed.} Then, the untruthful bid of buyer $i$ in  period $t \in \epi_{k-1}$ may influence the estimation of his preference vector in episode $k$ only when his untruthful bid changes  the allocation of the item in this period, i.e., $\ind(v_{it} > \max\{{\b}, r_{it}\})\neq\ind(b_{it} > \max\{{\b}, r_{it}\})$. This is the case because the preference vectors are estimated using the outcome of the auctions and \emph{not} the submitted bids. When $\ind(v_{it} > \max\{{\b}, r_{it}\})\neq\ind(b_{it} > \max\{{\b}, r_{it}\})$ holds, 
we say buyer $i$ ``lies" in period $t$.   
 For each buyer $i\in [N]$, we further define the set of ``lies" in episode $k-1$, indicated by $\Lie_{ik}$, as follows:  
\begin{align}\label{def:lie}
\Lie_{ik} = \Big\{t: ~
t\in \epi_{k-1}
,\,  \ind(v_{it} > \max\{\b, r_{it}\})\neq\ind(b_{it} > \max\{\b, r_{it}\}) \Big\}\,.
\end{align}  
In other words,  $\Lie_{ik}$ consists of all the periods in episode $k-1$ in which  buyer $i$ lies. 
{We note that the set of lies in episode $k-1$, $\Lie_{ik}$, depends on the reserve prices $r_{it}$, $t\in E_{k-1}$, where the reserve prices are (mostly) set  using the outcome of the auctions in episode $k-2$. Because of this dependency,  $\Lie_{ik}$ may also depend on all the submitted bids in episodes $1, 2, \ldots, k-1$. However, we will show that regardless of the values of $r_{it}$'s, the size of $\Lie_{ik}$ is logarithmic in the length of episode $k-1$; see Proposition~\ref{propo:lies}.}

Next, we quantify the adverse effect of lies on the {firm}'s estimates of the preference vectors. In particular, the
 next proposition provides an upper bound on the estimation error of $\hbeta_{ik}$ in terms of the number of samples used in the log-likelihood function ($\ell_{k-1}$),
the dimension of the feature vector $(d)$, and the number of lies $(|L_{ik}|)$. Proof of Proposition~\ref{propo:learning} is deferred to Section~\ref{proof:propo-learning}. 

\begin{propo}[Impact of Lies on Estimated Preference Vectors]\label{propo:learning}
Let $\hbeta_{ik}$ be the solution of the optimization problem~\eqref{eq:beta_estimate}.
Then, {under Assumption~\ref{assump:logcancavity}}, there exist constants $c_0$, $c_1$, and $c_2$ such that for $\ell_{k-1}\ge c_0 d$, with probability at least $1 - d^{-0.5}\ell_{k-1}^{-1.5}- 2e^{-{c_2}\ell_{k-1}}$, we have 
\begin{align}\label{eq:estimator}
\|\hbeta_{ik} - \beta_i\|^2\le \frac{c_1\ajb{d^2}}{\lF^2}\left(\left(\frac{|\Lie_{ik}|}{\ell_{k-1}}\right)^2 + \frac{\log(\ell_{k-1}d)}{\ell_{k-1}}\right)\quad i\in[N]\,, 
\end{align}
where $\lF$ is defined in Equation (\ref{eq:lB}). 
\end{propo}

{We note that} the estimation {error} of {$\hbeta_{ik}$'s affects} the {{firm}'s} regret, as reserve prices are set based on these estimates.
By Proposition \ref{propo:learning}, to keep our estimation errors  small, the buyers should not have the incentive to lie in too many periods.
In the next proposition,  we show that for each episode $k$,  the number of lies from a buyer is at most logarithmic in the length of the episode. 

{There is another way that bidding untruthfully can impact the {firm}'s regret.  Recall that in each period $t$, the {firm} collects the revenue of $\max\{b_t^-, r_t^+\}$ if the highest {buyer} clears his reserve. Then, by bidding untruthfully, 
 the second highest bid $b_t^-$ may go down. Further, the winner can change, and this, in turn, can lower reserve price of the winner, $r_t^+$.
  To bound this impact of untruthful bidding, in the following proposition
we bound the amount of underbidding from buyers who do not win an auction and the amount of overbidding from buyers who win an auction. Precisely, we  show that  the total amount of underbidding from each buyer $i$, in all periods $t\in \epi_k$ that he does not win the auction, is at most logarithmic in the length of that episode. We further show that  the total amount of overbidding from each buyer $i$, in all periods $t\in \epi_k$ that he  wins the auction, is at most logarithmic in the length of that episode.  }

\ngg{\begin{propo}[Bounding the Number of Lies]\label{propo:lies}
Denote by $\shade_{it}$ and {$o_{it}$} the amount of shading and {overbidding} from buyer $i\in [N]$ in period $t$, i.e., $s_{it} = (v_{it} - b_{it})_+$, and $o_{it} = (b_{it}-v_{it})_+$, {where $(y)_+$ is $y$ when $y\ge 0$ and zero otherwise.}  Then, there exist constants $c_3, c_4, \ldots, c_9$\footnote{{The constants $c_5 $, $c_6$, and $c_7$  depend on $B$. {Constants $c_8$ and $c_9$ depend on $\M$. (Recall that $\M$ is the bound on submitted bids.)}}}, such that for  any fixed $0\le \delta\le 1$, with probability at least {$1-(\delta+1)/\ell_{k-1}$}, the following holds:
\begin{align}
|\Lie_{ik}|~\le ~ c_3\log(\ell_{k-1}/\delta)+ \ngg{c_4\frac{ \log(\ell_{k-1})}{\log(1/\gamma)} +c_5 \frac{\log(N)}{\log(1/\gamma)}}\quad i\in [N]\,.\label{claim:lies} 
\end{align} 
Further, we have that {with probability at least $1 - 1/\ell_{k-1}$,}
\begin{align}
\sum_{{t\in \epi_{k-1}}} \shade_{it} (1-q_{it})~&\le~ \ngg{c_6\frac{ \log(\ell_{k-1})}{\log(1/\gamma)} +c_7 \frac{\log(N)}{\log(1/\gamma)}}\quad {i\in [N]}\,,\label{claim:shades}\\ 
\sum_{{t\in \epi_{k-1}}} o_{it} q_{it}~&\le~ \ngg{c_8\frac{ \log(\ell_{k-1})}{\log(1/\gamma)} +c_9 \frac{\log(N)}{\log(1/\gamma)}}\quad {i\in [N]}\,.\label{claim:overbids} 
\end{align}
\end{propo}}
Proof of Proposition~\ref{propo:lies} is given in Section~\ref{proof:propo-lies}. The main idea of the proof is to compute the excess utility that a strategic buyer can earn in {the next episodes} by {bidding untruthfully} in the current episode, and compare it with the utility loss that he suffers in the current episode because of his strategic behavior. The result then follows by using the fact that for a utility-maximizing buyer, the net excess  utility should be nonnegative.

Up to here, we have established the impact of lies on our estimation, bounded the number of lies {and the amount of underbidding from buyers}. Next, using these results, we present a lower bound on the expected revenue of  our policy in any period $t\in \epi_k$. We drop the superscript $\pi$ in our notation as it is clear from the context.

 For each period $t$, we define a random variable $\xi_t$ that {takes} values in $\{0,1\}$, with $\xi_t = 1$ if the firm is in the exploitation phase and  
  $\xi_t  = 0$ otherwise. From the description of our policy, for any period $t$ in episode $k$,  ($ t\in \epi_k$), we have $\prob(\xi_t = 0) = 1/\ell_k$. 
We first lower bound the firm's expected revenue in an  exploitation period $t \in \epi_k$. Recall that in {an} exploitation period, the firm runs a second-price auction with reserve. Thus, we have
\begin{align}\label{seller1}
\rev_t ~\ge~ \prob(\xi_t = 1) \E[\max\{b_t^-,r_t^+\} \ind(b_t^+\ge r_t^+)]\,,
\end{align} 
where the expectation is w.r.t. the randomness in the submitted bids.
Since in each period $t$, at most one of the  buyers gets the item, we can rewrite~\eqref{seller1} as follows: \begin{align}
\rev_t &~\ge~ \prob(\xi_t = 1)  \sum_{i=1}^N \E\left[\max\{b_t^-,r_{it}\} \ind(b_{it} > \max\{b_t^-,r_{it}\})\right]\nonumber\\
& ~=~ {\big(1 - \frac{1}{\ell_{k}}\big)}  \sum_{i=1}^N \E\left[\max\{b_t^-,r_{it}\} \ind(b_{it} > \max\{b_t^-,r_{it}\})\right]\,. \nonumber
\end{align} 
Next, we compare  $\rev_t$ with the expected revenue of the benchmark in period $t$, $\rev^\star_t$. 
Recalling~\eqref{eq:BenchmarkRev}, we have 
\begin{align}\label{benchmark1}
\rev^\star_t ~=~ \sum_{i=1}^N \E\left[\max\{v_t^-,r^\star_{it}\} \ind(v_{it} > \max\{v_t^-,r^\star_{it}\})\right]\,.
\end{align} 
Therefore, the regret of the policy in period $t$ can be upper bounded as
\begin{align}
&\rev^\star_t -\rev_t ~\le ~ {\big(\frac{1}{\ell_{k}}\big)} \rev^\star_t \nonumber\\
 &+{{\big(1 - \frac{1}{\ell_{k}}\big)}}  \sum_{i=1}^N \E\bigg[\max\{v_t^-,r^\star_{it}\} \ind(v_{it} > \max\{v_t^-,r^\star_{it}\}) 
-\max\{b_t^-,r_{it}\} \ind(b_{it} > \max\{b_t^-,r_{it}\})\bigg]\nonumber \\
&~\le~ {\big(\frac{B}{\ell_{k}}\big)}  +  {{\big(1 - \frac{1}{\ell_{k}}\big)}} \sum_{i=1}^N \E\bigg[\max\{v_t^-,r^\star_{it}\} \ind(v_{it} > \max\{v_t^-,r^\star_{it}\}) 
-\max\{b_t^-,r_{it}\} \ind(b_{it} > \max\{b_t^-,r_{it}\})\bigg]\,,\label{regret1}
\end{align} 
{where in the last equation, we used the fact that $\rev^\star_t \le B$.}
We break down the second expression in ~\eqref{regret1} into two terms:
\begin{eqnarray}
\Delta_{1,t}&=& \sum_{i=1}^N \bigg[\max\{v_t^-,r^\star_{it}\} \ind(v_{it} > \max\{v_t^-,r^\star_{it}\}) 
-\max\{v_t^-,r_{it}\} \ind(v_{it} > \max\{v_t^-,r_{it}\})\bigg]\,,\label{Delta-1t}\\
\Delta_{2,t}&=& \sum_{i=1}^N \bigg[\max\{v_t^-,r_{it}\} \ind(v_{it} > \max\{v_t^-,r_{it}\}) 
-\max\{b_t^-,r_{it}\} \ind(b_{it} > \max\{b_t^-,r_{it}\})\bigg]\,.\;\label{Delta-2t}
\end{eqnarray}
Using our notation, Equation~\eqref{regret1} can be rewritten as:
\begin{align}\label{regret0}
\rev^\star_t - \rev_t\le {\frac{B}{\ell_{k}}} + {{\big(1 - \frac{1}{\ell_{k}}\big)}} \E[\Delta_{1,t} + \Delta_{2,t}]\,.
\end{align}

In the sequel, we will bound each term $\Delta_{1,t}$ and $\Delta_{2,t}$ separately. But before proceeding, let us pause to explain these terms and the intuition behind their definition.
The regret of the firm's policy is due to two factors: 
\begin{enumerate}
\item \textbf{Mismatch between $\boldsymbol \beta_{i}$ and $\boldsymbol \hbeta_{ik}$:} The {mismatch} between the true preference vectors $\beta_i$ and the estimation $\hbeta_{ik}$ {leads to}  a difference between the
benchmark reserves ($r^\star_{it}$) and the posted reserves by the firm ($r_{it}$).  The term $\Delta_{1,t}$ captures this factor and its effect on the regret. We will use {Proposition~\ref{propo:learning} along with} our {first}  result in Proposition \ref{propo:lies} to bound $\Delta_{1,t}$.

\item \textbf{Mismatch between $\boldsymbol v_t^-$ and $\boldsymbol b_t^-$ and change of the winner:} Note that the benchmark revenue $\rev^\star_t$ is measured against truthful buyers, while the firm's revenue under our policy is measured against strategic buyers.
The strategic behavior of buyers not only affects the quality of estimates $\hbeta_{ik}$ (and therefore the reserves $r_{it}$) but it may also affect the firm's revenue via another quite subtle factor. 
Indeed, {due to the strategic behavior of buyers, the second highest bid
might go down or the winner of the auction might change from the case of truthful buyers and this may decrease the reserve of the winner. The decrease in the second highest bid or the reserve price of the winner can hurt the firm's revenue. 
The term $\Delta_{2,t}$ captures these effects.} We will use {our second result} in Proposition \ref{propo:lies} to bound $\Delta_{2,t}$.
\end{enumerate} 
\smallskip

\noindent{\bf Bounding $\boldsymbol{\Delta_{1,t}}$:} We now move to bounding $\Delta_{1,t}$.
Recall that
\begin{align*}
\E[\Delta_{1,t}]~= 
&\sum_{i=1}^N \E\bigg[\max\{v_t^-,r^\star_{it}\} \ind(v_{it} > \max\{v_t^-,r^\star_{it}\}) 
-\max\{v_t^-,r_{it}\} \ind(v_{it} > \max\{v_t^-,r_{it}\})\bigg]\,.
\end{align*}
\ngg{Here, the expectation is w.r.t. the randomness in the buyers' valuations and potential randomness in  reserve prices $r_{it}$ and $r^\star_{it}$.}
Note that {the first expression inside the summation} denotes the firm's revenue when buyer $i$ wins the auction with reserve $r^\star_{it}$, while the second expression is the analogous term when the buyer $i$'s reserve is $r_{it}$. Further,
conditional on the feature vector $x_t$, reserves $r^\star_{it}$ and $r_{it}$ are independent of $v_t^-$, and the right-hand side {of the last equation} can be written in terms of function $W_{it}(r)$, defined below:
\begin{align}\label{def:W1}
W_{it}(r) ~\equiv ~\E\Big[\max\{v_t^-,r\}  \ind(v_{it}\ge \max\{v_t^-,r\} )\Big| x_t\Big]\,,
\end{align}
where the  expectation  is with respect to valuation noises, conditional on $x_t$.  {By the law of iterated expectation, we can write $\E[\Delta_{1,t}]$ in terms of $W_{it}(r)$. More specifically, we first take the expectation conditional on $x_t$ and then take the expectation w.r.t. $x_t$.  }

Hence,  
\begin{align}
\E[\Delta_{1,t}] &~=~ \E[\E[\Delta_{1,t} |x_t]]\nonumber\\
&~=~  \sum_{i=1}^{\nbuyer} \E[W_{it}(r_{it}^\star) - W_{it}(r_{it})] \nonumber\\
&~=~ \sum_{i=1}^{\nbuyer}\E\left[W'_{it}(r_{it}^\star) (r^\star_{it} - r_{it}) -\frac{1}{2} W_{it}''(r) (r^\star_{it} - r_{it})^2\right]\,,\label{eq:taylor2nd}
\end{align}
for some $r$ between $r_{it}$ and $r^\star_{it}$.\footnote{This follows from the Remainder theorem for Taylor's expansion.} We will make use of the following two lemmas to bound  the above equation. The proof of all technical lemmas in this section are deferred to Section \ref{sec:technical}. 

\begin{lemma}[Property of Function $W_{it}$]\label{techlem1}
For the benchmark reserve $r^\star_{it}$, given by~\eqref{def:rstar}, and function $W_{it}(r)$, given by~\eqref{def:W1}, we have $W'_{it}(r^\star_{it}) = 0$. Further, for any $r$ between $r_{it}$ and $r^\star_{it}$, we have $|W''_{it}(r)|\le c$, for a constant $c>0$.
\end{lemma}

\begin{lemma}[Errors in Reserve Prices]\label{techlem2} {Let $t\in \epi_{k}$ with $\xi_t =1$.}
Then, conditioned on the feature vector $x_t$ and $\hbeta_{ik}$,   the following holds:
\begin{align}\label{r-B}
|r^\star_{it}- r_{it}| \le |\<x_t,\beta_i-\hbeta_{ik}\>|\,,
\end{align}
{where $r^\star_{it}$ and $r_{it}$ are defined in 
~\eqref{def:rstar} and \eqref{def:r}, respectively.}
\end{lemma}

{Applying Lemma~\ref{techlem1} in Equation (\ref{eq:taylor2nd})}, we get
\begin{align}
\E[\Delta_{1,t} ] &~
{\le} ~ \frac{c}{2}\sum_{i=1}^\nbuyer \E[(r^\star_{it} - r_{it})^2]\nonumber\\
 &~\le~ \frac{c}{2} \sum_{i=1}^\nbuyer \E\Big[\E\Big[(r^\star_{it} - r_{it})^2\Big|x_t,\hbeta_{ik}\Big]\Big]\nonumber\\
 &\le \frac{c}{2} \sum_{i=1}^\nbuyer \E[\<x_t,\beta_i -\hbeta_{ik}\>^2] \,,\label{Delta-1t-2}
\end{align}
where in the last step, we employed Lemma~\ref{techlem2}. {We next further simplify the r.h.s. of the last equation. 
{By using the fact that} our estimate $\hbeta_{ik}$ is constructed using samples from the previous episode and consequently is independent from the current feature $x_t$, we get}
\begin{align}\label{eq:Sigma-E}
\E[\<x_t,\beta_i -\hbeta_{ik}\>^2]~{=}~ \E[\<\beta_i-\hbeta_{ik}, \Sigma (\beta_i - \hbeta_{ik})\>] ~\le~ \frac{c_{\max}}{\ajb{d}} \E[\|\beta_i - \hbeta_{ik}\|^2]\,,
\end{align}
where $\Sigma = \E[x_t x_t^\sT]$ is the {second-moment} matrix of features $x_t$, and $c_{\max}/\ajb{d}$ is the bound on the maximum eigenvalue of covariance $\Sigma$. \footnote{\ajb{Note that by our normalization, the sum of  eigenvalues of $\Sigma$ would be ${{\rm trace}}(\Sigma) = \E[\|x_t\|^2]\le 1$, and that is why the eigenvalues are scaled by $1/d$.}}
Here, the first inequality follows from taking the expectation w.r.t. $x_t$ and using the fact that $x_t$ and $\hbeta_{ik}$ are independent; the second inequality follows from the definition of the maximum eigenvalue. %

Putting Equations~\eqref{Delta-1t-2} and~\eqref{eq:Sigma-E} together, we get
\begin{align}\label{Delta-1t-3}
\E[\Delta_{1,t}] ~\le~ \frac{c'}{\ajb{d}}\sum_{i=1}^\nbuyer {\E[\|\beta_i - \hbeta_{ik}\|^2}]\,.
\end{align}
 Here, $c' = \frac{1}{2} c c_{\max}$.
 
\noindent{\bf Bounding $\boldsymbol{\Delta_{2,t}}$:} We next proceed with bounding $\Delta_{2,t}$. To do so, we use the following preliminary  lemma.

\begin{lemma}\label{techlem3}
Let $v_t^-$ and $b_t^-$, respectively, denote the second highest valuation and the second highest bid submitted by the buyers in the CORP policy. 
Denote by $\shade_{it}$ and {$o_{it}$} the amount of shading and {overbidding} from buyer $i\in [N]$ in period $t$, i.e., $s_{it} = (v_{it} - b_{it})_+$, and {$o_{it} = (b_{it}-v_{it})_+$}. Then, 
\begin{align}
{(\vm - \bm)_+} \le \max\Big\{\shade_{it}(1-q_{it}):\, i\in [N]\Big\}\,. \label{second-claim}
\end{align}
{Further, for any buyer $i$ with $q_{it} = 0$, the following holds:}
\begin{align}
{(\b - \v)_+} \le \max\Big\{o_{jt} q_{jt}:\, j\in [N], {j \ne i}\Big\}\,.\label{second-claim-2}
\end{align}
\end{lemma}
Proof of Lemma~\ref{techlem3} is given in {Section}~\ref{proof:techlem3}.

Note that $\Delta_{2,t}$, given by~\eqref{Delta-2t}, can be written as
\if false\begin{align}
\Delta_{2,t}= \sum_{i=1}^N \bigg[\max\{v_t^-,r_{it}\} \ind(v_{it} > \max\{\v,r_{it}\}) 
-\max\{b_t^-,r_{it}\} \ind(b_{it} > \max\{\b,r_{it}\})\bigg]\,.\label{Delta-2t-1}
\end{align} 
\fi
{\begin{align}
\Delta_{2,t}~=~ \sum_{i=1}^N \bigg[\max\{\vm,r_{it}\} \ind(v_{it} > \max\{\v,r_{it}\}) 
-\max\{\bm,r_{it}\} \ind(b_{it} > \max\{\b,r_{it}\})\bigg]\,.\label{Delta-2t-1}
\end{align} }
Define $\Lie_{k+1} = \cup_{i=1}^N \Lie_{i(k+1)}$, where $\Lie_{i(k+1)}$, given by~\eqref{def:lie}, denotes the
set of periods in episode $k$ that buyer $i$ lies. For $t\in \Lie_{k+1}$, we avail the trivial bound
\begin{align}\label{Delta1Lie}
\Delta_{2,t} \le B\,,
\end{align}
which is true because the revenue of the benchmark in any period $t$ is at most $v_t^+ \le B$. For $t\notin \Lie_{k+1}$, we have $\ind(b_{it} > \max\{\b,r_{it}\}) = \ind(v_{it} > \max\{\b,r_{it}\})$, for all $i\in [N]$. 
 Therefore, we can write 
{\begin{align}
\E[\Delta_{2,t} &\, \ind(t\notin \Lie_{k+1}) ]~= \nonumber\\
&\sum_{i=1}^N \E\bigg[\max\{\vm ,r_{it}\} \ind(v_{it} > \max\{\v,r_{it}\}) 
-\max\{\bm ,r_{it}\} \ind(v_{it} > \max\{\b,r_{it}\})\bigg]\,.\label{Delta2-n1}
\end{align}
}
To bound the r.h.s of~\eqref{Delta2-n1}, we use {the fact} that for any two indicators $\chi_1$, $\chi_2$ and any $a, b\ge 0$, we have
$a\chi_1-b\chi_2 \le  (a-b)\chi_2 + a\chi_1(1-\chi_2)\,.$ 
{Applying  this inequality to~\eqref{Delta2-n1} with $\chi_1 = \ind(v_{it} > \max\{\v,r_{it}\}) $,  $\chi_2= \ind(v_{it} > \max\{\b,r_{it}\})$, $a =\max\{\vm,r_{it}\}$, and $b = \max\{\bm,r_{it}\}$, we get}
{\begin{align}
&\E[\Delta_{2,t} \, \ind(t\notin \Lie_{k+1}) ] \nonumber\\\nonumber
&\le \sum_{i=1}^N \E\big[(\max\{\vm, r_{it}\} - \max\{\bm, r_{it}\}) \ind(v_{it} > \max\{\b,r_{it}\})\big] 
\\\nonumber
&+\sum_{i=1}^N \E\left[\max\{\vm,r_{it}\} \ind\big(\max\{\v,r_{it}\}<v_{it} < \max\{\b,r_{it}\}\big)\right]\,.\nonumber \end{align}}
{Then, by using the fact that  $\max\{a,c\} - \max\{b,c\}~\le~ (a-b)_+$, we get}
\begin{align*}
&\E[\Delta_{2,t} \, \ind(t\notin \Lie_{k+1}) ] \nonumber\\\nonumber
&\le \sum_{i=1}^N \E\big[{(\vm -\bm)_+} \ind(v_{it} > \max\{\b,r_{it}\})\big] + \sum_{i=1}^N \E\left[\max\{\vm,r_{it}\} \ind\big(\max\{\v,r_{it}\}<v_{it} < \max\{\b,r_{it}\}\big)\right]\nonumber\\
&\le   {\E\bigg[
(\vm-\bm)_+ \sum_{i=1}^N \ind(v_{it} > \max\{\b,r_{it}\}) \bigg] }+ B\sum_{i=1}^N  \prob\big(\max\{\v,r_{it}\}<v_{it} < \max\{\b,r_{it}\}\big)\,\\
&= \E\bigg[
(\vm-\bm)_+ \sum_{i=1}^N q_{it}\bigg] + B\sum_{i=1}^N  \prob\big(\max\{\v,r_{it}\}<v_{it} < \max\{\b,r_{it}\}\big)\\
&\le \E[
(\vm-\bm)_+ ]+ B\sum_{i=1}^N  \prob\big(\max\{\v,r_{it}\}<v_{it} < \max\{\b,r_{it}\}\big)\,.
\end{align*}
{Here, in the second inequality we used $\max\{\vm,r_{it}\} \le B$.} In the equality thereafter, we used the fact that $t\notin\Lie_{k+1}$ and hence $\ind(v_{it} > \max\{\b,r_{it}\}) = \ind(b_{it} > \max\{\b,r_{it}\}) \equiv q_{it}$. The last inequality holds since the item can be allocated to at most one buyer and hence $\sum_{i=1}^N q_{it}\le 1$.  
We next bound the first term by virtue of Lemma~\ref{techlem3} (Equation~\eqref{second-claim}). Specifically,
\begin{align}
{(\vm-\bm)_+} ~\le~ \max\{\shade_{it}(1-q_{it}): i\in [N]\}
~\le~  \sum_{i=1}^N \shade_{it}(1-q_{it}) \,.\label{Delta-2t-2}
\end{align}
To bound the second term, {we  again use inequality that $\max\{a,c\} - \max\{b,c\}~\le~ (a-b)_+$ with $a = \b$, $b =\v$, and $c = r_{it}$:}
\begin{align}
&\prob\big(\max\{\v,r_{it}\}~<~v_{it} ~<~ \max\{\b,r_{it}\}\big\vert {\b, \v}\big) \nonumber\\
&~\le~ \prob\big(\max\{\b,r_{it}\} - {(\b-\v)_+}~<~v_{it} ~<~ \max\{\b,r_{it}\}\big\vert {\b, \v}\big)\nonumber\\
&~=~ \prob\big(\max\{\b,r_{it}\} - {(\b-\v)_+} - \<x_t,\beta_i\>~<~z_{it} ~<~ \max\{\b,r_{it}\} - \<x_t,\beta_i\>\big\vert {\b, \v}\big)\nonumber\\
&= \int_{\max\{\b,r_{it}\} - {(\b-\v)_+} - \<x_t,\beta_i\>}^{ \max\{\b,r_{it}\} - \<x_t,\beta_i\>} f(z)\de z
 {~<~ {\hat c} (\b-\v)_+}\,.\label{prob-B}
\end{align}
{
The first equality follows readily by substituting for $v_{it} =  \<x_t,\beta_i\>+ z_{it}$.}
{In addition, in the last equality, $\hat c \equiv \max_{v\in [-\maxn,\maxn]} f(v)$} is the bound on the noise density,\footnote{{Note} that the density $f$ is continuous and hence attains its maximum over compact sets.} and this equality {holds because $z_{it}$ is independent of $\v$,  $\b$, reserve $r_{it}$, and
the feature vector $x_t$.} {We point out that when $\ind(\max\{\v,r_{it}\}<v_{it} < \max\{\b,r_{it}\}) = 1$, buyer $i$ does not win the item. To see this, recall that we compute the probability of $\ind(\max\{\v,r_{it}\}<v_{it} < \max\{\b,r_{it}\})$ when $t\notin \Lie_{k+1}$. This implies that $\ind(b_{it}>\max\{\b,r_{it}\}) = \ind(v_{it}>\max\{\b,r_{it}\})$ and as a result when $\ind(\max\{\v,r_{it}\}<v_{it}< \max\{\b,r_{it}\}) = 1$, buyer $i$ does not win, i.e., $q_{it}  =0$. The fact $q_{it} =0$ enables us to} 
  use Lemma~\ref{techlem3} (Equation~\eqref{second-claim-2})  along with Equation~\eqref{prob-B} to get
\begin{align}
\prob&\big(\max\{\v,r_{it}\}<v_{it} < \max\{\b,r_{it}\}\big\vert {\b, \v}\big)\le {\hat c} \max\{o_{jt} q_{jt}: j\in [N], {j\ne i}\}
\le {\hat c}\sum_{j=1}^N o_{jt}q_{jt} \,.\label{Delta-2t-3}
\end{align}
Putting bounds in Equations~\eqref{Delta1Lie}, \eqref{Delta-2t-2}, \eqref{Delta-2t-3} together, we have
\begin{align}
\E[\Delta_{2,t}] &~=~ \E[\Delta_{2,t} \ind(t\in\Lie_{k+1})] + \E[\Delta_{2,t} \ind(t\notin\Lie_{k+1})] \nonumber \\
&~\le~ B \prob(t\in \Lie_{k+1}) + \E\Big[ \sum_{i=1}^N \shade_{it}(1-q_{it}) + {\hat c}B \sum_{j=1}^N o_{jt}q_{jt}\Big]\,.\label{Delta-2t-4}
\end{align}

\noindent{\bf Combining bounds on $\boldsymbol{\Delta_{1,t}}$ and $\boldsymbol{\Delta_{2,t}}$:}  To summarize, using bounds~\eqref{Delta-1t-3} and~\eqref{Delta-2t-4} in Equation~\eqref{regret0}, for all $t\in \epi_k$, we have
\begin{align}
\rev^\star_t - \rev_t &~\le~ \frac{B}{\ell_{k}} + \big(1-\frac{1}{\ell_k}\big)\E[\Delta_{1,t} + \Delta_{2,t}]\nonumber\\
&~\le~ \frac{B}{\ell_{k}} +  \frac{c'}{\ajb{d}}\sum_{i=1}^N \E[\|\beta_i - \hbeta_{ik}\|^2] + B \prob(t\in \Lie_{k+1}) + \E\Big[ \sum_{i=1}^N \shade_{it}(1-q_{it}) + {\hat c}B {\sum_{i=1}^N o_{it}q_{it}}\Big]\,.\label{regret2}
\end{align}

We are now ready to bound the total regret of {our} policy. Since the length of episodes doubles each time, the number of episodes up to time $t$ would be at most $K = \lfloor \log T \rfloor +1$.
We then have
\begin{align}
\reg(T) ~\le~ \sum_{k=1}^K \reg_k\,, \label{regret-episode}
\end{align}
where  $\reg_k$ is the regret of our policy in episode $k\in [K]$. 
We bound the total regret over each episode by considering the following two cases: $\ell_{k-1}\le c_0 d$ and $\ell_{k-1}> c_0 d$. Here, $c_0$ is the constant in the statement of Proposition~\ref{propo:learning}.
\begin{itemize}[leftmargin=*]
\item{\bf Case 1:} $\ell_{k-1}\le c_0 d$:  In this case,
we use the trivial bound $\rev^\star_t - \rev_t\le \rev^\star_t \le v_t^+\le B$. Given that the length of episode $k$ is $\ell_k\le 2c_0d$, the total lengths of all such episodes is at most $4c_0d$ and therefore, the total regret
over such episodes is at most $4c_0B d$.
\item{\bf Case 2:} $\ell_{k-1}> c_0 d$: 
In that case, we use bound~\eqref{regret2} on the regret in each period of episode $k$:
\begin{align}
\reg_k& ~=~ \sum_{t\in\epi_k} (\rev^\star_t - \rev_t)\nonumber\\
&~\le~ \frac{B}{\ell_{k}}\ell_k + \frac{c'}{\ajb{d}} \ell_k \sum_{i=1}^\nbuyer \E[\|\beta_i - \hbeta_{ik}\|^2] + B\E[|\Lie_{k+1}|] +  \sum_{i=1}^N\E\Big[\sum_{t\in \epi_k} \shade_{it}(1-q_{it}) + {\hat c}B \sum_{t\in \epi_k} o_{it} q _{it}\Big]\,.\label{regret3}
\end{align}
We treat each term on the right-hand side of~\eqref{regret3} separately.

We first bound {the second term, i.e., $c' \ell_k \sum_{i=1}^\nbuyer \E[\|\beta_i - \hbeta_{ik}\|^2]$.}   Define the probability event $\cG$, such that event $\cG$ happens when Equations~\eqref{eq:estimator} and \eqref{claim:lies} hold; that is, the number of lies satisfies {\eqref{claim:lies}} and the estimation errors satisfies {\eqref{eq:estimator}}.
 By Proposition~\ref{propo:learning} and \ref{propo:lies}, the probability of complement of  event $\cG$, denoted by $\cG^c$, is given by  
  $$\prob(\cG^c)\le {\frac{\delta+1}{\ell_{k-1}}}+d^{-0.5}\ell_{k-1}^{-1.5} + 2e^{-c_2\ell_{k-1}}\,.$$ 
\ngg{Using these propositions again, we get \begin{align}
\E[\|\beta_i - \hbeta_{ik}\|^2] &~=~ \E[\|\beta_i - \hbeta_{ik}\|^2\, \ind(\cG)] + \E[\|\beta_i - \hbeta_{ik}\|^2\, \ind(\cG^c)]\nonumber\\
&~\le~ \frac{c_1{d^2}}{\lF^2}\left(\left(\frac{|\Lie_{ik}|}{\ell_{k-1}}\right)^2 + \frac{\log(\ell_{k-1}d)}{\ell_{k-1}}\right) + 4B^2 \prob(\cG^c)\nonumber\\
&~\le~ {c_{10}}  \left(\left(\frac{{d}\log(T/\delta)}{\ell_{k-1}}\right)^2 +  \left(\frac{d\log(T)}{\ell_{k-1}\log(1/\gamma)}\right)^2+\frac{{d^2}\log(Td)}{\ell_{k-1}}+ \frac{{\delta+1}}{\ell_{k-1}} +d^{-0.5}\ell_{k-1}^{-1.5} +  e^{-c_2\ell_{k-1}} \right)\,,  \label{eq:Exp2}
\end{align}
where we absorb various constants into {constant $c_{10}$} and used $\ell_{k-1} \le T$ and $N\le T$.}

\ngg{Regarding {the third term, i.e.,  $B\E[|\Lie_{k+1}|] $,} by Proposition~\ref{propo:lies} we have
\begin{align}  
\E[|\Lie_{k+1}|] &~\le~ \sum_{i=1}^N \E[|\Lie_{i,k+1}|] \nonumber\\
&~\le ~N\left(c_3\log(\ell_{k-1}/\delta)+ \ngg{c_4\frac{ \log(\ell_{k-1})}{\log(1/\gamma)} +c_5 \frac{\log(N)}{\log(1/\gamma)}}\right) \left(1-\frac{{\delta+1}}{\ell_{k}}\right) + N \ell_{k} \frac{{\delta+1}}{\ell_{k}}
\label{eq:Exp1}
\end{align}
where $\delta$ and $c_3$ are defined in  Proposition~\ref{propo:lies}. }

Finally, we bound the last term of Equation (\ref{regret3}). 
Invoking Equations~\eqref{claim:shades} and \eqref{claim:overbids}, we have
\begin{align} 
\sum_{i=1}^N\E\Big[\sum_{t\in\epi_k} \shade_{it}(1-q_{it})+ {\hat c}B \sum_{t\in \epi_k} o_{it} q _{it} \Big]& \le N \left(\ngg{c_6\frac{ \log(\ell_{k-1})}{\log(1/\gamma)} +c_7 \frac{\log(N)}{\log(1/\gamma)}}\right)\nonumber \\
& +\hat c BN\left(\ngg{c_8\frac{ \log(\ell_{k-1})}{\log(1/\gamma)} +c_9 \frac{\log(N)}{\log(1/\gamma)}} \right)\,.\label{eq:Exp3}
\end{align}
\end{itemize}
\ngg{We employ bounds~\eqref{eq:Exp1}, \eqref{eq:Exp2} and \eqref{eq:Exp3} in bound~\eqref{regret3} and keep only the dominant terms, from which we get}
\begin{align}
\ngg{\reg_k \le {c_{11}}Nd \left(\frac{\log^2(T)}{\ell_k}\left(1+ \frac{1}{\log^2(1/\gamma)}\right)+\log(Td) 
\right)\,,} \label{eq:reg_k}
\end{align}
{for a constant {$c_{11}$} that depends on $B$ and  $\M$.}

\ngg{As the final step, we combine our regent bounds for the two cases to find the total regret of our policy. 
{Recall that $K = \lfloor \log T \rfloor +1$ is the upper bound on the number of episodes up to time $T$.} Then, by Equation (\ref{eq:reg_k}), 
 we obtain}
\begin{align*}
\ngg{\reg(T)}&~\ngg{ \le~4c_0Bd + {c_{11}} Nd\left(\log^2(T)\left(1+ \frac{1}{\log^2(1/\gamma)}\right) \sum_{k=1}^K \frac{1}{\ell_k}+K\log(Td) 
\right)}\\
&~\ngg{=~ 4c_0Bd + {c_{11}} Nd\left(\log^2(T)\left(1+ \frac{1}{\log^2(1/\gamma)}\right) \sum_{k=1}^K \frac{1}{2^{k-1}}+K\log(Td) 
\right)}\\
&~\ngg{\le~ 4c_0Bd + {c_{11}} Nd \left(\log^2(T)+\frac{\log^2(T)}{\log^2(1/\gamma)} + \log(Td)\log(T) 
\right)}\,,
\end{align*}
\ngg{which completes the proof.} 

\ngg{\section{Proof of Theorem~\ref{thm-F}}\label{proof:thm-F}
The proof follows along the same lines as proof of Theorem \ref{thm:main}. Here, we list the main steps that differ from that proof.

For any $t\ge 1$ and $i\in [N]$, define $\tilde{x}_t = [-x_t; \frac{1}{\sqrt{d}}r_{it}]$, $\eta = (\theta,\sqrt{d}\alpha)$, $\widehat\eta = (\widehat \theta, \sqrt{d}\widehat \alpha)$, and $\eta_0= (\theta_{ik},\sqrt{d}\alpha_0)$.
The negative-likelihood function~\eqref{eq:L-F} can then be written as
\begin{align*}
\tilde{{\cal L}}_{ik}(\eta) ~=~ -\Big\lceil\frac{N}{|I_k|}\Big \rceil \sum_{{\{t \in I_k, i^\circ_t = i\}}}\Big\{&q_{it} \log\big(1-F(\<\tilde{x}_t, \eta\>)\big)+(1-q_{it}) \log\big(F(\<\tilde{x}_t, \eta\>)\big)\Big\}\,.  
\end{align*} 
Moreover, for $\{t\in \cal I_k,; i_t^\circ = i\}$, we have $r_{it} = r_t\sim {{\sf Uniform}}(0,B)$, independent of $x_t$ and therefore,
\begin{align}
\tilde{\Sigma} \equiv \E[\tilde{x}_t \tilde{x}_t^\sT ] = \begin{pmatrix} \Sigma_x & \frac{-B\mu}{2\sqrt{d}}\\ \frac{-B\mu^\sT}{2\sqrt{d}} & \frac{B^2}{3d}  \end{pmatrix}\,,
\end{align}
with $\mu = \E[x_t]$ and $\covx =\E[x_tx_t^\sT]$. {The equation above follows from the definition of $\tilde{x}_t$ and because $r_{it} = r_t$, $i = i^\circ_t$, we have $r_{it}$ drawn uniformly at random from $(0,B)$, independently of $x_t$.}
Looking into the Schur complement of $\tilde{\Sigma}$ (corresponding to block $\Sigma_x$), we have\footnote{For a block matrix$
  M=
  [A \quad B;
   C \quad D]$ ,
 the Schur complement of the block $A$ is $D-CA^{-1} B$.} 
\begin{align*}
\frac{B^2}{3d} - \frac{B^2}{4d} \mu^\sT \Sigma_x^{-1} \mu \ge
\frac{B^2}{3d} - \frac{B^2}{4d} \ge \frac{B^2}{12d}\,.
\end{align*}
Here, we used the fact that $\mu^\sT \Sigma_x^{-1}\mu \le 1$ because if we look at the second moment of the vector $a \equiv [x_t;1]$, we have $\E[aa^\sT] = [\Sigma_x\quad \mu; \mu^\sT\quad 1] \succeq 0$. Therefore, by the property of Schur complement, we have $1 - \mu^\sT \Sigma_x^{-1}\mu \ge 0$.

As a result, the singular values of $\tilde{\Sigma}$ are larger than $\tilde{c}_{\min}/d$, with $\tilde{c}_{\min} \equiv \min (c_{\min}, B^2/12)$. Hence, we are in place to apply Proposition~\ref{propo:learning} (noting that the log-likelihood function $\tilde{\cal{L}}_{ik}$ contains $|I_k|/N$ samples), which gives
\begin{align}
\|\widehat\theta_{ik} - \theta_i\|^2 + d \|\widehat \alpha_{0k} - \alpha_0\|^2 = \|\widehat \eta - \eta_0\|^2 \le  \frac{c_1{d^2}}{\lF^2}\left(\left(\frac{N|\Lie_{ik}|}{|I_k|}\right)^2 + \frac{N\log(\ell_{k-1}d)}{|I_k|}\right)\,,\quad\quad i\in[N]\,,\label{eq:estimator-F1}
\end{align} 
with probability at least $1 - d^{-0.5} \ell_{k-1}^{-1.5} - 2e^{-c_2|I_k|/N}$.}

\ngg{The number of lies $|\Lie_{ik}|$ can be bounded as in Propositions~\ref{propo:lies} (or \ref{propo:lies2}), as follows. 
With probability at least $1-{(\delta+1)}/\ell_{k-1}$,
\begin{align}
|\Lie_{ik}|~\le ~ c_3\log(\ell_{k-1}/\delta)+ {c_4\frac{ \log(\ell_{k-1})}{\log(1/\gamma)} +c_5 \frac{\log(N)}{\log(1/\gamma)}}\quad i\in [N]\,.\label{myLies}
\end{align}  

We next bound the errors in reserve prices. Define the function $g:\reals\mapsto \reals_+$ as $g(z) ~=~ \arg\max_{y} \{y(1-F(y-z))\}$. We then have for $t\in E_{k}\backslash I_k$,
\begin{align}
r^\star_{it} - r_{it}  &= \frac{1}{\alpha_0}g(\<x_t,\theta_i\>) - \frac{1}{\widehat \alpha_{0k}}g(\<x_t,\widehat\theta_{ik}\>) \nonumber\\
&\le \frac{1}{\alpha_0} \Big| g(\<x_t,\theta_i\>) - g(\<x_t,\widehat\theta_{ik}\>) \Big| + \Big|\frac{1}{\alpha_{0}} - \frac{1}{\widehat\alpha_{0k}} \Big| g(\<x_t,\widehat \theta_{ik}\>) \nonumber\\
&\le \frac{1}{\alpha_0} \Big| \<x_t,\theta_i- \widehat\theta_{ik}\> \Big| + \frac{2}{\alpha_0^2}\Big|{\alpha_{0}} - {\widehat\alpha_{0k}} \Big| B \nonumber\\
&\le \bar\sigma \Big| \<x_t,\theta_i- \widehat\theta_{ik}\> \Big| + 2\bar{\sigma}^2 \Big|{\alpha_{0}} - {\widehat\alpha_{0k}} \Big| B\,,\label{error-r-F}
\end{align}
where {in the second inequality,}  we used the facts that $(i)$ $g$ is 1-Lipschitz, as shown in the proof of Lemma~\ref{techlem2}; $(ii)$ $g(z) \le z+\maxn$ since $f$ is supported on $[-\maxn,\maxn]$ and 
$|z|\le \|x_t\|\cdot \|\widehat\theta_{ik}\|\le \maxpv$. Hence, $g(\<x_t,\widehat \theta_{ik}\>)\le \maxpv+\maxn = B$; $(iii)$ We have $\widehat\alpha_{0k}\ge \tfrac{\alpha_0}{2}$ using~\eqref{eq:estimator-F1}.

We are now ready to bound the regret of CORP-II policy. If $t$ is a pure exploration period, we use the trivial bound
$\rev^\star_t - \rev_t \le B$. For the periods in the exploitation phase, we note that CORP-II does not use any of the submitted bids during the exploitation phases to update its estimates
of the preference vectors or the scaling parameter of the noise distribution. In addition, since the second-price auctions are strategy-proof, this means that in these phases, buyers have no incentive to 
bid untruthfully. Hence, the term $\Delta_{2,t}$ (that captures the effect of untruthfulness in the regret) becomes superfluous and rewriting Equation~\eqref{regret2} (with only $\Delta_{1,t}$) and \eqref{Delta-1t-2} give
\begin{align}
\rev^\star_t - \rev_t &~\le~ \E[\Delta_{1,t}] \nonumber\\
&~\le~ \frac{c}{2} \sum_{i=1}^N \E[(r^\star_{it} - r_{it})^2] \nonumber\\
&~\le~ \frac{c}{2} \sum_{i=1}^N \E\Big[ \E\Big[(r^\star_{it} - r_{it})^2\Big| x_t, \widehat \beta_{ik}\Big]\Big] \nonumber\\
&~\le~ {c\bar{\sigma}^2} \sum_{i=1}^N \left(\E[\<x_t,\theta_i- \widehat\theta_{ik}\>^2] + 4\bar{\sigma}^2 B^2 \E[({\alpha_{0}} - {\widehat\alpha_{0k}})^2]  \right)\,,\label{reg-gap-F}
\end{align}
 where we used the same derivation as in~\eqref{Delta-1t-2} along with inequality~\eqref{error-r-F}.}

\ngg{Invoking~\eqref{eq:Sigma-E}, we have
\begin{align}\label{pred-F}
\E[\<x_t,\theta_i -\widehat\theta_{ik}\>^2]~{=}~ \E[\<\theta_i-\widehat\theta_{ik}, \Sigma (\theta_i - \widehat\theta_{ik})\>] ~\le~ \frac{c_{\max}}{\ajb{d}} \E[\|\theta_i - \widehat\theta_{ik}\|^2]\,.
\end{align} 

We are now ready to bound the total regret up to time $T$.
Given that the length of episode doubles each time, letting $K = \lfloor \log T \rfloor +1$, we have 
${\reg(T) \le \sum_{k=1}^K \reg_k}$. 
Similar to the proof of Theorem~\ref{thm:main}, we bound the total regret over each episode by considering two cases:}
\ngg{
\begin{itemize}[leftmargin=*]
\item{\bf Case 1:} {$\ell_{k-1}\le c_0 d$:} Here, {$c_0$} is the constant in the statement of Proposition~\ref{propo:learning2}. {In this case, as we argued in the proof of Theorem~\ref{thm:main}, the total regret
over such episodes is at most {$4c_0B d$}.}

\item{\bf Case 2:} {$k >c_0 $}:  Define event $\cG$ such that  event $\cG$ happens  when {equations}~\eqref{eq:estimator-F1} and \eqref{myLies} hold. 
We then have $$\prob(\cG^c)~\le~ \frac{{\delta+1}}{\ell_{k-1}}+d^{-0.5}\ell_{k-1}^{-1.5} + 2e^{-c_2 |\initial_k|/N}\,.$$ 
\end{itemize}
Therefore, by using~\eqref{pred-F} and the definition of event $\cG$, we obtain
\begin{align}
\E[\<x_t,\theta_i- \widehat\theta_{ik}\>^2] &~=~ \E[\<x_t,\theta_i- \widehat\theta_{ik}\>^2 \, \ind(\cG)] + \E[\<x_t,\theta_i- \widehat\theta_{ik}\>^2 \, \ind(\cG^c)]\nonumber\\
&\le \frac{c_1 c_{\max}}{\ell_F^2} {d}\left( \frac{N\log(\ell_{k-1}d)}{|\initial_k|}+ N^2\left(\frac{ c_3\log(\ell_{k-1}/\delta)+ {c_4\frac{ \log(\ell_{k-1})}{\log(1/\gamma)} +c_5 \frac{\log(N)}{\log(1/\gamma)}}}{|\initial_k|}\right)^2 \right) + 4B^2 \prob(\cG^c)\nonumber\\
&\le {{c_6}} d \left( \frac{{N}\log(Td)}{|\initial_k|}+ N^2\left(\frac{\log(T/\delta)+{\log(T)/\log(1/\gamma)}}{|\initial_k|}\right)^2 +\frac{{\delta+1}}{\ell_{k-1}} +d^{-0.5}\ell_{k-1}^{-1.5} +  {e^{-{c_2}|\initial_k|/N}} \right)\,,  \label{eq:Exp2-2-F}
\end{align} \ngg{where the first inequality follows  from the definition of event  $\cG$ and our bound on the probability of $\cG$.}
Here {$c_6$} hides various constants and we used $\ell_{k-1} \le T$ and $N< T$. Likewise,
\begin{align}
\E[(\widehat \alpha_{0k} - \alpha_0)^2] \le {{c_6}} d \left( \frac{{N}\log(Td)}{|\initial_k|}+ N^2\left(\frac{\log(T/\delta)+{\log(T)/\log(1/\gamma)}}{|\initial_k|}\right)^2 +\frac{{\delta+1}}{\ell_{k-1}} +d^{-0.5}\ell_{k-1}^{-1.5} +  {e^{-{c_2}|\initial_k|/N}} \right)\,.\label{eq:Exp2-3-F}
\end{align}

We next employ bounds~\eqref{eq:Exp2-2-F} and \eqref{eq:Exp2-3-F} in Equation~\eqref{reg-gap-F}. By keeping the dominant terms we get 
\begin{align}
\sum_{t\in \epi_k\backslash \initial_k} (\rev^\star_t - \rev_t) \le {\hat c_6}  d \left(\frac{N\log(Td)}{|\initial_k|} \ell_k +\Big(\frac{N\log(T) (1+1/\log(1/\gamma))}{|I_k|} \Big)^2 \ell_k\right)\,,
\end{align}
for a constant ${\hat c_6}$ that depends on $B$ and  $\delta$. By adding the regret injured in the exploration phase $I_k$, we obtain
 \begin{align}\label{EkB-2}
{\reg_k}  &= \sum_{t\in \initial_k} (\rev^\star_t - \rev_t)  + \sum_{t\in \epi_k\backslash \initial_k} (\rev^\star_t - \rev_t)\nonumber\\
&\le B|I_k| + {\hat c_6}  d \left(\frac{N\log(Td)}{|\initial_k|} \ell_k +\Big(\frac{N\log(T) (1+1/\log(1/\gamma))}{|I_k|} \Big)^2 \ell_k\right)\,.
\end{align}

Finally, we are ready to bound the cumulative regret up to time $T$. Recall $K = \lfloor \log T \rfloor+1$, the number of episodes by time $T$. By combining the above two cases and substituting for $|I_k| = \lceil \sqrt{\ell_k}\rceil  = 2^{(k-1)/2}$, we obtain
\begin{align*}
\reg(T)&\le 4c_0 Bd + \sum_{k=1}^K B|I_k|\nonumber\\
&+ {\hat c_6}  d \left(\frac{N\log(Td)}{|\initial_k|} \ell_k +\Big(\frac{N\log(T) (1+1/\log(1/\gamma))}{|I_k|} \Big)^2 \ell_k\right)\\
&\le 4c_0 Bd +B \sum_{k=1}^K 2^{(k-1)/2} \nonumber\\
&+ {\hat c_6} d \left({N\log(Td)} \sum_{k=1}^K 2^{(k-1)/2} +K N^2\log^2(T) \Big(1+\frac{1}{\log^2(1/\gamma)}\Big) \right)\\
&\le  4c_0 Bd +B\sqrt{T}+ {\hat c_6} d \left(N \log(Td) \sqrt{T} + N^2 \Big(1+\frac{1}{\log^2(1/\gamma)}\Big) \log^3(T)  \right)\,,
\end{align*}
which completes the proof.}

\section{Proof of Theorem~\ref{thm:main2}}\label{proof:thm-main2}
{The proof, in sprit, is similar to that of Theorem \ref{thm:main}.}
We first state an upper bound on the estimation error of the preference vectors $\beta_i$. This proposition is analogous to Proposition~\ref{propo:learning}, where instead of log-likelihood estimator, we use the least square estimator.

\begin{propo}[{Impact of Lies on Estimated Preference Vectors in SCORP}]\label{propo:learning2}
{Suppose {that} Assumption~\ref{assump1} 
 holds} and let $\hbeta_{ik}$ be the solution of optimization~\eqref{optimization2}.
Then, there exist constants {$c_0$,  $c_1$, and $c_2$} such that for $\ell_{k-1}\ge c_0 d$, with probability at least $1 - d^{-0.5}\ell_{k-1}^{-1.5}- 2e^{-c_2|\initial_k|}$, we have 
\begin{align}\label{eq:estimator-2}
\|\hbeta_{ik} - \beta_i\|^2\le {c_1}\ajb{d^2} \left(\left(\frac{|\Lie_{ik}|}{|\initial_k|}\right)^2 + \frac{\log(\ell_{k-1}d)}{|\initial_k|}\right)\, \quad i\in [N]\,,
\end{align}
where $\Lie_{ik}$ is the set of lies associated to buyer $i$ in episode $k$, given by~\eqref{def:lie}, and {$\initial_k$ is the set of pure exploration periods in episode $k$.} 
\end{propo}

Proof of Proposition~\ref{propo:learning2} is given in Section~\ref{proof:propo-learning2}. We next proceed to bound the number of lies $|\Lie_{ik}|$. {We argue that the same bound given in Proposition~\ref{propo:lies} still holds for SCORP.
\ngg{\begin{propo}[Bounding the Number of Lies in SCORP]\label{propo:lies2} 
There exists  constant $c_3$, $c_4$, and $c_5$ such that for  any fixed $0\le \delta\le 1$, with probability at least $1-{(\delta+1)}/\ell_{k-1}$, the following holds:
\begin{align}
|\Lie_{ik}|~\le ~ c_3\log(\ell_{k-1}/\delta)+ \ngg{c_4\frac{ \log(\ell_{k-1})}{\log(1/\gamma)} +c_5 \frac{\log(N)}{\log(1/\gamma)}}\quad i\in [N]\,.\label{claim:lies-scorp} 
\end{align}   
\end{propo}}
Similar to Proposition \ref{propo:lies}, we prove Proposition \ref{propo:lies2} by balancing the utility loss of an untruthful buyer with his future utility gain.
{The proof is presented in Section~\ref{proof:lies2}.}
}

Our next {lemma} relates the difference between the reserves $r_{it}$, set by SCORP policy, and benchmark reserves $r^\star_{it}$, to the estimation error of preference vectors. {This lemma is analogous to Lemma \ref{techlem2}.}

\begin{lemma}[{Errors in Reserve Prices}]\label{techlem2-2}
For $r^\star_{it}$ and $r_{it}$ given by~\eqref{def:rstar-WC} and \eqref{def:r-WC}, respectively, {conditioned on the feature vector $x_t$ {and $\hbeta_{ik}$},} the following holds 
\begin{align}\label{r-B}
|r^\star_{it}- r_{it}| \le |\<x_t,\beta_i-\hbeta_{ik}\>|\,.
\end{align}
\end{lemma} 
We refer to {Section~\ref{sec:technical}} for the proof of {Lemma}~\ref{techlem2-2}.

{Having established the preliminary results, we proceed to bound the regret of SCORP.} The proof goes along the same lines of the proof of Theorem~\ref{thm:main}.
We fix $k\ge 1$ and focus on the total regret during episode $k$. For the pure exploration phase, i.e., $t\in \initial_k$, we use the trivial bound {$
\rev^\star_t - \rev_t\le B$, 
which holds since $\rev^\star_t \le v^+_t\le B$.

To bound the regret {in periods of} the exploitation phase ($t\in \epi_k\backslash \initial_k$), we note that {SCORP {does not use any of} the submitted bids during the exploitation phases to estimate the preference vectors, {and buyers are cognizant of this point as the seller's learning policy is fully known to them}}. {In addition,} since the second-price auctions are strategy-proof, this means that in the exploitation phase,  there is no incentive for buyers to be untruthful. \footnote{{ 
Indeed, we can make truthful strategy the unique best response strategy in the exploitation phase by tweaking the mechanism such that with a fixed small probability in each round, all the reserve prices are set to zero, independently.}}. 
Hence, for $t\in \epi_k\backslash \initial_k$, we have
\begin{align}\label{seller1-2}
\rev_t =  \sum_{i=1}^N \E\left[ \max\{v_t^-,r_{it}\} \ind(v_{it} > \max\{v_t^-,r_{it}\}) \right]\,. 
\end{align} 
This leads to 
\begin{align} \nonumber
\rev^\star_t - \rev_t = \E\left(\sum_{i=1}^N \bigg[\max\{v_t^-,r^\star_{it}\} \ind(v_{it} > \max\{v_t^-,r^\star_{it}\}) 
-\max\{v_t^-,r_{it}\} \ind(v_{it} > \max\{v_t^-,r_{it}\})\bigg] \right)\,, 
\end{align}
where the expectation is with respect to the true underlying noise distribution, {which can vary over time and}  is of course, unknown to the firm and the benchmark policy. 
To bound $(\rev^\star_t - \rev_t )$, we {first} write it in terms of function $W_{it}(r)$, 
 defined by~\eqref{def:W1}. {By virtue of the mean-value theorem,} we have
\begin{align}\label{regret0-2}
\rev^\star_t - \rev_t  ~=~ \sum_{i=1}^\nbuyer \E[W_{it}(r^\star_{it}) - W_{it}(r_{it})] ~=~ \sum_{i=1}^\nbuyer \E[W'_{it}(r)(r^\star_{it} -r_{it})]\,, 
\end{align}
for some $r$ between $r_{it}$ and $r^\star_{it}$. It is worth noting that, in contrast to Equation~\eqref{eq:taylor2nd}, here we do not go
with Taylor's expansion of order two. The reason is that here $W_{it}'(r^\star_{it})\neq 0$, because $W_{it}(r)$ is defined based on the \emph{true} unknown noise distribution, while $r^\star_{it}$
is the optimal reserve for the \emph{worst-case} distribution in ambiguity set $\cF$. 
 Similar to Lemma~\ref{techlem1}, it is straightforward to see that $|W_{it}'(r)|\le \tilde{c}$ for some constant $\tilde{c}>0$.
Therefore, continuing from {Equation}~\eqref{regret0-2}, we have
\begin{align}
\rev^\star_t - \rev_t  &~\le~ {\tilde{c}}\sum_{i=1}^\nbuyer \E[|r^\star_{it} - r_{it}|]{~=~ \tilde{c} \sum_{i=1}^\nbuyer \E\Big[\E\Big[|r^\star_{it} - r_{it}|\Big|x_t,\hbeta_{ik}\Big]\Big]}\nonumber\\
 &~\le~ \tilde{c} \sum_{i=1}^\nbuyer \E[|\<x_t,\beta_i -\hbeta_{ik}\>|]  ~\le~ \tilde{c} \sum_{i=1}^\nbuyer \E[\<x_t,\beta_i -\hbeta_{ik}\>^2]^{1/2} \nonumber\\
 &~\le~\frac{c'}{\ajb{\sqrt{d}}}\sum_{i=1}^\nbuyer \E[\|\beta_i - \hbeta_{ik}\|^2]^{1/2}\,,\label{regret2_mcorp}
\end{align}
with $c' = \tilde{c} \sqrt{c_{\max}}$. Here, the second inequality holds due to Lemma~\ref{techlem2-2}; the third inequality follows from Cauchy-Schwartz inequality, and the last step is derived as in Equation~\eqref{eq:Sigma-E}.

We are now ready to bound the total regret up to time $T$.
Given that the length of episodes double each time, letting $K = \lfloor \log T \rfloor +1$, we have 
${\reg(T) \le \sum_{k=1}^K \reg_k}$. 
Similar to the proof of Theorem~\ref{thm:main}, we bound the total regret over each episode by considering two cases:}
\begin{itemize}[leftmargin=*]
\item{\bf Case 1:} {$\ell_{k-1}\le c_0 d$:} Here, {$c_0$} is the constant in the statement of Proposition~\ref{propo:learning2}. {In this case, as we argued in the proof of Theorem~\ref{thm:main}, the total regret
over such episodes is at most {$4c_0B d$}.}

\item{\bf Case 2:} {$\ell_{k-1}>c_0 d$}:  Define event $\cG$ such that  event $\cG$ happens  when {equations}~\eqref{claim:lies} and \eqref{eq:estimator-F1} hold. 
 By Proposition~\ref{propo:learning2} and {\ref{propo:lies2}}, we have $$\prob(\cG^c)~\le~ \frac{{\delta+1}}{\ell_{k-1}}+d^{-0.5}\ell_{k-1}^{-1.5} + 2e^{-c_2 |\initial_k|}\,.$$ 
\end{itemize}
\ngg{Therefore, 
\begin{align}
{\E[\|\beta_i - \hbeta_{ik}\|^2]} &~=~ \E[\|\beta_i - \hbeta_{ik}\|^2\, \ind(\cG)] + \E[\|\beta_i - \hbeta_{ik}\|^2\, \ind(\cG^c)]\nonumber\\
&\le c_1{d^2}\left( \frac{\log(\ell_{k-1}d)}{|\initial_k|}+ \left(\frac{ c_3\log(\ell_{k-1}/\delta)+ \ngg{c_4\frac{ \log(\ell_{k-1})}{\log(1/\gamma)} +c_5 \frac{\log(N)}{\log(1/\gamma)}}}{|\initial_k|}\right)^2 \right) + 4B^2 \prob(\cG^c)\nonumber\\
&\le {{c_6}} \left( \frac{{d^2}\log(Td)}{|\initial_k|}+ \left(\frac{{d}\log(T/\delta)+\ngg{\log(T)/\log(1/\gamma)}}{|\initial_k|}\right)^2 +\frac{{\delta+1}}{\ell_{k-1}} +d^{-0.5}\ell_{k-1}^{-1.5} +  {e^{-{c_2}|\initial_k|}} \right)\,,  \label{eq:Exp2-2}
\end{align}
where we absorb various constants into {$c_6$} and used $\ell_{k-1} \le T$ \ngg{and $N< T$}.}

\ngg{We next employ bound~\eqref{eq:Exp2-2} in Equation~\eqref{regret2_mcorp}. By keeping the dominant terms and following the same argument of Equation~\eqref{eq:reg_k}, we get \begin{align}
\sum_{t\in \epi_k\backslash \initial_k} (\rev^\star_t - \rev_t) \le {\hat c_6}  \ngg{N}\left(\sqrt{d} \sqrt{\frac{\log(Td)}{|\initial_k|}} \ell_k +\sqrt{d}  \frac{\log(T)}{|\initial_k|}\ell_k  +\ngg{\frac{\log(T)}{\log(1/\gamma) |\initial_k|}\ell_k}\right)\,,
\end{align}
for a constant ${\hat c_6}$ that depends on \ngg{$B$ and  $\delta$.}

 Adding the total regret during the pure exploration phase, we obtain 
  \begin{align}\label{EkB-2}
{\reg_k}  &= \sum_{t\in \initial_k} (\rev^\star_t - \rev_t)  + \sum_{t\in \epi_k\backslash \initial_k} (\rev^\star_t - \rev_t)\nonumber\\
&\le B |\initial_k| + {\hat c_6} \left(\sqrt{d} \sqrt{\frac{\log(Td)}{|\initial_k|}} \ell_k +\sqrt{d}  \frac{\log(T)}{|\initial_k|}\ell_k  +\ngg{\frac{\log(T)}{\log(1/\gamma) |\initial_k|}\ell_k}\right)\,.
\end{align}}

\ngg{Finally, we are ready to bound the cumulative regret up to time $T$. {Let $K_1 = \lfloor \log(c_0 d) \rfloor+3$}. Reconciling the above two cases into {Equation}~\eqref{EkB-2}, and substituting for $|\initial_k| = \lceil\ell_k^{2/3}\rceil$, we obtain
\begin{align*}
\reg(T)&\le 4c_0Bd + \sum_{k=K_1}^K  B |\initial_k| \\
&+ {\hat c_6}  \ngg{N}\left(\sqrt{d\log(Td)} \sum_{k=K_1}^K \frac{\ell_k}{\sqrt{|\initial_k|}} + \sqrt{d} \log(T)\sum_{k=K_1}^K \frac{\ell_k}{|\initial_k|} + \ngg{\frac{\log(T)}{\log(1/\gamma)}  \sum_{k=K_1}^K  \frac{\ell_k}{\sqrt{|\initial_k|}}} \right)\\
&\le 4c_0Bd + B\sum_{k=K_1}^K {2}^{\frac{2(k-1)}{3}} \\
&+ {\hat c_6}\ngg{N}  \left(\sqrt{d \log(Td)} \sum_{k=K_1}^K {2}^{\frac{2(k-1)}{3}}+ \sqrt{d} \log(T) \sum_{k=K_1}^K {2}^{\frac{k-1}{3}}+ \ngg{\frac{\log(T)}{\log(1/\gamma)} \sum_{k=K_1}^K {2}^{\frac{2(k-1)}{3}}}\right)\\
&\le 4c_0Bd + BT^{2/3} + {\hat c_6} \ngg{N}\left(\sqrt{d\log(Td)}\, T^{2/3} + \sqrt{d} \log(T)\; T^{1/3} + \ngg{\frac{\log(T)}{\log(1/\gamma)}T^{1/3}} \right)\,,
\end{align*}
which completes the proof.}

\section{Proof of Proposition~\ref{prop:opt_reserve}}\label{proof:prop-opt_reserve}
We restate {the} definition of function $W_{it}(r)$, given by Equation~\eqref{def:W1}:
\begin{align*}
W_{it}(r) &~\equiv~ \E\Big[\max\{v_t^-,r\}  \ind(v_{it}\ge \max\{v_t^-,r\} )\Big| x_t\Big]
\\
&{~=~\E\Big[\max\{\v,r\}  \ind(v_{it}\ge \max\{\v,r\} )\Big| x_t\Big]\,,}
\end{align*}
where the expectation is with respect to valuation noises, {conditional} on $x_t$, {and the equality holds because $\vm=\v$ when $\ind(v_{it}\ge \max\{\v,r\}) =1$.} 
Note that $W_{it}(r)$ is the firm's revenue {in} period $t$, when
 buyer $i$ wins the auction with reserve price $r$. 
 
Let ${H_{it}}$ be the distribution of {$\v$} for fixed $x_t$ and denote by ${h_{it}}$ its density. 
The specific form of ${H_{it}}$ does not matter for the sake of our proof. We have
\begin{align}
W_{it}(r) &~=~ \E\Big[\ind(v_{it} > {\v} >r) \v + r \ind(v_{it}>r> {\v} )\Big| x_t\Big]\nonumber\\\nonumber
&~=~ {\E\Big[\ind(\<x_t, \beta_i\>+z_{it} > {\v} >r) \v + r \ind(\<x_t, \beta_i\>+z_{it}>r> {\v} )\Big| x_t\Big]}\\
&~=~ \int_r^\infty  v {h_{it}}(v) (1-F(v - \<x_t,\beta_i\>) \de v + r {H_{it}}(r) (1- F(r - \<x_t,\beta_i\>))\,.
\end{align}
By definition, the optimal reserve price of buyer $i$, denoted by $r^\star_{it}$, is the maximizer of $W_{it}(r)$. By setting the derivative with respect to $r$ equal to zero, we get
\begin{align}
W'_{it}(r) &~=~ {H_{it}}(r) \Big((1- F(r - \<x_t,\beta_i\>)) - r f(r-\<x_t,\beta_i\>)\Big) = 0\,,\label{W-1-dev}
\end{align}
which implies that the optimal price $r^\star_{it}$ should satisfy 
\begin{align}\label{stationary}
1- F(r - \<x_t,\beta_i\>) ~=~ r f(r-\<x_t,\beta_i\>)\,.
\end{align}{Now it is easy to see that Equation~\eqref{stationary} is also the stationary condition for the function $y(1-{F(y-\< x_t,\pv_i\>))}$. Since $1-F$ is log-concave by Assumption~\ref{assump:logcancavity},  {function $y\mapsto y\big(1-F(y-\< x_t,\pv_i\>)$} is also strictly log-concave for $y>0$.\footnote{Note that at a negative value of $y$,  {function $y\mapsto y\big(1-F(y-\< x_t,\pv_i\>)$} is also negative and hence this function cannot take its maximum at a negative $y$.} Therefore, the stationary condition for $y\big(1-F(y-\< x_t,\pv_i\>)$ gives its unique global maximum and the proof is complete.\footnote{If $h(y)$ is strictly log-concave function, then at a stationary point $y_0$ that $h'(y_0) = 0$, we have $\log'(h(y_0)) = h'(y_0)/h(y_0) = 0$. Given that $\log(h(y))$ is strictly concave, this means that $y_0$ is the unique global maximizer of $\log(h(y))$ and by strict monotonicity of the logarithm function, this implies that $y_0$ is also the unique global maximizer of $h(y)$.}
}

\section{Proof of Proposition~\ref{propo:learning}}\label{proof:propo-learning}
Recall that $\hbeta_{ik} \in \reals^d$ is the solution to {the} optimization problem~\eqref{optimization}. 
By the second-order Taylor's theorem, expanding around $\beta_i$, we have  
\begin{align}\label{taylor}
\cL_{ik}(\beta_i) - \cL_{ik}(\hbeta_{ik}) = -\<\nabla \cL_{ik}(\beta_i),\hbeta_{ik}-\beta_i\> -\frac{1}{2} \<\hbeta_{ik} - \beta_i, \nabla^2\cL_{ik}(\tilde{\beta})(\hbeta_{ik}-\beta_{ik})\>\,,
\end{align}
for some $\tilde{\beta}$ on the segment connecting $\beta_i$ and $\hbeta_{ik}$. 
Throughout, $\nabla\cL_{ik}$ and $\nabla^2\cL_{ik}$ respectively denote the gradient and the {Hessian} of $\cL_{ik}$. {In the following, we bound $\|\hbeta_{ik} - \beta_i\|^2$ by bounding the gradient and the {Hessian} of $\cL_{ik}$. }

{We start with} computing the gradient and the {Hessian} of the loss function {$\cL_{ik}(\beta)$}:
\begin{align}
\nabla\cL_{ik}(\beta) = \frac{1}{\ell_{k-1}} \sum_{t \in \epi_{k-1}} {\mu_{it}}(\beta) x_t\,,\quad \nabla^2\cL_{ik}(\beta) = \frac{1}{\ell_{k-1}}\sum_{t \in \epi_{k-1}} {\eta_{it}}(\beta) x_t x_t^\sT\,.\label{grad-H-characterization}
\end{align}
Here, letting ${w_{it}}(\beta) = \max\{{\b},r_{it}\}-\<x_t,\beta\>$, the term ${\mu_{it}}(\beta)$ is given by %
\begin{align}
{\mu_{it}}(\beta) &~=~ q_{it} \frac{f({w_{it}}(\beta))}{1- F({w_{it}}(\beta))} - (1-q_{it}) \frac{f({w_{it}}(\beta))}{F({w_{it}}(\beta))}\label{grad-1}\\
&~=~ -q_{it} \log'\left(1- F({w_{it}}(\beta))\right) -(1-q_{it}) \log'\left(F({w_{it}}(\beta))\right)\,,\label{eq:grad-1-1}
\end{align}
where $\log'F(y)$ {is} the derivative of $\log F(y)$ with respect to $y$.\footnote{Since the density $f$ is zero outside the interval $[-\maxn,\maxn]$, we have $F(z)=0$ for $z<-\maxn$, and $F(z) = 1$ for $z>\maxn$. In Equation~\eqref{grad-1}, if ${w_{it}}(\beta)$ is outside the {interval}  $[-\maxn,\maxn]$, we use the convention {of} $\frac{0}{0}=0.$} Further, the term ${\eta_{it}}(\beta)$ is given by
\begin{align}
{\eta_{it}}(\beta) =  -q_{it} \log''\left(1- F({w_{it}}(\beta))\right) -(1-q_{it}) \log''\left(F({w_{it}}(\beta))\right)\,.
\end{align}

{We are now ready to provide} an upper bound and a lower bound on the gradient and  {Hessian} of the loss function. {These bounds will be used in bounding the estimation error of preference vectors, which is the main goal of this proposition.}

\begin{lemma}\label{lem:grad-hessian}
Define the probability event  
\begin{eqnarray}
\event ~\equiv~ \bigg\{\|\nabla \cL_{ik}(\beta_{i})\| \le \lambda_0 \bigg\}\,, \quad \text{ with } \lambda_0 ~\equiv~ 2 \uF \sqrt{\frac{\log(\ell_{k-1} d)}{\ell_{k-1}}} +2\uF\frac{|\Lie_{ik}|}{\ell_{k-1}}\,, \label{event-prob}
\end{eqnarray}
{where constant $\uF$ is given by}
\begin{align}\label{def:uF}
\uF ~\equiv~ \sup_{|x|\le \maxn} \Big\{\max\Big\{\log'F(x), -\log'(1-F(x)) \Big\}\Big\}\,.
\end{align}
 Then, we have
$\prob(\event)\ge 1- d^{-0.5}\ell_{k-1}^{-1.5}$. Moreover, we have the following lower bound on the {Hessian}:
\begin{align}\label{hessian}
\nabla^2\cL_{ik}(\beta) \succeq \lF \left(\frac{1}{\ell_{k-1}} \sum_{t\in\epi_{k-1}} x_t x_t^\sT\right),\quad \text{ for all }\;\|\beta\|\le B\,,
\end{align}
where {$\lF\ge 0$} is given by~\eqref{eq:lB}. Here, $A\succeq B$ means $A-B$ is a positive semidefinite matrix.
\end{lemma}
Lemma~\ref{lem:grad-hessian} is proved in Section~\ref{proof:lem-grad-hessian}.

By optimality of $\hbeta_{ik}$, we have  $\cL(\hbeta_{ik})\le \cL(\beta_i)$ and therefore by~\eqref{taylor}, we have
\begin{align} \label{eq:L_bound}
\frac{1}{2} \<\hbeta_{ik} - \beta_i, \nabla^2\cL_{ik}(\tilde{\beta})(\hbeta_{ik}-\beta_i)\> \le -\<\nabla \cL_{ik}(\beta_i),\hbeta_{ik}-\beta_i\>\,,
\end{align}
where the l.h.s. can be bounded as follows 
\begin{align}\frac{1}{2} \<\hbeta_{ik} - \beta_i, \nabla^2\cL_{ik}(\tilde{\beta})(\hbeta_{ik}-\beta_i)\> &~=~ \frac{1}{2} (\hbeta_{ik} - \beta_i)^{\sT}{\nabla^2\cL_{ik}}(\tilde{\beta})(\hbeta_{ik}-\beta_i) \nonumber\\
& ~\ge~\frac{1}{2} (\hbeta_{ik} - \beta_i)^{\sT}  \lF \Big(\frac{1}{\ell_{k-1}} \sum_{t\in\epi_{k-1}} x_t x_t^\sT \Big)(\hbeta_{ik}-\beta_i)\nonumber\\ &~=~ \frac{\lF}{2\ell_{k-1}}(\hbeta_{ik} - \beta_i)^{\sT}   \Big( X_k^{\sT} X_k \Big)(\hbeta_{ik}-\beta_i) \nonumber\\
&~=~ \frac{\lF}{2 {\ell_{k-1}}} \|X_k(\hbeta_{ik} - \beta_i)\|^2\,. \nonumber
    \end{align}
    {Here,} $X_k$ is the matrix of size $\ell_{k-1}$ by $d$ whose rows are the feature vectors $x_t$ arriving in episode $k-1$.
    {Moreover,} the inequality follows from Lemma~\ref{lem:grad-hessian}.  
  Applying the above bound  in Equation (\ref{eq:L_bound}) and considering the fact that the l.h.s. of this equation is less than or equal to $\|\nabla\cL_{ik}(\beta_i)\|\, \|\hbeta_{ik} - \beta_i\|$, we get 
\[
\frac{1}{2 {\ell_{k-1}}}\lF \|X_k(\hbeta_{ik} - \beta_i)\|^2~\le~  \|\nabla\cL_{ik}(\beta_i)\|\, \|\hbeta_{ik} - \beta_i\|\,.
\]
This implies that on event $\event$, {defined in}~\eqref{event-prob}, the following holds:
\begin{align}\label{main-ineq0}
\frac{1}{2\ell_{k-1}} \lF \|X_k(\hbeta_{ik} - \beta_i)\|^2 ~\le~ \lambda_0  \|\hbeta_{ik} - \beta_i\|\,.
\end{align}
{To present a lower bound on the l.h.s. of the above equation,} we next lower bound the minimum eigenvalue of $\hSigma_k \equiv (X_k^\sT X_k)/\ell_{k-1}$. Since rows of $X_k$ are bounded (recall that $\|x_t\|\le 1$ by our normalization), they are subgaussian. Using~\citep[Remark 5.40]{vershynin2010introduction}, there exist universal constants $c$ and $C$ such that for every {$m\ge0$}, the following holds with probability at least $1 - 2e^{-c{m}^2}$:
\begin{align}
\Big\|\hSigma_k - \Sigma \Big\|_{\rm op} ~\le~ \max(\delta,\delta^2) \quad \text{where} \quad \delta = C\sqrt{\frac{d}{\ell_{k-1}}} + \frac{{m}}{\sqrt{\ell_{k-1}}}\,, 
\end{align}
where {$\Sigma = \E[x_t x_t^\sT] \in \reals^{d\times d}$} is the covariance of the feature vectors. 
Further, $\|A\|_{\rm op}$ represents the operator norm of a matrix $A$ and {is given by $\|A\|_{\rm op}= \inf\{c\ge 0:  \|Av\|\le c \|v\|, \text{ for any vector $v$} \}$.}
By our assumption that $\Sigma$ is positive definite\footnote{{A symmetric matrix is said to be positive definite if all of its eigenvalues are strictly positive.} {In general, if the distribution of features
{$\px$}, is bounded below from zero on an open set around the origin, then its second-moment matrix is positive definite. This assumption holds for many common distributions such as normal and uniform distributions.}}, we can choose constant $0<c_{\min}<1$ such that $\lambda_{\min}(\Sigma) > c_{\min}/\ajb{d}$, with $\lambda_{\min}(A)$ denoting the minimum eigenvalue of a matrix $A$. \footnote{\ajb{Note that by our normalization, the sum of  eigenvalues of $\Sigma$ would be ${{\rm trace}}(\Sigma) = \E[\|x_t\|^2]\le 1$, and that is why the eigenvalues are scaled by $1/d$.}} Set ${m}= c_{\min}\sqrt{\ell_{k-1}}/(4\ajb{d})$, $c_0 = (4C\ajb{d}/{c_{\min}})^2$ and $c_2 = cc_{\min}^2/16\ajb{d}^2$. Then, for $\ell_{k-1}>c_0d$ with probability at least $1 - 2e^{-c_2\ell_{k-1}}$, the following is true: 
\begin{align}\label{eventG}
\Big\|\hSigma_k- \Sigma \Big\|_{\rm op} \le \frac{1}{2\ajb{d}} c_{\min}\,.
\end{align}
Denote by $\cG$ the probability event that~\eqref{eventG} holds. Then, on event $\cG\cap \event$, we have
\begin{align}\label{eq:prediction}
\frac{1}{4\ajb{d}} c_{\min}\lF \|\hbeta_{ik} - \beta_i\|^2 ~\le~ \frac{1}{2\ell_{k-1}} \lF \|X_k(\hbeta_{ik} - \beta_i)\|^2 ~\le~ \lambda_0 \|\hbeta_{ik} - \beta_i\|\,,
\end{align}
{where the first  inequality holds because of Equation (\ref{eventG}) and the definition of the operator norm.}
This results in
\begin{align}
\|\hbeta_{ik} - \beta_i\|^2 \le \frac{16\ajb{d^2}}{c_{\min}^2\lF^2} \lambda_0^2  &= \Big(\frac{8\ajb{d} \uF}{c_{\min}\lF}\Big)^2 \bigg(\sqrt{\frac{\log(\ell_{k-1} d)}{\ell_{k-1}}} +\frac{|\cL_{ik}|}{\ell_{k-1}} \bigg)^2\nonumber\\
&\le 2\Big(\frac{8\ajb{d}\uF}{c_{\min}\lF}\Big)^2  \bigg(\frac{\log(\ell_{k-1} d)}{\ell_{k-1}} + \left(\frac{|\Lie_{ik}|}{\ell_{k-1}}\right)^2  \bigg)\,,\label{main-ineq1}
\end{align}
where in the last line, we used inequality $(a+b)^2\le 2a^2+ 2b^2$.

Note that $$\prob((\event\cap\cG)^c)\le \prob(\event^c)+\prob(\cG^c)\le  d^{-0.5}\ell_{k-1}^{-1.5} +2e^{-c_2\ell_{k-1}}\,,$$
and hence the result follows readily from~\eqref{main-ineq1}, by defining $c_1~\equiv~ 128(\uF/c_{\min})^2$.

\subsection{Proof of Lemma~\ref{lem:grad-hessian}}\label{proof:lem-grad-hessian}
{We first show the first result in the lemma. We start with few definitions. }
Let $\tilde{q}_{it} = \ind(v_{it}>\max\{{\b},r_{it}\})$ be the allocation variable as if buyer $i$ was truthful {and the highest competing bid was $\b$.} Then, by definition of set of lies $\Lie_{ik}$, as per \eqref{def:lie}, 
for $t\notin \Lie_{ik}$, we have $q_{it} = \tilde{q}_{it}$. {Let} 
$$\tilde{\mu}_{it}(\beta) = -\tilde{q}_{it}\log'(1-F({w_{it}}(\beta))) - (1-\tilde{q}_{it}) \log'(F({w_{it}}(\beta))$$
be the corresponding quantity to ${\mu_{it}}(\beta)$, where we replace $q_{it}$ by $\tilde{q}_{it}$. {Note  that $w_{it}(\beta) = \max\{{\b},r_{it}\}-\<x_t,\beta\>$ and ${\mu_{it}}(\beta)$ is defined in (\ref{eq:grad-1-1})}.
 {By definition,} ${\mu}_{it}(\beta) = \tilde{\mu}_{it}(\beta)$ for $t\notin \Lie_{ik}$, and so we can write
\begin{align*}
\nabla\cL_{ik}(\beta) &~=~- \ngg{\frac{1}{\ell_{k-1}}\sum_{t\in \epi_{k-1}} \Big(\tilde{\mu}_{it}(\beta) x_t - {\mu_{it}}(\beta) x_t \Big)+ \frac{1}{\ell_{k-1}} \sum_{t\in \epi_{k-1}} \tilde{\mu}_{it}(\beta) x_t}  \\
&~=~ - \frac{1}{\ell_{k-1}}\sum_{t\in \Lie_{ik}} \Big(\tilde{\mu}_{it}(\beta) x_t - {\mu_{it}}(\beta) x_t \Big)+\frac{1}{\ell_{k-1}} \sum_{t\in \epi_{k-1}} \tilde{\mu}_{it}(\beta) x_t\,.  
\end{align*}
 {To bound the first term on the right hand side of the last equation, we note that  
\begin{align}\label{mu-bound}
|{\mu_{it}}(\beta_i)| ~\le~ \underset{|y|\le \maxn}{\sup} \Big\{\max\Big\{\log'F(y), -\log'(1-F(y))\Big\}\Big\} ~=~ \uF\,,
\end{align}
with $u_F$ given by~\eqref{def:uF}. Here, the first inequality follows from definition of ${\mu_{it}}(\beta_i)$ as per~\eqref{eq:grad-1-1} and using the fact that functions $f$ and $F$ are zero outside the interval $[-\maxn,\maxn]$. Similarly, we have $|\tilde{\mu}_{it}(\beta_i)| \le \uF$. Then, considering the fact that 
$\|x_t\|\le1$, we get}
\begin{align}\label{grad-2}
\Big\|\nabla\cL_{ik}(\beta_i)\Big\| \le \frac{2\uF}{\ell_{k-1}} |\Lie_{ik}| + \frac{1}{\ell_{k-1}} \bigg\|\sum_{t\in \epi_{k-1}} \tilde{\mu}_{it}(\beta_i) x_t  \bigg\|  \,.
\end{align}
We next bound the {second} term on the right hand side of~\eqref{grad-2}. {Define $S_{j} = \sum_{t=\ell_{k-1}}^{j-1+\ell_{k-1}}\tilde{\mu}_{it}(\beta_i) x_{t}$, $j= 1, 2, \ldots, \ell_k-1$, and set $S_0=0$. Note that the second term in \eqref{grad-2} is equal to $S_{\ell_k-1}$. We upper bound $\frac{1}{\ell_{k-1}}\Big\|S_{\ell_k-1}\Big\|$ by showing $S_j$ is a vector martingale with bounded differences. }

{Observe that $\|S_{j} - S_{j-1}\|~\le~ \uF \|x_t\|~\le~ \uF \,$. 
 Further, $S_{j}- S_{j-1}  = \tilde{\mu}_{it}(\beta_i) x_{t}$ with $t = \ell_{k-1}+j-1$, and 
\begin{align*}
\E[\tilde{\mu}_{it}(\beta_i)|{w_{it}}(\beta_i)] &~=~ \prob(\tilde{q}_{it} = 1)  \frac{f({w_{it}}(\beta_i))}{1- F({w_{it}}(\beta_i))} - \prob(\tilde{q}_{it} = 0) \frac{f({w_{it}}(\beta_i))}{F({w_{it}}(\beta_i))}\\
&~=~(1- F({w_{it}}(\beta_i))\frac{f({w_{it}}(\beta_i))}{1- F({w_{it}}(\beta_i))}  - F({w_{it}}(\beta_i))\frac{f({w_{it}}(\beta_i))}{F({w_{it}}(\beta_i))} ~=~ 0\,,
\end{align*}
where the equation holds because $z_{it}$ is independent of ${w_{it}}(\beta_i)$. 
Then, considering the fact that $z_{it}$ is independent {from} the history set~\eqref{def:history}, we also have 
\[
\E[S_{j}- S_{j-1} | S_{1},\dots, {S_{j-1}}] = \E[\tilde{\mu}_{it}(\beta_i)|S_{1},\dots, {S_{j-1}}] = 0\,.
\] 
}
{So far, we have established that $S_j$ is a matrix martingale with bounded differences.}  Then, by Matrix Freedman inequality (See Appendix~\ref{sec:freedman}),
\begin{align}\label{Azuma}
 \prob\Big(\|S_{\ell_{k-1}}\|\ge 2 \uF\sqrt{\log(\ell_{k-1}d)\ell_{k-1}}\Big) ~\le~ (d+1) \exp^{-(12/8) \log(\ell_{k-1}d)} ~=~ \frac{1}{d^{0.5}\ell_{k-1}^{1.5}}\,.
\end{align}
{Then, by Equation (\ref{grad-2}) and  definition of $S_{\ell_{k-1}}$ and event $\event$, given in \eqref{event-prob}, we have $\prob(\event) \ge 1- d^{-0.5}\ell_{k-1}^{-1.5}$.  This completes the proof of the first part of the lemma.}

\if false By plugging in for $S_{\ell_{k-1}}$, we obtain \begin{align}\label{Azuma-2}
\prob\bigg(\frac{1}{\ell_{k-1}}\Big\|\sum_{t\in \epi_{k-1}}\tilde{\mu}_{it}(\beta_i) {x_{t}}\Big\| \ge 2\uF\sqrt{ \frac{\log(\ell_{k-1}d)}{\ell_{k-1}} } \bigg) = \prob(\event) \le \frac{1}{d^{0.5}\ell_{k-1}^{1.5}}\,.
\end{align}  
Combining Equations~\eqref{grad-2} and \eqref{Azuma-2} shows that $\prob(\event) \ge 1- d^{-0.5}\ell_{k-1}^{-1.5}$, for the probability event $\event$, defined by\eqref{event-prob}.
\fi

We next prove claim~\eqref{hessian} on the {Hessian} $\nabla^2 \cL({\beta})$. By {characterization}~\eqref{grad-H-characterization}, it suffices to show that ${\eta_{it}}(\beta)\ge \lF$. To see this, 
\begin{align*}
{\eta_{it}}(\beta) &~=~  -q_{it} \log''\left(1- F({w_{it}}(\beta))\right) -(1-q_{it}) \log''\left(F({w_{it}}(\beta))\right)\\
&~\ge~ \inf_{|y|\le \maxn} \left\{\min \left\{ -\log''F(y), -\log''(1-F(y)) \right\}\right\}~\equiv~ \lF\,,
\end{align*}
{where we used the fact that function $F$ is zero outside the interval $[-\maxn,\maxn]$.}
This completes the proof of {the lemma}.

\section{Proof of Proposition~\ref{propo:lies}}\label{proof:propo-lies}
{Here, we need to show claims \eqref{claim:lies}, \eqref{claim:shades}, {and \eqref{claim:overbids}.} {Let  $o_{it} = (b_{it}-v_{it})_+$ and $\shade_{it} = (v_{it}-b_{it})_+$ be the amount of overbidding and shading (underbidding) of buyer $i$ in period $t$, respectively.}
 As a common step to show these claims, we upper
bound the size of sets $\cS_{ik} \equiv\{t:\, t\in{\epi_{k-1}},\, q_{it} = 0,\, \shade_{it} \ge {1/\ell_{k-1}} \}$} {and $\cO_{ik} \equiv\{t:\, t\in{\epi_{k-1}},\, q_{it} = 1,\, o_{it} \ge {1/\ell_{k-1}} \}$.}
In words, a period $t$ belongs to $\cS_{ik}$, if buyer $i$ has shaded his bid significantly in this period, i.e., $s_{it}\ge {1/\ell_{k-1}}$, and he does not get the item in this period. {Similarly, a period $t$ belongs to $\cO_{ik}$, if in this period, buyer $i$ has over-bided  by at least  ${1/\ell_{k-1}}$ amount, and he gets the item in this period.}
 We next use the bounds that we establish on $|\cS_{ik}|$ {and $|\cO_{ik}|$} to prove the {three} aforementioned claims.

{{To}  bound the size of sets $\cS_{ik}$ and $\cO_{ik}$, we use the fact that buyers are utility-maximizer and as a result, they aim for balancing the utility loss due to {bidding untruthfully} with its potential gain.}  We define $u_{it}^-$ as the {utility} that buyer $i$ loses {in period $t\in \epi_{k-1}$} {due to bidding untruthfully}{, relative to the truthful bidding}. {Precisely, given reserve price $r_{it}$ and the highest competing bid {$\b$}, $u_{it}^-$ is defined as follows: 
{\begin{align}\nonumber u_{it}^- ~&=~ (v_{it}-\max\{\b, r_{it}\})\;  \ind(v_{it}> \max\{\b, r_{it}\}) \; \ind(b_{it}< \max\{\b, r_{it}\})\\
~&-~(v_{it}-\max\{\b, r_{it}\})\;  \ind(v_{it}< \max\{\b, r_{it}\}) \; \ind(b_{it}> \max\{\b, r_{it}\})\,. \nonumber 
\end{align}}
{Note that the first and second terms are the loss due to underbidding and overbidding, respectively. }
}
Our lemma below provides a lower bound on {the} expected value of $u_{it}^-$.
\begin{lemma}\label{lem:utilityloss}
For each buyer $i\in [N]$ and {$t\in [\ell_{k-1}, \ell_{k}-1]$,} we have
\begin{align}\label{eq:utilityloss}
\E[u_{it}^-|{\shade_{it}, o_{it}}, q_{it}] ~\ge~ 
{\frac{1}{{2}BN \ell_{k-1}}} \gamma^ t \shade_{it}^2 (1-q_{it})+{ \frac{1}{2BN{\ell_{k-1}}}\gamma^t o_{it}^2 q_{it}} \,,
\end{align}
{where the expectation is taken w.r.t. to the randomness in reserve prices. }
\end{lemma} 
\begin{proof}{Proof of Lemma~\ref{lem:utilityloss}.}
Note that {in} each period, buyers may suffer a utility loss due to bidding untruthfully.  {We start with characterizing  the impact of underbidding. We then focus on overbidding.}

\textbf{Underbidding:} After observing the outcome of auction $t$, if buyer $i$ receives the item, then underbidding has no effect on the buyer's instant utility. But if the buyer does not receive the item, then there is a chance that is due to the underbidding. To lower bound $u_{it}^-$, note that {in} each period $t\in \epi_{k-1}$, with probability {$1/\ell_{k-1}$}, the firm does not run a second-price auction. Instead, she picks one of the {buyers} equally likely and for a reserve price, chosen uniformly at random from $[0,B]$, allocates the item to that buyer if his {bid} exceeds the corresponding reserve price.
Therefore, if a buyer $i$ shades his bid by {$\shade_{it}$}, i.e., {$\shade_{it} = (v_{it} - b_{it})_+$}, then the utility loss incurred relative to being truthful can be lower bounded as follows: \begin{align} \label{eq:shading}
\E[u_{it}^-|\shade_{it},  q_{it}, v_{it}] &~\ge~ \frac{\gamma^t (1-q_{it})}{BN{\ell_{k-1}}} \int_{v_{it}-\shade_{it}}^{v_{it}} (v_{it} - r)  \de r ~{=}~  \frac{1}{2BN{\ell_{k-1}}}\gamma^t \shade_{it}^2 (1-q_{it})\,.
\end{align}

{\textbf{Overbidding:} After observing the outcome of auction $t$, if buyer $i$ does not get the item, then overbidding has no effect on the buyer's instant utility. But if the buyer  receives the item, then there is a chance that is due to the overbidding. Then, one can follow a similar argument that we used for underbidding to show that
\begin{align}\label{eq:overbidding}
\E[u_{it}^-|o_{it}, q_{it}, v_{it}] &~\ge~ \frac{\gamma^t (q_{it})}{BN{\ell_{k-1}}} \int_{v_{it}}^{v_{it}+o_{it}} (-v_{it} + r)  \de r ~{=}~  \frac{1}{2BN{\ell_{k-1}}}\gamma^t o_{it}^2 q_{it}\,.
\end{align}
Then, the result follows from Equations (\ref{eq:shading}) and (\ref{eq:overbidding}), and by taking expectation  {w.r.t $v_{it}$}, from both sides of these equations, conditioning on $q_{it}$, $\shade_{it}$, and $o_{it}$.}

\qed \end{proof}
{With a slight abuse of notation, let $U^-_{i(k-1)}$ be the utility loss of buyer $i$ in episode $k-1$ due to untruthful bidding. That is, $U^{-}_{i(k-1)} = \sum_{t\in \epi_{k-1}} u^-_{it}$. }
 {In contrast to {the utility loss $U^-_{i(k-1)}$}, we define {$U^+_{ik}$}  as the total utility gain that buyer $i$ can achieve by being untruthful in episode $k-1$. More precisely, we fix all other buyer's bidding strategy and consider a reference strategy for buyer $i$. The reference strategy is the same as buyer $i$'s strategy up to episode $k-1$, and in episode $k-1$, the reference policy is just the truthful bidding strategy. We define {$U_{ik}^+$} as the total excess {utility} that buyer $i$ can earn, over the reference strategy.} 
{Considering the fact that 
the bidding strategy of buyer $i$ in episode $k-1$ can only benefit him in the next episodes $k, k+1, \dotsc$, we have}
\begin{align}\label{eq:utilitygain}
{U^+_{ik}} &~\le~ \sum_{t= \ell_{k}}^\infty \gamma^t {v_{it}} ~\le~ B \sum_{t= \ell_{k}}^{\infty} \gamma^t 
~=~ {B}  \frac{\gamma^{\ell_{k}}}{1-\gamma}\,,
\end{align}  
{where the second inequality holds because $v_{it}\le B$.} Indeed, upper bound~\eqref{eq:utilitygain} applies
to the total {utility} any buyer can hope to collect {over periods $t\ge \ell_{k}$.} 

Now, since we are assuming that the strategic buyers are maximizing their {cumulative utility}, it must be the case that 
{$$\E\Big[{{U^+_{ik}} - U^{-}_{i(k-1)}} \Big]\ge 0\,.$$}
Using Lemma~\ref{lem:utilityloss} along with upper bound~\eqref{eq:utilitygain}, we obtain
{
\begin{align}\label{main-ineq}
B  \frac{\gamma^{\ell_{k}}}{1-\gamma} &\ge E[U^+_{ik}] \ge E[U^-_{i(k-1)}] \nonumber\\
&=\sum_{t\in E_{k-1}} \E[u^-_{it}]= \sum_{t\in E_{k-1}} \E[\E[u^-_{it}|o_{it},q_{it},s_{it}]]\nonumber\\
& \stackrel{(a)}{\ge} \sum_{t\in \cS_{ik}\cup \cO_{ik}} \E\Big[\E[u^-_{it}|o_{it},q_{it},s_{it}] \ind(|\cS_{ik}\cup \cO_{ik}|\ge m)\Big]\nonumber\\
&\stackrel{(b)}{\ge}\sum_{t\in \cS_{ik}\cup \cO_{ik}} \E\Big[\frac{\gamma^t}{2BN\ell_{k-1}^3}  \ind(|\cS_{ik}\cup \cO_{ik}|\ge m)\Big]\nonumber\\
&\stackrel{(c)}{\ge}\sum_{t= \ell_k-|\cS_{ik}\cup \cO_{ik}|}^{\ell_k - 1} \E\Big[\frac{\gamma^t}{2BN\ell_{k-1}^3}  \ind(|\cS_{ik}\cup \cO_{ik}|\ge m)\Big]\nonumber\\
&=\sum_{t= \ell_k -m}^{\ell_k-1} \frac{\gamma^t}{2BN\ell_{k-1}^3} \prob(|\cS_{ik}\cup \cO_{ik}|\ge m)\nonumber\\
& = \frac{\gamma^{\ell_k}}{2BN\ell_{k-1}^3} \frac{\gamma^{-m}-1}{1-\gamma}   \prob(|\cS_{ik}\cup \cO_{ik}|\ge m)
\end{align}
where the parameter $m\ge 0$ in (a) is arbitrary but fixed; (b) follows from Lemma~\ref{lem:utilityloss} and definition of sets $\cS_{ik}$ and $\cO_{ik}$; $(c)$ holds since $\gamma\in (0,1)$.

By simplifying Equation~\eqref{main-ineq}, we have that for any $m\ge 0$,
\[
 (\gamma^{-m}-1) \prob(|\cS_{ik}\cup \cO_{ik}|\ge m) \le 2B^2N\ell_{k-1}^3\,.
\]
Taking $m = \log(2B^2N\ell_{k-1}^4+1)/\log(1/\gamma)$ and rearranging the terms, the above inequality yields
\[
\prob\Big(|\cS_{ik}\cup \cO_{ik}| \ge \frac{\log(2B^2N\ell_{k-1}^4+1)}{\log\Big(\frac{1}{\gamma}\Big)}\Big) \le \frac{1}{\ell_{k-1}}\,.
\]

}
{Finally using that $\cS_{ik}$ and $\cO_{ik}$ are disjoint, we have that with probability at least $1-1/\ell_{k-1}$,
\begin{align}\label{eq:S}
{|\cS_{ik}| + |\cO_{ik}| ~\le~ \frac{\log\left({2}B^2N {\ell_{k-1}^3} + 1\right)}{\log(1/\gamma)} ~\le~ \ngg{C_1\frac{ \log(\ell_{k-1})}{\log(1/\gamma)} +C_2 \frac{\log(N)}{\log(1/\gamma)}}\,,}
\end{align}
for a constant $C_1$ and constant $C_2 = C_2(B)$.}

{So far, we have established an upper bound on $|\cS_{ik}| $ {and $|\cO_{ik}|$.}
We next bound {$|\Lie_{ik}|$}. {With this aim, we partition the set of lies $\Lie_{ik}$ into two subsets $\Lie_{ik}^{s}$ and $\Lie_{ik}^{o}$, defined below.
\begin{align}\label{def:lie_s}
\Lie_{ik}^s = \Big\{t: ~
t\in \epi_{k-1}
,\,  \ind(v_{it} > \max\{\b, r_{it}\}) ~=~1, ~~~ \ind(b_{it} > \max\{\b, r_{it}\})~=~ 0 \Big\}\,,\\ \label{def:lie_o}
\Lie_{ik}^o = \Big\{t: ~
t\in \epi_{k-1}
,\,  \ind(v_{it} > \max\{\b, r_{it}\}) ~=~0, ~~~ \ind(b_{it} > \max\{\b, r_{it}\})~=~ 1 \Big\}\,.
\end{align}  
In the following, we bound $|\Lie_{ik}^s|$ and $|\Lie_{ik}^o|$ in order to provide an upper bound on $|\Lie_{ik}|$. }

{We start with bounding $|\Lie_{ik}^s|$.}  
{Define $\cS_{ik}^c \equiv\{t:\, t\in{\epi_{k-1}},\, q_{it} = 1 \text{ or }  \shade_{it} < {1/\ell_{k-1}} \}$. Then, 
$| \Lie_{ik}^s| \le |\cS_{ik}|+  |\cS_{ik}^c\cap \Lie_{ik}^s|$. We have already bounded $ |\cS_{ik}|$. In the following, we bound $|\cS_{ik}^c\cap \Lie_{ik}^s|$.}

By definition~\eqref{def:lie_s}, we first note that for {$t\in \Lie_{ik}^s$}, $q_{it}=0$. Therefore, for $t\in \cS_{ik}^c\cap \Lie_{ik}^s$, we have $s_{it}<1/\ell_{k-1}$. {Let $\cF_t \equiv \{(x_\tau, b_{-i\tau}^+, r_{i\tau}):\, 1\le \tau \le t\}$. Then, by substituting for {$b_{it} = v_{it} - s_{it}$}\footnote{{Note that  here $s_{it}>0$ as $t\in \Lie_{ik}^s$ and consequently $b_{it}< v_{it}$.}} and $v_{it} = \<x_t,\beta_i\>+z_{it}$, we have
\begin{align}
\prob(t\in \cS_{ik}^c\cap{\Lie_{ik}^s} | \cF_t) &= \prob\left(z_{it}\in [\max\{\b,r_{it}\} - \<x_t,\beta_i\>, \max\{\b,r_{it}\} - \<x_t,\beta_i\> + s_{it}] \text{ and } s_{it}\le \frac{1}{\ell_{k-1}} \Big|\cF_t\right)\nonumber\\
&\le \prob\left(z_{it}\in \big[\max\{\b,r_{it}\} - \<x_t,\beta_i\>, \max\{\b,r_{it}\} - \<x_t,\beta_i\> + \frac{1}{\ell_{k-1}}\big] \Big|\cF_t\right)\nonumber\\
&=\int_{\max\{\b,r_{it}\} - \<x_t,\beta_i\>}^{\max\{\b,r_{it}\} - \<x_t,\beta_i\> + 1/\ell_{k-1}} f(z)\de z 
\le  \frac{c}{\ell_{k-1}}\,, \label{eq:probB}
\end{align} 
where the equality holds because $z_{it}$ is independent of $\cF_t$. In addition, in the last step, $c \equiv \max_{v\in [-\maxn,\maxn]} f(v)$ is the bound on the noise density.\footnote{{Note} that the density $f$ is continuous and hence attains its maximum over compact sets.}

Define $\zeta_t~\equiv~ \ind(t\in \cS_{ik}^c\cap \Lie_{ik})$ and  {$\omega_t~\equiv~ \prob(t\in \cS_{ik}^c\cap \Lie_{ik}|\cF_t)$.} Then,  $|\cS_{ik}^c\cap \Lie_{ik}| = \sum_{t=\ell_{k-1}}^{\ell_k-1} \zeta_t$ and 
$\E(\zeta_t - \omega_t|\cF_t) = 0$. Therefore, by using a multiplicative Azuma inequality (see e.g.~\cite[Lemma 10]{koufogiannakis2014nearly}), for any $\epsilon\in (0,1)$ and any $\eta > 0$ we have
\begin{align}\label{Azuma-Mul}
\prob\Big(|\cS_{ik}^c\cap \Lie_{ik}^s| ~\ge~ \frac{1+\eta}{1-\epsilon} \sum_{t=\ell_{k-1}}^{\ell_k - 1} \omega_t  \Big)\le \exp\Big(-\epsilon\eta \sum_{t=\ell_{k-1}}^{\ell_k - 1} \omega_t\Big)\,.
\end{align} 
We use the shorthand $A \equiv \sum_{t=\ell_{k-1}}^{\ell_k - 1} \omega_t$. By setting $\epsilon = 1/2$, $\eta = (2/A) \log(\ell_{k-1}/\delta)$, the r.h.s of Equation~\eqref{Azuma-Mul} becomes $\delta/\ell_{k-1}$.  Further, recalling Equation~\eqref{eq:probB}, we have $A\le \ell_{k-1} (c/\ell_{k-1}) = c$. Hence, rewriting bound~\eqref{Azuma-Mul}, we get that with probability at least $1-\delta/\ell_{k-1}$,
\[
{|\cS_{ik}^c\cap \Lie_{ik}^s|} ~=~ \sum_{t=\ell_{k-1}}^{\ell_k-1} \zeta_t ~\le~ 2(1+\eta) A ~\le~ 2c + 4 \log(\ell_{k-1}/\delta)\,.
\]

\ngg{Combining the above inequality with bound~\eqref{eq:S}, we get
\[
|\Lie_{ik}^s|  ~\le~ |\cS_{ik}^c\cap \Lie_{ik}^s| +|\cS_{ik}| ~\le~ 2c + 4 \log(\ell_{k-1}/\delta) +  \ngg{C_1\frac{ \log(\ell_{k-1})}{\log(1/\gamma)} +C_2 \frac{\log(N)}{\log(1/\gamma)}}\,.
\]
{One can establish a similar bound for $|\Lie_{ik}^o|$. Then, claim~\eqref{claim:lies} follows by using the bounds on $|\Lie_{ik}^s|$ and $|\Lie_{ik}^o|$.}
}}

\ngg{To prove claim~\eqref{claim:shades}, we write
\begin{align*}
\sum_{t\in E_{k-1}} \shade_{it}(1-q_{it}) &~=~ \sum_{t\in \epi_{k-1}} \shade_{it}(1-q_{it}) \ind(t\in \cS_{ik})+ \sum_{t\in \epi_{k-1}} \shade_{it}(1-q_{it}) \ind(t\in \cS_{ik}^c)\\
&~\le~ B|\cS_{ik}| +  \sum_{t\in E_{k-1}} \frac{1}{\ell_{k-1}}\\
&~\le~ B\left(\ngg{C_1\frac{ \log(\ell_{k-1})}{\log(1/\gamma)} +C_2 \frac{\log(N)}{\log(1/\gamma)}}\right)+ 1\,,  
\end{align*}
{where we used the fact that $(i)$ $\shade_{it}\le v_{it} \le B$, and $(ii)$ for any $t\in \cS_{ik}^c$, either $q_{it}=1$ or $\shade_{it} < {1}/{\ell_{k-1}}$}. This complete the proof of claim~\eqref{claim:shades}.

{Finally, we show claim~\eqref{claim:overbids}. Let $\cO_{ik}^c \equiv\{t:\, t\in{\epi_{k-1}},\, q_{it} = 0 \text{ or }  o_{it} < {1/\ell_{k-1}} \}$. Then, we write
\begin{align*}
\sum_{t\in E_{k-1}} o_{it}q_{it} &~=~ \sum_{t\in \epi_{k-1}} o_{it}q_{it} \ind(t\in \cO_{ik})+ \sum_{t\in \epi_{k-1}} o_{it}q_{it} \ind(t\in \cO_{ik}^c)\\
&~\le~ \M|\cO_{ik}| +  \sum_{t\in E_{k-1}} \frac{1}{\ell_{k-1}}\\
&~\le~ \M\left(\ngg{C_1\frac{ \log(\ell_{k-1})}{\log(1/\gamma)} +C_2 \frac{\log(N)}{\log(1/\gamma)}}\right) + 1\,.  
\end{align*}
Here,
 we used the fact that $(i)$ $o_{it} \le \M$, and $(ii)$ for any $t\in \cO_{ik}^c$, either $q_{it}=0$ or $o_{it} < {1}/{\ell_{k-1}}$}. This {completes} the proof of claim~\eqref{claim:overbids}.
 }}


\section{Proof of Proposition~\ref{propo:learning2}}\label{proof:propo-learning2} 
The proposition can be proved by following similar steps {used in} the proof of Proposition~\ref{propo:learning}. 
{For the quadratic loss function
\[
\tilde{\cal L}_{ik}(\pv) = \frac{1}{|I_k|}\sum_{t\in I_k} (BNq_{it} - \<x_t,\pv\>)^2\,,
\] 
the gradient and  Hessian are given by
\begin{align}\label{hessian-2}
\nabla\tilde\cL_{ik}(\pv) = \frac{1}{|\initial_k|} \sum_{t \in \initial_k} {\mu_{it}}(\pv) x_t\,,\quad \nabla^2 \tilde\cL_{ik}(\pv) = \frac{1}{|\initial_k|} \sum_{t \in \initial_k} 2 x_t x_t^\sT\,,
\end{align}
where with a slight abuse of notation, ${\mu_{it}}(\beta) = 2(\<x_t,\beta\> - BNq_{it})$.

By the second-order Taylor's theorem, expanding around $\beta_i$, we have  
\begin{align}\label{taylor2}
\tilde\cL_{ik}(\beta_i) - \tilde\cL_{ik}(\hbeta_{ik}) = -\<\nabla \tilde\cL_{ik}(\beta_i),\hbeta_{ik}-\beta_i\> -\frac{1}{2} \<\hbeta_{ik} - \beta_i, \nabla^2\tilde\cL_{ik}(\tilde{\beta})(\hbeta_{ik}-\beta_{ik})\>\,,
\end{align}
for some $\tilde{\beta}$ on the segment connecting $\beta_i$ and $\hbeta_{ik}$. 

By optimality of $\hbeta_{ik}$, we have  $\tilde\cL(\hbeta_{ik})\le \tilde\cL(\beta_i)$ and therefore by~\eqref{taylor2}, we have
\begin{align} 
\frac{1}{2} \<\hbeta_{ik} - \beta_i, \nabla^2\tilde\cL_{ik}(\tilde{\beta})(\hbeta_{ik}-\beta_i)\> \le -\<\nabla \tilde \cL_{ik}(\beta_i),\hbeta_{ik}-\beta_i\>\,.
\end{align}
Using Equation~\eqref{hessian-2}, the r.h.s in the above equation can be written as 
\begin{align}\frac{1}{2} \<\hbeta_{ik} - \beta_i, \nabla^2\tilde \cL_{ik}(\tilde{\beta})(\hbeta_{ik}-\beta_i)\> 
& ~=~(\hbeta_{ik} - \beta_i)^{\sT} \Big(\frac{1}{|I_k|} \sum_{t\in\epi_{k-1}} x_t x_t^\sT \Big)(\hbeta_{ik}-\beta_i)\nonumber\\ &~=~ \frac{1}{|I_k|}(\hbeta_{ik} - \beta_i)^{\sT}   \Big( X_k^{\sT} X_k \Big)(\hbeta_{ik}-\beta_i) \nonumber\\
&~=~ \frac{1}{|I_k|} \|X_k(\hbeta_{ik} - \beta_i)\|^2\,. \nonumber
    \end{align}
    Here, $X_k$ is the matrix of size $|I_k|$ by $d$, whose rows are the feature vectors $x_t$, with $t\in I_k$ (the exploration phase of episode $k$).
   Therefore, 
\begin{align}\label{eq:Bound-2}
\frac{1}{|I_k|} \|X_k(\hbeta_{ik} - \beta_i)\|^2\le -\<\nabla \tilde \cL_{ik}(\beta_i),\hbeta_{ik}-\beta_i\> \le  \|\nabla\tilde \cL_{ik}(\beta_i)\|\, \|\hbeta_{ik} - \beta_i\|\,.
\end{align}

In the next lemma, we bound the gradient of the quadratic loss function. This Lemma is analogous to Lemma~\ref{lem:grad-hessian}.
}
\begin{lemma}\label{lem:grad-hessian-2}
Consider the quadratic loss~\eqref{eq:L_2} and define the probability event
\begin{eqnarray}
\event \equiv \big\{\|\nabla \tilde\cL_{ik}({\beta_i})\| \le \lambda_0 \big\}\,,\quad \text{with} \quad \lambda_0 \equiv {4B(N+1)} \sqrt{\frac{\log(\ell_{k-1} d)}{|\initial_k|}} + {4B(N+1)} \frac{|\Lie_{ik}|}{|\initial_k|}\,. \label{event-prob-2}
\end{eqnarray}
Then, we have $\prob(\event)\ge 1 - d^{-0.5}\ell_{k-1}^{-1.5}$. 
\end{lemma}
Proof of Lemma~\ref{lem:grad-hessian-2} is given in Section~\ref{proof:lem-grad-hessian-2}. 
{
Using Lemma~\ref{lem:grad-hessian-2} in bound~\eqref{eq:Bound-2}, we get
\begin{align}
\frac{1}{|I_k|} \|X_k(\hbeta_{ik} - \beta_i)\|^2\le  \lambda_0\, \|\hbeta_{ik} - \beta_i\|\,.
\end{align}
The proof of Proposition~\ref{propo:learning2} then follows exactly along the lines after Equation~\eqref{main-ineq0} in the proof of its counterpart, Proposition~\ref{propo:learning}.}
\subsection{Proof of Lemma~\ref{lem:grad-hessian-2}} \label{proof:lem-grad-hessian-2}

Let $\tilde{q}_{it} = \ind(v_{it}> {\max\{b_t^-,r_{it}\}})$ be the allocation variables as if buyer $i$ was truthful. Then by definition of set of lies $\Lie_{ik}$, as per \eqref{def:lie}, 
for {$t\notin \Lie_{ik}$}, we have $q_{it} = \tilde{q}_{it}$. We define $\tilde{\mu}_{it}(\beta)$ as the counterpart of ${\mu_{it}}(\beta)$, where we replace $q_{it}$ by $\tilde{q}_{it}$, i.e., 
$$\tilde{\mu}_{it}(\beta) = 2(\<x_t,\beta\> - BN\tilde{q}_{it})\,.$$
{Recall that ${\mu_{it}}(\beta) = 2(\<x_t,\beta\> - BNq_{it})$.} 
{Since} ${\mu}_{it}(\beta) = \tilde{\mu}_{it}(\beta)$ for $t\notin \Lie_{ik}$, we can write
\begin{align} \nonumber 
\nabla{\tilde\cL}_{ik}(\beta) &~=~ \frac{1}{|\initial_k|} \sum_{t\in\initial_{k}} \tilde{\mu}_{it}(\beta) x_t - \frac{1}{|\initial_k|} \sum_{t\in\initial_{k}}  \left\{\tilde{\mu}_{it}(\beta) x_t - {\mu_{it}}(\beta) x_t \right\}\\
&~=~  \frac{1}{|\initial_k|} \sum_{t\in\initial_{k}} \tilde{\mu}_{it}(\beta) x_t - \frac{1}{|\initial_k|} \sum_{t\in\Lie_{ik}\cap \initial_k}  \left\{\tilde{\mu}_{it}(\beta) x_t - {\mu_{it}}(\beta) x_t \right\}\,. \label{eq:grad_2}
\end{align}
 {To bound  ${\nabla{\tilde\cL}_{ik}(\beta_i)}$, we start with bounding $|{\mu}_{it}(\beta_i)|$ and $|\tilde{\mu}_{it}(\beta_i)|$.}  By our normalization $\|x_t\|\le 1$. Further, since $\|\beta_i\|\le \maxpv< B$, we obtain $|\<x_t,\beta_i\>|\le B$. This implies that  $|{\mu_{it}}(\beta_i)| = 2|\<x_t,\beta_i\> - BNq_{it} |\le 2B(N+1)$. 
Similarly, we have $|\tilde{\mu}_{it}({\beta_i})| \le 2B(N+1)$. 
 {Therefore, by  Equation (\ref{eq:grad_2}), we have }
\begin{align}\label{grad-2-2}
{\Big\|\nabla{\tilde\cL}_{ik}(\beta_i)\Big\| ~\le~ \frac{1}{|\initial_k|}  \bigg\|\sum_{t\in\initial_{k}}  \tilde{\mu}_{it}(\beta_i) x_t  \bigg\|+ \frac{4B(N+1)}{|\initial_k|} |\Lie_{ik}|  \,,}
\end{align}
where we used that $\|x_t\|\le 1$. {To complete the proof of the first part of the lemma, we bound the first term on the right hand side of~\eqref{grad-2-2} using the Matrix Freedman inequality for bounded martingale matrices (see Appendix~\ref{sec:freedman}). {Similar to the proof of Lemma \ref{lem:grad-hessian}, define $S_j = \sum_{t=\ell_{k}}^{j-1+\ell_{k}}\tilde{\mu}_{it}(\beta_i) x_{t}$ and $S_0 = 0$. In order to show that $S_j$ is a vector martingale with bounded differences,  we need to show that $\E[\tilde{\mu}_{it}(\beta_i)x_t] = 0$ and bound $\|\tilde{\mu}_{it}(\beta_i)x_t\|$. }}

Recall that in the pure exploration phase, for a buyer chosen uniformly at random, we set the reserve $r\sim {\sf uniform(0,B)}$, and for other buyers we set their reserves to $\infty$. {Therefore, for any period $t$ in the pure exploration phase of episode $k$, i.e.,  for any $t\in \initial_k$, we have}
\[\prob(\tilde{q}_{it} = 1|v_{it},x_t) ~=~ \frac{v_{it}}{BN}\,.\]
As a result, $\E[\tilde{q}_{it}|v_{it}, x_t] = v_{it}/(BN)$, {where the expectation is taken w.r.t. to the randomness in reserve prices.} Thus,  
\begin{align}
\E[\tilde{\mu}_{it}(\beta_i)|x_{t}] ~=~ 2\E[(BN\E[\tilde{q}_{it}|v_{it}, x_t] - \<x_t,\beta_i\>)~|~x_t] ~=~ 2\E[v_{it} -\<x_t,{\beta_i}\>|x_t] ~=~ 2\E[z_{it}|x_t] = 0\,. 
\end{align}
This also implies that {$\E[\tilde{\mu}_{it}(\beta_i)x_t] = 0$}. Further, {$\|\tilde{\mu}_{it}(\beta_i)x_t\| \le 2B(N+1) \|x_t\| \le 2B(N+1)$}.
Thus, by virtue of Matrix Freedman inequality, we have
\begin{align}\label{Azuma-2-2}
\prob\bigg(\frac{1}{|\initial_k|} \bigg\|\sum_{t\in\initial_{k}} \tilde{\mu}_{it}(\beta_i) {x_{t}}\Big\| \ge 4B(N+1) \sqrt{\frac{\log(\ell_{k-1}d)}{|\initial_k|}} \bigg) \le (d+1) \exp^{-(12/8) \log(\ell_{k-1}d)} = \frac{1}{d^{0.5}\ell_{k-1}^{1.5}}\,.
\end{align}  
Combining Equations~\eqref{grad-2-2} and \eqref{Azuma-2-2} shows that $\prob(\event) \ge 1- d^{-0.5}\ell_{k-1}^{-1.5}$, {where the probability event $\event$ is defined in \eqref{event-prob-2}.}

\section{Proof of Proposition~\ref{propo:lies2}}\label{proof:lies2}
The proof is based on comparing the utility loss of an untruthful buyer with his future utility gain and using the fact that buyers are utility-maximizing. Note that in SCORP, any utility loss due to untruthful bidding can only happen in the exploration phase of the episodes, as the submitted bids in the exploitation phase of the episodes are not used in estimating the preference vectors. Therefore, during the exploitation phase,  there is no incentive for buyers to deviate from being truthful. Hence, to bound the utility loss of a buyer due to untruthful bidding, we only need to focus on the pure exploration phase. 

By focusing on the exploration phase, it is easy to verify that by following similar steps as in the proof of Lemma~\ref{lem:utilityloss}, we have 
\begin{align}\label{eq:utilityloss2}
{\E[u_{it}^-|{\shade_{it}, o_{it}}, q_{it}] ~\ge~ 
{\frac{1}{{2}BN}} \gamma^ t \shade_{it}^2 (1-q_{it})+{ \frac{1}{2BN}\gamma^t o_{it}^2 q_{it}}\,, }
\end{align}
{where the expectation is taken w.r.t. to the randomness in reserve prices. }

Observe that this bound is stronger than Lemma~\ref{lem:utilityloss} in that the factor ${1}/({2}BN\ell_{k-1})$ is replaced by ${1}/{{2}BN}$. The reason is that \emph{in each period} of pure exploration phase, for a randomly chosen buyer we set his reserve $r\sim{\sf uniform}(0,B)$ and we set other buyer's reserves to $\infty$. This is in contrast to the CORP policy (under known distribution $F$)  that we do such exploration only with probability $1/\ell_{k-1}$ in each period of episode $k$.
We remove the proof of Equation~\eqref{eq:utilityloss2}, as it is very similar to the proof of Lemma~\ref{lem:utilityloss}.

By having Equation~\eqref{eq:utilityloss2} in place, the rest of the proof is exactly the same as the proof of Proposition~\ref{propo:lies}.

\section{{Proof of Technical Lemmas}}\label{sec:technical}

\subsection{Proof of Lemma~\ref{techlem1}}
Note that in a second-price auction with truthful buyers, $W_{it}(r)$ indicates the revenue that firm earns when buyer $i$ wins the auction and has been posted reserve price $r$. Therefore by definition of optimality $r^\star_{it}= {\arg\max_{r}} W_{it}(r)$. (In Proposition~\ref{prop:opt_reserve}, it is shown that $r^\star_{it}$ {is the optimal solution of optimization problem}~(\ref{def:rstar}).) Therefore, $W_{it}'(r^\star_{it}) = 0$.

Also, by Equation~\eqref{W-1-dev}, we have
\begin{align}
W_{it}'(r) ~=~ {H_{it}}(r) \Big((1- F(r - \<x_t,\beta_i\>)) - r f(r-\<x_t,\beta_i\>)\Big)\,.
\end{align}
Hence, 
\begin{align} \nonumber 
W_{it}''(r) ~&=~ {{h_{it}}(r) \Big((1- F(r - \<x_t,\beta_i\>)) - r f(r-\<x_t,\beta_i\>)\Big)}\\
~&~- {2{H_{it}}(r) f(r-\<x_t,\beta_i\>) -{H_{it}}(r) r f'(r-\<x_t,\beta_i\>)\,.} \label{W-2-dev}
\end{align}
Since valuations and bids are bounded by constant $B$, clearly $0\le r^\star_{it}, r_{it} \le B$ and given that $r$ is between them, we also have $0\le r\le B$. {In addition, considering the fact that 
the market noise} is bounded in $[-\maxn, \maxn]$ and $f$ and $f'$ are continuous, {both  $f$ and $f'$} attain their maximum over the compact interval $[-\maxn,\maxn]$. 
Let $c_1 = \max_{y\in[-\maxn, \maxn]} f(y)$ and $c_2 = \max_{y\in[-\maxn, \maxn]} f'(y)$. Further, since $0\le v_t^-\le B$, its density $h_{it}$ is supported in $[-B,B]$ and due to continuity, it attains its maximum over this interval.
Let $c_3 = \max_{y\in[-B,B]} h(y)$. Therefore,
\[ |W_{it}''(r)| \le c_3+{2c_1}+Bc_2\,.\]
The result follows by {setting} $c\equiv c_3+{2c_1}+Bc_2$.

\subsection{Proof of Lemma~\ref{techlem2}}
We define  function $g:\reals \mapsto \reals$ as follows:
\begin{align}\label{g}
g(\theta) ~=~ \arg\max_{y} \{y(1-F(y-\theta))\}\,.
\end{align}

By this definition, {for any $t\in \epi_{k}$,} we have $r^\star_{it} = g(\<x_t,\beta_i\>)$ and $r_{it} = g({\<x_t,\hbeta_{ik}\>})$. {Then, by showing $g(\cdot)$ is 1-Lipschitz function, claim~(\ref{r-B}) follows. To see this note that 
\begin{align}
|r^\star_{it}- r_{it}|  ~=~ |g(\<x_t,\beta_i\>) - g(\<x_t,\hbeta_{ik}\>)| ~\le~ |\<x_t,\beta_i - \hbeta_{ik}\>| \,,
\end{align} 
where the inequality holds  because of 1-Lipschitz property of function $g$.}

{By definition~{(\ref{g})}, $g(\theta)$ should satisfy the following stationary condition:
\[
1-F(g(\theta) - \theta) ~=~ g(\theta) f(g(\theta)  -\theta)\,.
\]
 Define $\varphi(y) \equiv y - \frac{1-F(y)}{f(y)}$ as the \emph{virtual valuation} function. Then, 
we can write $g(\theta)$ in terms of virtual valuation function:
$\varphi(g(\theta)-\theta) = -\theta\,.
$
Since $\varphi$ is injective, by applying $\varphi^{-1}$ to both sides, we can write $g(\theta)$ explicitly in terms of virtual valuation function:
\begin{align}\label{g-phi}
g(\theta) = \theta + \varphi^{-1}(-\theta)\,.
\end{align}
Using characterization~\eqref{g-phi}, we show that $g$ is 1-Lipschitz. To do so, we verify $g'(\theta) = 1 - 1/\varphi'(\varphi^{-1}(-\theta))$ is less than  one. In particular, $g'(\theta) \le 1$ because 
  $\varphi'(y) \ge 1$. To see why this holds note that
$\varphi(y)$ can be written as 
$\varphi(y) = y + \frac{1}{\log'(1-F(y))}$. 
Then, by Assumption~\ref{assump:logcancavity},  $1-F$ is log-concave. This implies that $\log'(1-F(y))$ is decreasing,  and consequently  $\varphi(y)$ is increasing. Indeed, this implies that $\varphi'(y)\ge  1$.}

\if false
 As we showed above, $\varphi'(y) > 1$ for all $y$, and hence 
$0<g'(\theta)<1$ for all $\theta$. This clearly implies that $g$ is a 1-Lipschitz function.

 We start with  writing function $g$ in terms of the \emph{virtual valuation} function, defined as $\varphi(y) \equiv y - \frac{1-F(y)}{f(y)}$.

We can also write $\varphi(y)$ as
$\varphi(y) = y + \frac{1}{\log'(1-F(y))}
$. 
Since $1-F$ is log-concave as per Assumption~\ref{assump:logcancavity}, $\log'(1-F(y))$ is decreasing, which implies that $\varphi(y)$ is increasing. Indeed it implies that {$\varphi'(y)\ge  1$.}
\fi

\subsection{Proof of Lemma~\ref{techlem3}}\label{proof:techlem3}
We first prove Claim~\eqref{second-claim}. {Observe that when $\vm - \bm<0$, Claim~\eqref{second-claim} holds, as $\shade_{it}\ge 0$ for any $i\in [N]$. Thus, it suffices to show that ${(\vm - \bm)} \le \max\big\{\shade_{it}(1-q_{it}):\, i\in [N]\big\}$. Without loss of generality, assume $v_{1t} > v_{2t} > \dotsc> v_{Nt}$. Then, $v_t^- = v_{2t}$.
 If $b_t^-\ge b_{2t}$, then buyer $2$ will not receive the item, i.e., $q_{2t} = 0$ and we have
\[
v_t^- - b_t^- = v_{2t} - b_t^- \le v_{2t} - b_{2t} = (v_{2t} - b_{2t}) (1-q_{2t})  = \shade_{2t}(1-q_{2t})\,,
\]
proving the claim in this case. 
The other case is when $b_t^- < b_{2t}$ and hence $b_{2t}$ is the highest bid. {This implies that} $b_{1t}\le b_t^-$ and we have the following chain of inequalities:
\begin{align}
v_t^- - b_t^-  = v_{2t} - b_t^- \le v_{2t} -b_{1t} < v_{1t} - b_{1t}\,. \label{chain}
\end{align}
Further, since buyer $2$ has the highest bid, $q_{2t} =1$ and $q_{it} = 0$ for all $i\neq 2$. In particular, $q_{1t} = 0$. Combining this with~\eqref{chain}, we get
\[
v_t^- - b_t^- < (v_{1t} - b_{1t}) (1-q_{1t})\,,
\]
which proves the claim in this case as well.

{We next prove Claim~\eqref{second-claim-2}. Suppose $q_{it} = 0$ and let buyer $j$ be the winner ($q_{jt} = 1$ and $j\neq i$). Then, by definition $b_{-it}^+ = b_{jt}$ and
\[
b_{-it}^+ - v_{-it}^+ = b_{jt} - v_{-it}^+ \le b_{jt} - v_{jt} = o_{jt}q_{jt}\,.
\]
Here, we use that $v_{jt}\le v_{-it}^+$ because $j\neq i$.
}

\subsection{Proof of {Lemma}~\ref{techlem2-2}}
Recall that $r^\star_{it}$ and $r_{it}$ are given by the following equations:
\begin{align*}
r^\star_{it} &= \underset{r}{\arg\max\;}\min_{F\in\cF}\; r(1-F(r-\<x_t,\beta_i\>))\,,\\
r_{it} &= \underset{r}{\arg\max\;}\min_{F\in\cF}\; r(1-F(r-\<x_t,\hbeta_{ik}\>))\,.
\end{align*}
Let $\tr^\star_{it} = r^\star_{it} -\<x_t,\beta_i\>$ and $\tr_{it} = r_{it} -\<x_t,\hbeta_{ik}\>$. By a change of variable, it is easy to see that $\tr^\star_{it}$ and $\tr_{it}$ are the solutions to the following optimization problems:
\begin{align*}
\tr^\star_{it} &= \underset{r}{\arg\max\;}\min_{F\in\cF} \Big\{(r+{\<x_t,\beta_i\>}) (1-F(r))\Big\}\,,\\
\tr_{it} &= \underset{r}{\arg\max\;}\min_{F\in\cF} \Big\{(r+{\<x_t,\hbeta_{ik}\>}) (1-F(r))\Big\}\,.
\end{align*}
Define {function} $H:\reals \to \reals$ as $H(r)\equiv \max_{F\in \cF} F(r)$.  Observe that $\tr^\star_{it}+\<x_t,\beta_{i}\> = r^\star_{it} >0$ and hence,
\begin{align}
\tr^\star_{it} &~=~ \underset{r}{\arg\max\;}\min_{F\in\cF} ~(r+{\<x_t,\beta_i\>}) (1-F(r))\nonumber\\
&~=~ \underset{r}{\arg\max\;} (r+{\<x_t,\beta_i\>})~\min_{F\in\cF} (1-F(r)) \nonumber\\
&~=~ \underset{r}{\arg\max\;} (r+{\<x_t,\beta_i\>}) (1-H(r))\,. \label{opt0}
\end{align}
Using the change of variable $r\leftarrow r+\<x_t,\beta_i\>$, we obtain
\begin{align}\label{opt1}
r^\star_{it} = \underset{r}{\arg\max\;} ~r (1-H(r-\<x_t,\beta_i\>))\,.
\end{align}
Likewise, 
\begin{align}\label{opt2}
{r_{it}} = \underset{r}{\arg\max\;} ~r (1-H(r-\<x_t,\hbeta_{ik}\>))\,.
\end{align}
Now, note that by definition of function $H$, we have $\log(1-H(r)) = \min_{F\in \cF} \log(1-F(r))$. Further, $F$ is log-concave for all $F\in \cF$, as per Assumption~\ref{assump1}. Moreover, using the fact that the (pointwise) minimum of concave functions is also concave, we have that $1-H$ is log-concave.  By virtue of characterizations~\eqref{opt1} and~\eqref{opt2}, and log-concavity of $1-H$, the claim follows from the same proof of {Lemma~\ref{techlem2}} and hence is omitted. The only subtle point is that {function} $H$, although continuous, may not be differential at some points. Therefore, in using the argument of {Lemma~\ref{techlem2}}, derivative should be replaced by subgradient.
\section{Matrix Freedman Inequality}\label{sec:freedman}
For readers' convenience, {here} we state the Matrix Freedman inequality for martingales.
\begin{thm}[Rectangular Matrix Freedman] \label{thm:freedman}
Consider a matrix martingale $\{Y_k: k = 0, 1, 2,\dotsc\}$ whose values are matrices with dimension $d_1\times d_2$ and let $\{X_k: k = 1, 2,\dotsc\}$ be the difference
sequence. Assume that the difference sequence is uniformly bounded:
\begin{align}
{\|X_k\|_{\rm op}}\le R \quad \text{ almost surely }\quad \text{ for } k\ge 1\,,
\end{align}
where $\|\cdot\|_{\rm op}$ denotes the operator norm\footnote{{For a matrix $A$, its operator norm is defined as $\|A\|_{\rm op}= \inf\{c\ge 0:  \|Av\|\le c \|v\|, \text{ for any vector $v$}\}$. {Equivalently,} the operator norm is the largest singular value of a matrix.}}, {and $R$ is a constant.} 
Define two predictable quadratic variation processes for this martingale
\begin{eqnarray*}
W_{1,k} &\equiv& \sum_{j=1}^k \E[X_j X_j^\sT | Y_1, \dotsc, Y_{j-1}]\,,\\
W_{2,k} &\equiv& \sum_{j=1}^k \E[X_j^\sT X_j | Y_1, \dotsc, Y_{j-1}]\,,
\end{eqnarray*}
for $k\ge 1$. Further, for given $\sigma^2 > 0$ and $t\ge 0$, let event ${\cal A} \equiv \Big\{\exists k\ge 0:\; \|Y_k\|_{{\rm op}}\ge t, ~ \max(\|W_{1,k}\|_{\rm op}, \|W_{2,k}\|_{\rm op}) \le \sigma^2\Big\} $. Then,
\begin{eqnarray*}
\prob({\cal{A}})&\le& (d_1+d_2) \exp\left(-\frac{t^2/2}{\sigma^2+RT/3}\right)
~=~\begin{cases}
(d_1+d_2) \exp(-3t^2/8\sigma^2)\quad t\le \sigma^2/R\,,\\
(d_1+d_2) \exp(-3t/8R)\quad\;\; t\ge \sigma^2/R\,.
\end{cases}
\end{eqnarray*}
\end{thm}
We refer to~\citep{tropp2011freedman} for the proof of Theorem~\ref{thm:freedman}. We next state the result of Matrix Freedman theorem specialized to the vector case. This corollary is used in the proof of
Propositions~\ref{proof:lem-grad-hessian} and~\ref{proof:lem-grad-hessian-2}.

\begin{coro}\label{coro:MM}
Consider a vector martingale $\{u_k: k = 0, 1, 2,\dotsc\}$ whose values are vector with dimension $d$ and let $\{v_k: k = 1, 2,\dotsc\}$ be the difference
sequence. Assume that the difference sequence is uniformly bounded:
\begin{align}
\|v_k\|\le R \quad \text{ almost surely }\quad \text{ for } k\ge 1\,,
\end{align}
Define a predictable quadratic variation processes for this martingale:
\begin{eqnarray*}
w_{k} &\equiv& \sum_{j=1}^k \E\Big[\|v_j\|^2 \Big| u_1, \dotsc, u_{j-1}\Big]\,,\quad \text{ for } k\ge 1\,.
\end{eqnarray*}
Then, for all $t\ge 0$ and $\sigma^2 > 0$, we have
\begin{eqnarray}
\prob\Big\{\exists k\ge 0:\; \|u_k\|\ge t\quad \text{and} \quad \|w_{k}\| \le \sigma^2\Big\} &\le& (d+1) \exp\left(-\frac{t^2/2}{\sigma^2+R\ajb{t}/3}\right)\label{coroMM:eq}\\
&=&\begin{cases}\nonumber
(d+1) \exp(-3t^2/8\sigma^2)\quad \text{ for }t\le \sigma^2/R\,,\\
(d+1) \exp(-3t/8R)\quad\;\; \text{ for }t\ge \sigma^2/R\,.
\end{cases}
\end{eqnarray}
\end{coro}
\medskip

{We used Corollary~\ref{coro:MM} in the proof of Lemma~\ref{lem:grad-hessian} to bound the norm of martingale $S_{j}$ (see below Equation~\eqref{grad-2}). Specifically, we used the corollary with $u_j = S_j$, $v_j = \tilde{\mu}_{it}(\beta_i)x_t$ {(with $t = \ell_{k-1}+j-1$)}, $R = u_F$, {and} $\sigma^2 = u_F^2 \ajb{\ell_{k-1}}$. Then, using bound~\eqref{coroMM:eq} for $S_{\ell_{k-1}}$ with $t = 2u_F\sqrt{\log(\ell_{k-1}d)\ell_{k-1}}$, we obtain bound~\eqref{Azuma}.

Likewise, we used Corollary~\ref{coro:MM} in the proof of Lemma~\ref{lem:grad-hessian-2} to bound the {norm} of martingale $S_j$ (see below Equation~\eqref{grad-2-2}). Here, again we set $u_j = S_j$, $v_j = \tilde{\mu}_{it}(\beta_i)x_t$ with $t = \ell_{k-1}+j-1$ (note that in this case $\tilde{\mu}_{it}(\beta_i) = 2(\<x_t,\beta_i\> - BN\tilde{q}_{it})$.) We then have $R = 2B(N+1)$ and $\sigma^2 = 4B^2(N+1)^2 j$. We then use bound~\eqref{coroMM:eq} for $S_{\ell_{k-1}}$ with $t = 4B(N+1)\sqrt{\log(\ell_{k-1}d)|I_k|}$ to {obtain}~\eqref{Azuma-2-2}.}

\ngg{
\section{Proof of Theorem \ref{thm:uniform}\label{sec:example}} We first show the result when $w \in (\underline a, \bar a)$.  Define $F_a(\cdot)$ as the probability distribution of the uniform distribution  with the support of $[-a, a]$. To solve the optimization problem given in  Theorem \ref{thm:uniform}, we first consider the following optimization problem for any $r\ge 0$.
\begin{align}\min_{a\in [\underline a, \bar a]} (1-F_a(r-w))\,.\label{eq:first_opt}\end{align}
We will show that 
\[\min_{a\in [\underline a, \bar a]} (1-F_a(r-w)) =\left\{ \begin{array}{ll}
        \frac{\bar a -(r-w)}{2\bar a} &~~~~ \mbox{if $r< w$};\\
        \frac{1}{2} &~~~~ \mbox{if $r=w$};\\
         \max (\frac{\underline  a -(r-w)}{2\underline a}, 0) &~~~~ \mbox{if $r>w$},
        \end{array} \right.  \]
        where $\frac{\bar a -(r-w)}{2\bar a}$ and $\frac{\bar a -(r-w)}{2\bar a}$ are respectively $(1-F_{\bar a}(r-w))$ and $(1-F_{\underline  a}(r-w))$. To do so, we consider the following three cases:
        \begin{itemize}
        \item Case 1 ($r< w$): Consider any $a\in [\underline a, \bar a]$ such that $r-w\le -a$. Then, $(1-F_a(r-w)) = 1$. Now, consider any $a \in [\underline a, \bar a]$ such that $-a<r-w< a$.\footnote{Observe that $r-w$ cannot exceed $a$ as under case $1$, $r-w<0$ and $a> 0$.} When $-a<r-w< a$, we have 
        \[ (1-F_a(r-w)) =r\frac{ a -(r-w)}{2 a}\,.\]
        It is easy to see that $\frac{ a -(r-w)}{2 a}$ is decreasing in $a$ and as  a result, $\arg\min_{a\in [w-r, \bar a]}\frac{ a -(r-w)}{2 a} =\bar a$. This shows that $\arg\min_{a\in [\underline a, \bar a]}\frac{ a -(r-w)}{2 a} =\bar a$, which is the desired result.
        
        \item Case 2 ($r= w$): This case is simple as for any $a\in[\underline a, \bar a]$, we have $1- F_{a}(r-w) = 1-F_a(0) = 1/2$.
        \item Case 3 ($r> w$): If $\underline a < r-w$, then $\min_{a\in [\underline a, \bar a]} (1-F_a(r-w))$ is indeed zero. Otherwise,  $(1-F_a(r-w))$ is increasing in $a$ and obtains its minimum at $a = \underline a$. 
        \end{itemize}
        
        Next, we solve the following optimization problem:
        \begin{align} \label{eq:opt}\max_{r} \min_{a \in [\underline a , \bar a]} r (1-F_{a}(r-w))\,.\end{align}
To characterize the optimal solution of the aforementioned optimization problem, we consider the following three regions for $w$. 

\begin{itemize}
\item Region 1 ($w\in [\underline a, \bar a]$): To solve problem (\ref{eq:opt}),  we divide this problem into two subproblems. In the first subproblem,  $r \le w$, and in the second one, $r\ge w$. Precisely, the first subproblem concerns the following optimization problem:         \[\max_{r\le w} \min_{a \in [\underline a , \bar a]} r (1-F_{a}(r-w))\ = \max_{r\le w} r \frac{\bar a- (r-w)}{2\bar a}  = \frac{w}{2}\,,\]
where the first equation follows from case 1 and  last equation holds because when $r\le w< \bar a$,  $r\frac{\bar a-(r-w)}{2 \bar a}$ is increasing in $r$. 

The second subproblem is given by 
\[\max_{r \ge w} \min_{a \in [\underline a , \bar a]} r (1-F_{a}(r-w)) = \max_{r\ge  w}~~ \left(r ~\max(\frac{\underline a- (r-w)}{2 \underline a}, 0)\right) = \frac{w}{2}\,, \]
where the first equation follows from case 3 and  last equation holds because when $ w\ge  \underline  a$,  $r\frac{\underline  a-(r-w)}{2 \bar a}$ is increasing in $r$. Put these together, the optimal value of problem (\ref{eq:opt}) is $w/2$, which happens at $r= w$.

\item Region 2 ($w < \underline a$): The solution of the first subproblem is the same  as that in region 1. But, the optimal solution of the second subproblem is $\frac{w+\underline a}{2}$, which implies that the optimal solution of problem (\ref{eq:opt}) is $\frac{w+\underline a}{2}$.

\item Region 3 ($w> \bar a$): It is easy to show that in this region, the optimal solution of problem (\ref{eq:opt}) is $\frac{w+\bar a}{2}$.
\end{itemize}  }

\end{document}